\documentclass[twocolumn]{svjour3}          
\usepackage{tcolorbox}
\usepackage{verbatim}
\usepackage{amsmath}
\usepackage{amssymb}
\usepackage{algorithm}
\usepackage[noend]{algpseudocode}
\usepackage{graphicx}
\usepackage{drum}
\usepackage[toc,page]{appendix}
\usepackage{booktabs}
\usepackage{multirow}
\usepackage{lscape}
\usepackage{xcolor,colortbl}
\definecolor{green}{rgb}{0,1,0}
\definecolor{red}{rgb}{1,0,0}
\definecolor{black}{rgb}{0,0,0}

\tcbuselibrary{most}

\usepackage{url}
\usepackage{natbib}

\usepackage[wide]{sidecap}                                      

\usepackage[
    font=footnotesize,
    format=plain,
    labelfont={bf,sf},
    textfont={it},
  ]{caption}

\usepackage{wrapfig}
\usepackage{float}
\usepackage{xspace}
\usepackage{sidecap}
\usepackage{mathtools}
\usepackage{bm}
\usepackage[utf8x]{inputenc}
\usepackage{soul}
\usepackage{fixltx2e}
\usepackage{adjustbox}

\newcommand{\gooditem}{~~\llap{$\mathbf{+}$}~~}
\newcommand{\baditem}{~~\llap{$\mathbf{-}$}~~}

\newcommand{\Moby}[0]{\software{Moby}\xspace}

\newenvironment{tcfigure}[1]
{\stepcounter{figure}%
\tcbset{enhanced,attach boxed title to top center={yshift=-3mm,yshifttext=-1mm},
  colback=white,colbacktitle=black,fonttitle=\bfseries,
  boxed title style={size=small,colframe=black},title=\figurename~\thefigure: #1}
\begin{tcolorbox}}
{\end{tcolorbox}}

\newenvironment{tcfigure*}[1]
{\begin{figure*}\stepcounter{figure}%
\tcbset{enhanced,attach boxed title to top center={yshift=-3mm,yshifttext=-1mm},
  colback=white,colbacktitle=black,fonttitle=\bfseries,
  boxed title style={size=small,colframe=black},title=\figurename~\thefigure: #1}
\begin{tcolorbox}}
{\end{tcolorbox}\end{figure*}}

\newenvironment{tceqn}[1]
{\tcbset{lefttitle=2mm,righttitle=1mm,left=2mm,title=#1}
\begin{tcolorbox}
\vspace{-5mm}}
{
\end{tcolorbox}}

\newtcolorbox{redbox}{colframe=red,enhanced,breakable,hbox}
\newtcolorbox{bluebox}{colframe=blue,enhanced,breakable}

\newcommand{\software}[1]{\textsc{#1}}
\providecommand{\before}{\hspace{.25mm}^{^-}\hspace{-.75mm}}
\providecommand{\beforef}{\hspace{.25mm}^{^-}\hspace{-1mm}}
\providecommand{\after}{\hspace{.25mm}^{^+}\hspace{-1mm}}

\newenvironment{smallbmatrix}
{\left[ \begin{smallmatrix}}
{\end{smallmatrix}  \right]}

\newcommand{\specialcell}[2][c]{%
  \begin{tabular}[#1]{@{}c@{}}#2\end{tabular}}

\providecommand{\Moby}{\textsc{Moby}\xspace}

\providecommand{\Pacer}{\textsc{Pacer}\xspace}

\providecommand{\LEMKE}{\texttt{LEMKE}\xspace}

\providecommand{\inv}[1]{{#1}^{\ensuremath{\mathsf{-1}}}} 
\providecommand{\tr}[1]{{#1}^{\ensuremath{\mathsf{T}}}} 

\providecommand{\norm}[1]{\left\lVert#1\right\rVert}
\providecommand{\vect}[1]{\bm#1}
\providecommand{\mat}[1]{\mathbf#1}

\newcommand*{\figuretitle}[1]{%
    {\centering
    \textbf{#1}
    \par}
}

\graphicspath{{./}}

\begin{document}

\title{Inverse Dynamics with Rigid Contact and Friction}

\author{Samuel Zapolsky         \and
        Evan Drumwright}

\institute{Samuel Zapolsky \at
		George Washington University \\
		Dept. of Computer Science\\
              samzapo@gwu.edu
                         \and
           Evan Drumwright \at
		George Washington University \\
		Dept. of Computer Science \\
	drum@gwu.edu
}

\date{Received: date / Accepted: date}

\maketitle

\begin{abstract}

Inverse dynamics is used extensively in robotics and biomechanics applications. In manipulator
and legged robots, it can form the basis of an effective nonlinear control strategy by providing a robot with both accurate positional tracking and active compliance. In biomechanics applications, inverse dynamics control can approximately
determine the net torques applied at anatomical joints that correspond to an observed motion.

In the context of robot control, using inverse dynamics requires knowledge of all contact forces acting on the robot; accurately perceiving external forces applied to the  robot requires filtering and thus significant time delay.  An alternative approach has been suggested in recent literature: predicting contact and actuator forces simultaneously under the assumptions of rigid body dynamics, rigid contact, and friction. Existing such inverse dynamics approaches have used
approximations to the contact models, which permits use of fast numerical linear
algebra algorithms. In contrast, we describe inverse dynamics algorithms  that are 
derived only from first principles and use established phenomenological models like Coulomb friction.
 
	We assess these inverse dynamics algorithms in a control context using two virtual robots: a locomoting quadrupedal robot and a fixed-based manipulator gripping a box while using perfectly accurate sensor data from simulation. The data collected from these experiments gives an upper bound on the performance of such controllers \emph{in situ}.  For points of comparison, we assess performance on the same tasks with both error feedback control and inverse dynamics control with virtual contact force sensing.  
\end{abstract}


\section{Introduction}
\label{section:intro}
Inverse dynamics can be a highly effective method in multiple contexts including control 
of robots and physically simulated characters and estimating muscle
forces in biomechanics. The inverse dynamics problem, which entails computing actuator (or muscular) forces that yield desired accelerations, requires knowledge of all ``external'' forces acting on the robot, character, or human. Measuring such forces is limited by the ability to instrument surfaces and filter force readings, and such filtering effectively delays force sensing to a degree unacceptable for real-time operation. An alternative is to use contact force predictions,
for which reasonable agreement between models and reality have been observed
(see, \emph{e.g.},~\citealp*{Aukes:2013}). Formulating such approaches is 
technically challenging, however, because the actuator forces are generally coupled to the
contact forces, requiring simultaneous solution. 

  	Inverse dynamics approaches that simultaneously compute contact forces exist in 
literature. Though these approaches were developed without incorporating
all of the established modeling aspects (like complementarity) and addressing all 
of the technical challenges (like inconsistent configurations) of hard contact, 
these 
methods have been demonstrated performing effectively on real robots. 
In contrast, \emph{this article focuses on formulating inverse dynamics with these established 
modeling aspects---which allows forward and inverse dynamics models to match---and addresses the technical challenges, including solvability}.

\subsection{Invertibility of the rigid contact model}
An obstacle to such a formulation has been the claim that the rigid contact model is not invertible~\citep{Todorov:2014}, implying that inverse dynamics is unsolveable for multi-rigid bodies subject to rigid contact. If forces on the multi-body \emph{other than contact forces} at state $\{\ \vect{q}, \dot{\vect{q}}\ \}$ are designated $\vect{x}$ and contact forces are designated $\vect{y}$, then the rigid contact model (to be described in detail in Section~\ref{section:related-work:LCP}) yields the relationship $\vect{y} = f_{\vect{q}, \dot{\vect{q}}}(\vect{x})$. It is then true that there exists no left inverse $g(.)$ of $f$ that provides the mapping $\vect{x} = g_{\vect{q}, \dot{\vect{q}}}(\vect{y})$ for $\vect{y} = f_{\vect{q}, \dot{\vect{q}}}(\vect{x})$. However, this article will show that there does exist a right inverse $h(.)$ of $f$ such that, for $\vect{h}_{\vect{q}, \dot{\vect{q}}}(\vect{y}) = \vect{x}$, $f_{\vect{q}, \dot{\vect{q}}}(\vect{x}) = \vect{y}$, and in Section~\ref{section:ID:Coulomb} we will show that this mapping is computable in expected polynomial time.  This article will use this mapping to formulate inverse dynamics approaches for rigid contact with both no-slip constraints and frictional surface properties.

\subsection{Indeterminacy in the rigid contact model}
The rigid contact model is also known to be susceptible to the problem of contact indeterminacy, the presence of multiple equally valid solutions to the contact force-acceleration mapping.  This indeterminacy is the factor that prevents strict invertibility and which, when used for contact force predictions in the context of inverse dynamics, can result in torque chatter that is potentially destructive for physically situated robots.  We show that computing a mapping from accelerations to contact forces that evolves without harmful torque chatter is no worse than NP-hard in the number of contacts modeled for Coulomb friction and can be calculated in polynomial time for the case of infinite (no-slip) friction.  

This article also describes a computationally tractable approach for mitigating 
torque chatter that is based upon a rigid contact model without complementarity conditions (see Sections \ref{section:related-work:LCP} and \ref{section:related-work:QP}). The model appears to
produce reasonable predictions: \cite{Anitescu:2006,Drumwright:2010b,Todorov:2014} have all used the model within simulation and physical artifacts have
yet to become apparent.

We will assess these inverse dynamics algorithms in the context of controlling a virtual locomoting robot and a fixed-base manipulator robot.  We will examine performance of error feedback and inverse dynamics controllers with virtual contact force sensors for points of
comparison. Performance will consider smoothness of torque commands, trajectory 
tracking accuracy, locomotion performance, and computation time. 

\subsection{Contributions}
 This paper provides the following contributions:

\begin{enumerate}
\item Proof that the coupled problem of computing inverse dynamics-derived torques and contact forces under the rigid body dynamics, non-impacting rigid contact, and Coulomb friction models (with linearized friction cone) is solvable in expected polynomial time.
\item An algorithm that computes inverse dynamics-derived torques without torque chatter under the rigid body dynamics model and the rigid contact model \emph{assuming no slip along the surfaces of contact}, in expected polynomial time.
\item An algorithm that yields inverse dynamics-derived torques without torque chatter under the rigid body dynamics model and the rigid, non-impacting contact model with Coulomb friction in exponential time in the number of points of contact, and hence a proof that this problem is no harder than NP-hard.
\item An algorithm that computes inverse dynamics-derived torques without torque chatter under the rigid body dynamics model and a rigid contact model with Coulomb friction but does not enforce complementarity conditions (again, see Sections \ref{section:related-work:LCP} and \ref{section:related-work:QP}), in expected polynomial time. 
\end{enumerate}

\emph{These algorithms differ in their operating assumptions}. For example, the algorithms that enforce normal complementarity (to be described in Section~\ref{section:rigid-body-contact}) assume that all contacts are non-impacting; similarly, the algorithms that do not enforce complementarity assume that bodies are impacting at \emph{one or more} points of contact. As will be explained in Section~\ref{section:discretized-idyn}, control loop period endpoint times do not necessarily coincide with contact event times, so a single algorithm must deal with both impacting and non-impacting contacts. It is an open problem of the effects of  enforcing complementarity when it should not be enforced, or vice versa. The algorithms also possess various computational strengths. As results of these open problems and varying computational strengths, we present multiple algorithms to the reader as well as a guide (see Appendix~\ref{appendix:score-idyn}) that details task use cases for these controllers.

This work validates controllers based upon these approaches in simulation to determine their performance under a variety of measurable phenomena that would not be accessible on present day physically situated robot hardware. Experiments indicating performance differentials on present day (prototype quality) hardware might occur due to control forces exciting unmodeled effects, for example.  Results derived from simulation using ideal sensing and perfect torque control indicate the limit of performance for control software described in this work.  We also validate one of the more robust presented controllers on a fixed-base manipulator robot grasping task to demonstrate operation under disparate morphological and contact modeling assumptions.  

\subsection{Contents}
Section~\ref{section:related-work} describes background in
rigid body dynamics and rigid contact, as well as related work in inverse
dynamics with contact and friction. We then present the implementation of three disparate inverse dynamics formulations in Sections~\ref{section:ID:no-slip}, \ref{section:ID:Coulomb}, and \ref{section:approximate-idyn}.  With each implementation, we seek to: \1 successfully control a robot through its assigned task; \2 mitigate torque chatter from indeterminacy; \3 evenly distribute contact forces between active contacts; \4 speed computation so that the implementation can be run at realtime on standard hardware.  Section~\ref{section:ID:no-slip} presents an inverse dynamics formulation with contact force prediction that utilizes the non-impacting rigid contact model (to be described in Section~\ref{section:rigid-body-contact}) with no-slip frictional constraints. Section~\ref{section:ID:Coulomb}  presents an inverse dynamics formulation with contact force prediction that utilizes the non-impacting rigid contact model with Coulomb friction constraints. We show that the problem of mitigating torque chatter from indeterminate contact configurations is no harder than NP-hard.  Section~\ref{section:approximate-idyn} presents an inverse dynamics formulation that uses a rigid impact model and permits the contact force prediction problem to be convex.  This convexity will allow us to mitigate torque chatter from indeterminacy.
	
	Section~\ref{section:experiments} describes experimental setups for assessing the inverse dynamics formulations in the context of simulated robot control along multiple dimensions: accuracy of trajectory tracking; contact force prediction accuracy; general locomotion stability; and computational speed. Tests  are performed on terrains with varied friction and compliance.  Presented controllers are compared against both PID control and inverse dynamics control with sensed contact and perfectly accurate virtual sensors.  Assessment under both rigid and compliant contact models permits both exact and in-the-limit verification that controllers implementing these inverse dynamics approaches for control work as expected.  These experiments also examine behavior when modeling assumptions break down. Section~\ref{section:results} analyzes the findings from these experiments.

\subsection{The multi-body}
This paper centers around \emph{a multi-body}, which is the system of rigid bodies to which inverse dynamics is applied. The multi-body may come into contact with ``fixed'' parts of the environment (\emph{e.g.}, solid ground) which are sufficiently modeled as non-moving bodies--- this is often the case when simulating locomotion. Alternatively, the multi-body may contact other bodies, in which case effective inverse dynamics will require knowledge of those bodies' kinematic and dynamic properties --- necessary for manipulation tasks. 

  The articulated body approach can be extended to a multi-body to account for physically interacting with movable rigid bodies by appending the six degree-of-freedom velocities ($\vect{v}_{\textrm{cb}}$) and external wrenches ($\vect{f}_{\textrm{cb}}$) of each contacted rigid body to the velocity and external force vectors and by augmenting the generalized inertia matrix ($\mat{M}$) similarly: 
\begin{align}
	\vect{v}  &=  \tr{\begin{bmatrix} \tr{\vect{v}_{\textrm{robot}}} & \tr{\after\vect{v}_{\textrm{cb}}}\end{bmatrix}} \\
 	\vect{f}_{\textrm{ext}}  &= \tr{\begin{bmatrix} \tr{\vect{f}_{\textrm{robot}}} & \tr{\vect{f}_{\textrm{cb}}} \end{bmatrix}} \\
  \mat{M}  &= \begin{bmatrix} \mat{M}_{\textrm{robot}} & 0 \\
                                                			0 & \mat{M}_{\textrm{cb}} \end{bmatrix}
\end{align}

Without loss of generality, our presentation will hereafter consider only a single multi-body in contact with a static environment (excepting an example with a manipulator arm grasping a box in Section~\ref{section:experiments}).

\section{Background and related work}

\label{section:related-work}
	This section surveys the independent parts that are combined to formulate algorithms for calculating inverse dynamics torques with simultaneous contact force computation.  Section~\ref{section:complementarity-problems} discusses complementarity problems, a domain outside the purview of typical roboticists. Section~\ref{section:rigid-body-dynamics} introduces the rigid body dynamics model for Newtonian mechanics under generalized coordinates.  Section~\ref{section:rigid-body-contact} covers the rigid contact model, and unilaterally constrained contact.  Sections~\ref{section:modeling-Coulomb} --\ref{section:related-work:LCP} show how to formulate constraints on the rigid contact model to account for Coulomb friction and no-slip constraints. Section~\ref{section:related-work:QP} describe an algebraic impact model that will form the basis of one of the inverse dynamics methods. Section~\ref{section:contact-models:indeterminacy} describes the phenomenon of ``indeterminacy'' in the rigid contact model.  Lastly, Sections~\ref{section:related-work-id-robot}~and~\ref{section:related-work-id-biomech} discusses other work relevant to inverse dynamics with simultaneous contact force computation. 
	
\subsection{Complementarity problems}
\label{section:complementarity-problems}
Complementarity problems are a particular class of mathematical programming problems often used to model hard and rigid contacts. A nonlinear complementarity problem (NCP) is composed of three nonlinear constraints~\citep{Cottle:1992}, which taken together constitute a complementarity condition:
\begin{align}
\vect{x} & \ge \vect{0}\\
f(\vect{x}) & \ge \vect{0}\\
\tr{\vect{x}}f(\vect{x}) & = 0
\end{align}
where $\vect{x} \in \mathbb{R}^n$ and $f : \mathbb{R}^n \to \mathbb{R}^n$. Henceforth, we will use the following shorthand to denote a complementarity constraint:
\begin{equation}
0 \le a \perp b \ge 0
\end{equation}
which signifies that $a \ge 0, b \ge 0,$ and $a \cdot b = 0$.

A LCP, or linear complementarity problem ($\vect{r}, \mat{Q}$), where $\vect{r} \in \mathbb{R}^n$ and $\vect{Q} \in \mathbb{R}^{n \times n}$, is the linear version of this problem:
\begin{align}
\vect{w} & = \mat{Q}\vect{z} + \vect{r} \nonumber \\
\vect{w} & \ge \vect{0} \nonumber \\
\vect{z} & \ge \vect{0} \nonumber \\
\tr{\vect{z}}\vect{w} & = 0 \nonumber
\end{align}
for unknowns $\vect{z} \in \mathbb{R}^n, \vect{w} \in \mathbb{R}^n$.

Theory of LCPs has been established to a greater extent than for NCPs. For
example, theory has indicated certain classes of LCPs that are solvable, which includes both determining when a solution
does not exist and computing a solution, based on properties of the matrix $\mat{Q}$ (above). Such 
classes include positive definite matrices, positive semi-definite matrices, $P$-matrices, and $Z$-matrices, to name only a few; \citep{Murty:1988,Cottle:1992} contain far more information on complementarity problems, including algorithms for solving them. Given that the knowledge of NCPs (including algorithms for solving them) is still relatively thin, this article will relax NCP instances that correspond to contacting bodies to LCPs using a common practice, linearizing the friction cone.

Duality theory in optimization establishes a correspondence between LCPs
and quadratic programs (see~\citealp{Cottle:1992}, Pages 4 and 5) via the Karush-Kuhn-Tucker conditions; for example, positive semi-definite LCPs are equivalent to convex QPs. Algorithms for quadratic programs (QPs) can be used to solve LCPs, and vice versa.
 
\subsection{Rigid body dynamics}
\label{section:rigid-body-dynamics}

The multi rigid body dynamics (generalized Newton) equation governing the dynamics of a robot undergoing contact can be written in its generalized form as:
\begin{equation}
\label{eqn:generalized-newton}
\mat{M} \ddot{\vect{q}} = \vect{f}_C + \tr{\mat{P}}\vect{\tau} + \vect{f}_{\textrm{ext}} 
\end{equation}
Equation~\ref{eqn:generalized-newton} introduces the variables $\vect{f}_{\textrm{ext}} \in \mathbb{R}^m$ (``external'', non-actuated based, forces on the $m$ degree-of-freedom multibody, like gravity and Coriolis forces); $\vect{f}_C \in \mathbb{R}^m$ (contact forces); unknown actuator torques $\vect{\tau} \in \mathbb{R}^{nq}$ ($nq$ is the number of actuated degrees of freedom); and a binary selection matrix $\mat{P} \in \mathbb{R}^{nq \times m}$. If all of the degrees-of-freedom of the system are actuated $\mat{P}$ will be an identity matrix. For, \emph{e.g.}, legged robots, some variables in the system will correspond to unactuated, ``floating base'' degrees-of-freedom (DoF); the corresponding columns of the binary matrix $\mat{P} \in \mathbb{R}^{(m-6) \times m}$ will be zero vectors, while every other column will possess a single ``1''.

\subsection{Rigid contact model}
\label{section:rigid-body-contact}
This section will summarize existing theory of modeling non-impacting rigid contact and draws from~\cite{Stewart:1996,Trinkle:1997,Anitescu:1997}. Let us define a set of gap functions $\phi_i(\vect{x})$ (for $i=1,\ldots,q$), where gap function $i$ gives the signed distance between a link of the robot and another rigid body (part of the environment, another link of the robot, an object to be manipulated, \emph{etc.}) 

Our notation assumes independent coordinates $\vect{x}$ (and velocities $\vect{v}$ and accelerations $\dot{\vect{v}}$), and that generalized forces and inertias are also given in minimal coordinates.

 \begin{figure}[H]
 \centering
 \includegraphics[width=0.5\linewidth]{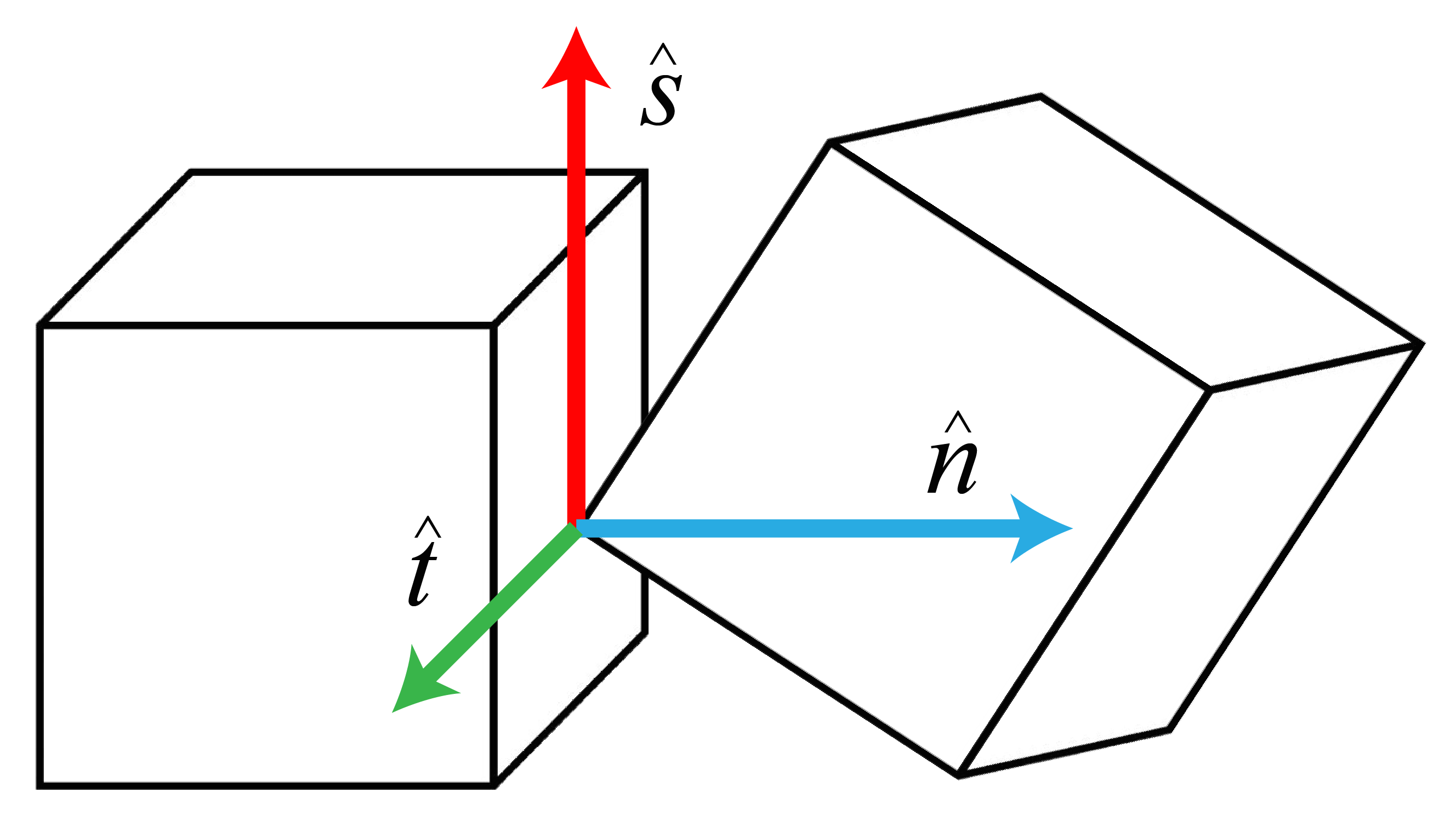}
\caption{The contact frame consisting of $\hat{\vect{n}}$, $\hat{\vect{s}}$, and , $\hat{\vect{t}}$ vectors corresponding to the normal, first tangential, and second tangential directions (for 3D) to the contact surface.}
 \end{figure}

The gap functions return a positive real value if the
bodies are separated, a negative real value if the bodies are geometrically intersecting, and zero if the bodies are in a ``kissing'' configuration. The rigid contact model specifies that bodies never overlap, \emph{i.e.}:
\begin{equation}
\phi_i(\vect{x}) \ge 0 \quad \textrm{ for } i=1,\ldots,q
\end{equation}
One or more points of contact between bodies is determined for two bodies in a kissing configuration ($\phi_i(\vect{x}) = 0$). \emph{For clarity of presentation, we will assume that each gap function corresponds to exactly one point of contact (even for disjoint bodies), so that $n = q$.}  In the absence of friction, the constraints on the gap functions are enforced 
by forces that act along the contact normal. Projecting these forces along
the contact normal yields scalars $f_{N_1}, \ldots, f_{N_n}$. The forces
should be compressive (\emph{i.e.}, forces that
can not pull bodies together), which is denoted by restricting
the sign of these scalars to be non-negative: 
\begin{equation}
f_{N_i} \ge 0 \quad \textrm{ for } i=1,\ldots,n
\end{equation}
  A \emph{complementarity} constraint keeps frictional contacts from doing
work: when the constraint is inactive (\mbox{$\phi_i > 0$}) no force is applied and when force is applied, the constraint must be active ($\phi_i = 0$). This
constraint is expressed mathematically as $\phi_i \cdot f_{N_i} = 0$. 
All three constraints can be combined into one equation using the notation in Section~\ref{section:complementarity-problems}:
\begin{align}
0 \le f_{N_i}\ \bot\ \phi_i (\vect{x}) \ge 0 \quad \textrm{ for } i=1,\ldots,n \label{eqn:normalCC}
\end{align}
These constraints can be differentiated with respect to time to yield velocity-level or acceleration-level constraints suitable for expressing the differential algebraic equations (DAEs), as an index 1 DAE:
\begin{align}
0 & \le f_{N_i}\ \bot\ \dot{\phi_i} (\vect{x}) \ge 0 \ \textrm{ if } \phi_i = 0 \quad \textrm{ for } i=1,\ldots,n \label{eqn:normalCC-vel} \\
0 & \le f_{N_i}\ \bot\ \ddot{\phi_i} (\vect{x}) \ge 0\ \textrm{ if } \phi_i = \dot{\phi}_i = 0 \ \textrm{ for } i=1,\ldots,n
\end{align}




\subsubsection{Modeling Coulomb friction}
\label{section:modeling-Coulomb}
Dry friction is often simulated 
using the Coulomb model, a relatively simple, empirically derived model that yields the approximate outcome of sophisticated physical  
interactions. Coulomb friction covers two regimes: sticking/rolling friction
(where the tangent velocity at a point of contact is zero) and sliding friction (nonzero tangent velocity at a point of contact). Rolling friction is distinguished from sticking friction by whether the bodies are moving relative to one another other than at the point of contact. 

There are many phenomena Coulomb friction does not model, including ``drilling friction'' (it is straightforward to augment computational models of Coulomb friction to incorporate this feature, as seen in~\citealp{Leine:2003}), the Stribeck effect~\citep{Stribeck:1902ul}, and viscous friction, among others. This article focuses only on Coulomb friction, because it captures important stick/slip transitions and uses only a single parameter; the LuGRe model~\citep{Do:2007}, for example, is governed by seven parameters, making system identification tedious.

Coulomb friction uses a unitless friction coefficient, commonly denoted $\mu$. If we define the tangent velocities and accelerations in 3D frames (located at the $i^{\textrm{th}}$ point of contact) as $v_{S_i}/v_{T_i}$ and $a_{S_i}/a_{T_i}$, respectively, and the tangent forces as $f_{S_i}$ and $f_{T_i}$, then the sticking/rolling constraints which are applicable exactly when $0 = v_{S_i} = v_{T_i}$, can be expressed via the Fritz-John optimality conditions~\citep{Mangasarian:1967,Trinkle:1997}:
\begin{align}
0 \le \mu_i^2 f_{N_i}^2 - f_{S_i}^2 - f_{T_i}^2\ \bot\ a_{S_i}^2 + a_{T_i}^2 \ge\ & 0 \label{eqn:sticking:fcone} \\
\mu_i f_{N_i} a_{S_i} + f_{S_i} \sqrt{a_{S_i}^2 + a_{T_i}^2} = &\ 0 \label{eqn:sticking:acc1} \\
\mu_i f_{N_i} a_{T_i} + f_{T_i} \sqrt{a_{S_i}^2 + a_{T_i}^2} = &\ 0 \label{eqn:sticking:acc2} 
\end{align}
These conditions ensure that the friction force lies within the friction cone (Equation~\ref{eqn:sticking:fcone}) and that the friction forces push against the tangential acceleration (Equations~\ref{eqn:sticking:acc1}--\ref{eqn:sticking:acc2}).

In the case of sliding at the $i^{\textrm{th}}$ contact ($v_{S_i} \ne 0$ or $v_{T_i} \ne 0$), the
constraints become:
\begin{align}
\mu_i^2 f_{N_i} - f_{S_i}^2 - f_{T_i}^2 \ge & 0 \label{eqn:sliding:fcone} \\
\mu_i f_{N_i} v_{S_i} + f_{S_i} \sqrt{v_{S_i}^2 + v_{T_i}^2} = &\ 0 \label{eqn:sliding:vel1} \\
\mu_i f_{N_i} v_{T_i} + f_{T_i} \sqrt{v_{S_i}^2 + v_{T_i}^2} = &\ 0 \label{eqn:sliding:vel2} 
\end{align}
Note that this case is only applicable if \mbox{$v_{S_i}^2 + v_{T_i}^2 > 0$}, so there is
no need to include such a constraint in Equation~\ref{eqn:sliding:fcone} (as 
was necessary in Equation~\ref{eqn:sticking:fcone}). 

The rigid contact model with Coulomb friction is subject to \emph{inconsistent configurations}~\citep{Stewart:2000a}, exemplified by Painlev\'{e}'s Paradox~\citep{Painleve:1895}, in which impulsive forces may be necessary to satisfy all constraints of the model \emph{even when the bodies are not impacting}. The acceleration-level dynamics can be approximated using finite differences; a first order approximation is often used (see, \emph{e.g.},~\citealp*{Posa:2012}), which moves the problem to the velocity/impulsive force domain.  Such an approach generally uses an \emph{algebraic collision law} (see~\citealp*{Chatterjee:1998}) to model all contacts, both impacting, as inelastic impacts; typical ``time stepping methods''~\citep{Moreau:1983} for simulating dynamics often treat the generalized coordinates as constant over the small timespan of contact/impact (\emph{i.e.}, a first order approximation); see, \emph{e.g.},~\cite{Stewart:2000}. \citeauthor{Stewart:1998} has shown that this approximation converges to the solution of the continuous time formulation as the step size tends to zero~(\citeyear{Stewart:1998}).

Upon moving to the velocity/impulsive force domain, Equations~\ref{eqn:sticking:fcone}--\ref{eqn:sliding:vel2} require a slight transformation to the equations:
\begin{align}
0 \le \mu_i^2 f_{N_i}^2 - f_{S_i}^2 - f_{T_i}^2\ \bot\ v_{S_i}^2 + v_{T_i}^2 \ge\ & 0 \label{eqn:velocity-impulse-Coulomb-start} \\
\mu f_{N_i} v_{S_i} + f_{S_i} \sqrt{v_{S_i}^2 + v_{T_i}^2} = &\ 0 \label{eqn:velocity-impulse-Coulomb-middle} \\
\mu f_{N_i} v_{T_i} + f_{T_i} \sqrt{v_{S_i}^2 + v_{T_i}^2} = &\ 0 \label{eqn:velocity-impulse-Coulomb-end} 
\end{align}
and there is no longer separate consideration of sticking/rolling and slipping contacts.

\subsubsection{No-slip constraints}
\label{section:related-work:no-slip}
If the Coulomb friction constraints are replaced by no-slip constraints, which
is a popular assumption in legged locomotion research, one must also use
the discretization approach; without permitting impulsive forces, slip can 
occur even with infinite friction~\citep{Lynch:1995}. The no-slip constraints
are then simply $v_{S_i} = v_{T_i} = 0$ (replacing Equations~\ref{eqn:velocity-impulse-Coulomb-start}--\ref{eqn:velocity-impulse-Coulomb-end}).

\subsubsection{Model for rigid, non-impacting contact with Coulomb friction}
\label{section:related-work:LCP}
The model of rigid contact with Coulomb friction for two bodies in non-impacting rigid contact at $\vect{p}$ can be summarized by the following equations:
\begin{align}
0 & \leq f_n\ \bot\ a_n \geq 0 \label{eqn:NCP1} \\
0 & \leq \mu^2 f_n^2 - f_s^2 - f_t^2\ \bot\ \sqrt{v_s^2 + v_t^2} \geq 0 \label{eqn:NCP2} \\
0 & = \mu f_n v_s + f_s \sqrt{v_s^2 + v_t^2} \label{eqn:NCP3} \\ 
0 & = \mu f_n v_t + f_t \sqrt{v_s^2 + v_t^2} \label{eqn:NCP4} \\
0 & = \mu f_n a_s + f_s \sqrt{a_s^2 + a_t^2} \label{eqn:NCP5} \\ 
0 & = \mu f_n a_t + f_t \sqrt{a_s^2 + a_t^2} \label{eqn:NCP6} 
\end{align} 
where $f_n$, $f_s$, and $f_t$ are the signed magnitudes of the contact force applied along the normal and two tangent directions, respectively; $a_n$ is the relative acceleration of the bodies normal to the contact surface; and $v_s$ and $v_t$ are the relative velocities of the bodies projected along the two tangent directions. The operator $\bot$ indicates that $\vect{a}\ \cdot\ \vect{b} = 0$, for vectors $\vect{a}$ and $\vect{b}$. Detailed interpretation of these equations can be found in~\cite{Trinkle:1997}; we present a summary below. Equation~\ref{eqn:NCP1} ensures that \1 only compressive forces are applied ($f_n \geq 0$); \2 bodies cannot interpenetrate ($a_n \geq 0$); and \3 no work is done for frictionless contacts ($f_n \cdot a_n = 0$). Equation~\ref{eqn:NCP2} models Coulomb friction: either the velocity in the contact tangent plane is zero---which allows frictional forces to lie within the friction cone---or the contact is slipping and the frictional forces must lie strictly on the edge of the friction cone. Equations~\ref{eqn:NCP3} and \ref{eqn:NCP4}---applicable to sliding contacts (\emph{i.e.}, those for which $v_s \neq 0$ or $v_t \neq 0$)---constrain friction forces to act against the direction of slip, while Equations~\ref{eqn:NCP5} and \ref{eqn:NCP6} constrain frictional forces for rolling or sticking contacts (\emph{i.e.}, those for which $v_s = v_t = 0$) to act against the direction of tangential acceleration.

These equations form a nonlinear complementarity problem~\citep{Cottle:1992}, and this problem 
may not possess a solution with nonimpulsive forces due to the existence of inconsistent configurations like Painlev\'{e}'s Paradox~\citep{Stewart:2000b}. This issue led to the movement to the impulsive force/velocity domain for modeling rigid contact, which can provably model the dynamics of such inconsistent configurations. 

A separate issue with the rigid contact model is that of indeterminacy, where 
multiple sets of contact forces and possibly multiple sets of accelerations---or velocities, if an impulse-velocity model is used---can satisfy the contact model equations. Our inverse dynamics approaches, which use rigid contact models, address inconsistency and, given some additional computation, can address indeterminacy (useful for controlled systems).

\subsubsection{Contacts without complementarity}
\label{section:related-work:QP}
Complementarity along the surface normal arises from Equation~\ref{eqn:normalCC} for contacting rigid bodies that are not impacting. For impacting bodies, complementarity conditions are unrealistic~\citep{Chatterjee:1999}. Though the distinction between impacting and non-impacting may be clear in free body diagrams and symbolic mathematics, the distinction between the two modes is arbitrary in floating point arithmetic. This arbitrary distinction has led researchers in dynamic simulation, for example, to use one model---either with complementarity or without---for both impacting and non-impacting contact.  

\cite{Anitescu:2006} described a contact model without complementarity (Equation~\ref{eqn:normalCC}) used for multi-rigid body simulation. \cite{Drumwright:2010b} and \cite{Todorov:2014}
rediscovered this model (albeit with slight modifications, addition of viscous friction, and guarantees of solution existence and non-negativity energy dissipation in the former); 
\cite{Drumwright:2010b} also motivated acceptability of removing the complementarity
condition based on the work by~\cite{Chatterjee:1999}. When treated as a simultaneous 
impact
model, the model is
consistent with first principles. Additionally, using arguments in~\cite{Smith:2012}, it can be shown that solutions of this model exhibit
symmetry. This impact model, under the assumption of inelastic impact---it is
possible to model partially or fully elastic impact as well, but one must then
consider numerous models of restitution, see, \emph{e.g.},~\cite{Chatterjee:1998}---will form the basis of the
inverse dynamics approach described in Section~\ref{section:approximate-idyn}.

The model is formulated as the convex quadratic program below. For consistency of presentation with the non-impacting rigid model described in the previous section, only a single impact point is considered.

\begin{tceqn}{Complementarity-free impact model (single point of contact)}
\begin{align}
                   &_\mathrm{dissipate\ kinetic\ energy\ maximally:} \notag \\
\minimize_{\after\vect{v},\vect{f}_n,\vect{f}_s,\vect{f}_t} \ & \frac{1}{2}\tr{\after\vect{v}}\mat{M}\after\vect{v}\label{KE:objective} \\
                   &_\mathrm{non-interpenetration:} \notag \\
\textrm{subject to: } & \mat{n}\after\vect{v} \geq \vect{0} \label{TODO}\\
                   &_\mathrm{compressive\ normal\ forces} \notag \\
                   & \vect{f}_n \geq \vect{0} \\
                    &_\mathrm{Coulomb\ friction:} \notag \\
                   & \mu^2 f_{n} \geq f_{s} + f_{t} \\
                   &_\mathrm{first-order\ dynamics:} \notag \\
                   & \after\vect{v} = \before\vect{v} + \inv{\mat{M}}(\tr{\mat{n}}\vect{f}_n + \tr{\mat{s}}\vect{f}_s + \tr{\mat{t}}\vect{f}_t ) \notag \\
\end{align}
\end{tceqn}


where $f_n$, $f_s$, and $f_t$ are the signed magnitudes of the impulsive forces applied along the normal and two tangent directions, respectively; $\before \vect{v} \in \mathbb{R}^m$ and $\after\vect{v} \in \mathbb{R}^m$ are the generalized velocities of the multi-body immediately before and after impact, respectively; $\mat{M} \in \mathbb{R}^{m \times m}$ is the generalized inertia matrix of the $m$ degree of freedom multi-body; and $\vect{n} \in \mathbb{R}^m$, $\vect{s} \in \mathbb{R}^m$, and $\vect{t} \in \mathbb{R}^m$ are generalized wrenches applied along the normal and two tangential directions at the point of contact (see Appendix~\ref{section:generalized-wrenches} for further details on these matrices).

The physical
interpretation of the impact model is straightforward: it selects impact forces that
maximize the rate of kinetic energy dissipation. Finally, we note that 
rigid impact models do not enjoy the same degree of community consensus as the 
non-impacting rigid contact models because 
three types of impact models (algebraic, incremental, and full
deformation) currently exist~\citep{Chatterjee:1998}, because simultaneous 
impacts and impacts between multi-bodies can
be highly sensitive to initial conditions~\citep{Ivanov:1995}, and because 
intuitive physical parameters for
capturing all points of the feasible impulse space do not yet exist~\citep{Chatterjee:1998}, among other issues. These difficulties lead this article to consider
only inelastic impacts, a case for which the feasible impulse space is
constrained.

\subsection{Contact force indeterminacy}
\label{section:contact-models:indeterminacy}
In previous work~\citep{Zapolsky:2013}, we found that indeterminacy in
the rigid contact model can be a significant problem for controlling quadrupedal robots (and, presumably, hexapods, \emph{etc.}) 
by yielding torques that switch rapidly between various actuators (torque chatter). The
problem can occur in bipedal walkers; for example,
\cite{Collins:2001} observed instability from rigid contact indeterminacy in passive 
walkers.  Even manipulators may also experience the phenomenon of rigid contact indeterminacy, indicated by torque chatter.


Rigid contact configurations can be indeterminate in terms of forces; for the
example of a table with all four legs perfectly touching a ground plane, 
infinite enumerations of force configurations satisfy the contact model (as
discussed in~\citealp*{Mirtich:1996vt}), although the 
accelerations predicted by the model are unique. Other rigid contact configurations 
can be indeterminate in terms of predicting different accelerations/velocities through multiple sets
of valid force configurations.    We present two methods of mitigating indeterminacy in this article (see Sections~\ref{section:no-slip:indeterminacy} and~\ref{section:phaseII}).  Defining a manner by which actuator torques evolve over time, or selecting a preferred distribution of contact forces may remedy the issues resulting from indeterminacy.

\subsection{Contact models for inverse dynamics in the context of robot control}
\label{section:related-work-id-robot}



This section focuses on ``hard'', by which we mean perfectly rigid, models
of contact incorporated into inverse dynamics and whole body control for
robotics. We are unaware of research that has attempted to combine inverse dynamics
with compliant contact (one possible reason for absence of such work is that such compliant models can 
require significant parameter tuning for accuracy and to prevent prediction of large contact forces).

\cite{Mistry:2010} developed a fast inverse dynamics control framework for legged robots in contact with rigid environments under the assumptions that feet do not slip. \cite{Righetti:2013} extended 
this work with a framework that permits quickly optimizing a mixed 
linear/quadratic function of motor torques and contact forces using fast
linear algebra techniques. 
\cite{Hutter:2012} also uses this 
formulation in an operational space control scheme, simplifying the contact mathematics by assuming contacts are sticking. \citeauthor{Mistry:2010,Righetti:2013,Hutter:2012} demonstrate effective trajectory tracking performance on 
quadrupedal robots.

The inverse dynamics approach of \cite{Ames:2013a} assumes sticking impact upon contact with the ground and immediate switching of support to the new contact, while enforcing a unilateral constraint of the normal forces and predicting no-slip frictional forces.  

\cite{Kuindersma:2014} use a no-slip constraint but allow for bounded violation of that constraint in order to avoid optimizing over an infeasible or inconsistent trajectory.

\cite{Stephens:2010} incorporate a contact model into an inverse dynamics formulation for dynamic balance force control. Their approach uses a quadratic program (QP) to estimate contact forces quickly on a simplified model of a bipedal robot's dynamics.  Newer work by \cite{Feng:2013} builds on this by approximating the friction cone with a circumscribed friction pyramid.

\cite{Ott:2011} also use an optimization approach for balance, modeling contact to distribute forces among a set of pre-defined contacts to enact a generalized wrench on a simplified model of a biped; their controller seeks to minimize the Euclidian norm of the predicted contact forces to mitigate slip.
In underconstrained cases (where multiple solutions to the inverse dynamics with contact system exist), \cite{Saab:2013} and \cite{Zapolsky:2013} use a multi-phase QP 
formulation for bipeds and quadrupeds, respectively.  \citeauthor{Zapolsky:2013} mitigates the indeterminacy in the rigid contact model by selecting a solution that minimizes total actuator torques, while \citeauthor{Saab:2013} use the rigid
contact model in the context of cascades of QPs to perform
several tasks in parallel (\emph{i.e.}, whole body control).  The latter work primarily considers
contacts without slip, but does describe modifications that would incorporate 
Coulomb friction (inconsistent and indeterminate rigid
contact configurations are not addressed). \cite{Todorov:2014} uses the same contact model (to be described below) but without using a two-stage solution; that approach uses regularization to make the optimization problem strictly convex (yielding a single, globally optimal solution).
None of~\citeauthor{Saab:2013,Zapolsky:2013,Todorov:2014} utilize the \emph{complementarity constraint} (\emph{i.e.}, $f_N\ \bot\ \phi$ 
in Equation~\ref{eqn:normalCC}) in their formulation.
Zapolsky~\emph{et al.} and Todorov motivate dropping this constraint in favor of maximizing energy dissipation through contact, an assumption that they show performs reasonably in practice~\citep{Drumwright:2010b,Todorov:2011}.  
 
\subsection{Contact models for inverse dynamics in the context of biomechanics}
\label{section:related-work-id-biomech}
Inverse dynamics is commonly applied in biomechanics 
to determine approximate net torques at anatomical joints for observed motion
capture and force plate data. Standard Newton-Euler inverse dynamics
algorithms (as described in~\citealp*{Featherstone:2008}) are applied; 
least squares is required because the problem is overconstrained. Numerous
such approaches are found in biomechanics literature, including~\citep{Kuo:1998,Hatze:2002,Blajer:2007,Bisseling:2006,Yang:2007,Van-Den-Bogert:2008,Sawers:2010}.
These force plate based approaches necessarily limit the environments
in which the inverse dynamics computation can be conducted.


\section{Discretized inverse dynamics}
\label{section:discretized-idyn}
We discretize inverse dynamics because the resolution to rigid contact models both without slip and with Coulomb friction can require impulsive forces even when there are no impacts (see Section~\ref{section:modeling-Coulomb}). This choice will imply that the dynamics are accurate to only first order, but that approximation should limit modeling error considerably for typical control loop rates~\citep{Zapolsky:2015}.

As noted above, dynamics are discretized using a first order approximation to acceleration. Thus, the solution to the equation of motion $\dot{\vect{v}} = \inv{\mat{M}}\vect{f}$ over $[t_0, t_f]$ is approximated by $\after\vect{v} = \before \vect{v} + \Delta t\before \inv{\mat{M}} \beforef \vect{f}$, where $\Delta t = (t_f - t_0)$.  We use the superscript ``$^+$'' to denote that a term is evaluated at $t_f$ and the superscript ``$^-$'' is applied to denote that a term is computed at $t_0$.  For example, the generalized inertia matrix $\mat{M}$ has the superscript ``$^-$''  and the generalized post-contact velocity ($\after\vect{v}$) has the superscript ``$^+$''.  We will hereafter adopt the convention that application of a superscript once will indicate implicit evaluation of that quantity at that time thereafter (unless another superscript is applied). For example, we will continue to treat $\mat{M}$ as evaluated at $t_0$ in the remainder of this paper. 

The remainder of this section describes how contact constraints should be determined for discretized inverse dynamics.

\subsection{Incorporating contact into planned motion}
First, we note that the inverse dynamics controller attempts to realize a planned motion. That planned motion must account for pose data and geometric models of objects in the robot's environment. If planned motion is inconsistent with contact constraints, e.g., the robot attempts to push through a wall, undesirable behavior will clearly result. Obtaining accurate geometric data (at least for novel objects) and pose data are presently challenging problems; additional work in inverse dynamics control with predictive contact is necessary to address contact compliance and sensing uncertainty. 

\subsection{Incorporating contact constraints that do not coincide with control loop period endpoint times}
Contact events---making or breaking contact, transitioning from sticking to sliding or vice versa---do not generally coincide with control loop period endpoint times.  Introducing a contact constraint ``early'', i.e., before the robot comes into contact with an object, will result in a poor estimate of the load on the robot (as the anticipated link-object reaction force will be absent). Introducing a contact constraint ``late'', i.e., after the robot has already contacted the object, implies that an impact occurred; it is also likely that actuators attached to the contacted link and on up the kinematic chain are heavily loaded, resulting in possible damage to the robot, the environment, or both. Figure~\ref{fig:contacts} depicts both of these scenarios for a walking bipedal robot.

\begin{figure}[htbp]
\centering
\includegraphics[width=\linewidth]{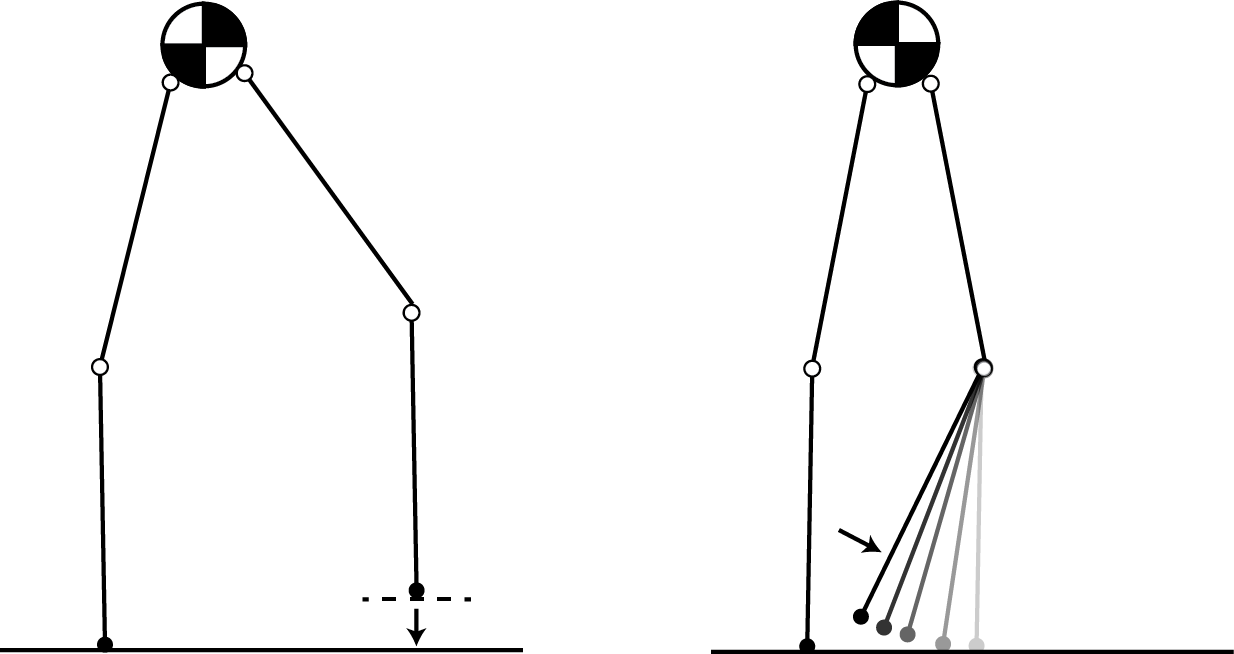}
\caption{If the contact constraint is introduced early (left figure, constraint depicted using dotted line) the anticipated load will be wrong. The biped will pitch forward, possibly falling over in this scenario. If the contact constraint is introduced late, an impact may occur while the actuators are loaded. The biped on the right is moving its right lower leg toward a foot placement; the impact as the foot touches down is prone to damaging the loaded powertrain.}
\label{fig:contacts}
\end{figure}

We address this problem by borrowing a \emph{constraint stabilization}~\citep{Ascher:1995a} approach from~\cite{Anitescu:2004}, which is itself a form of \emph{Baumgarte Stabilization}~\citep{Baumgarte:1972}. Recalling that two bodies are separated by signed distance $\phi(.)$, constraints on velocities are determined such that . 

To realize these constraints mathematically, we first convert Equation~\ref{eqn:normalCC-vel} to a discretized form:
\begin{align}
0 & \le f_{N_i}(t)\ \bot\ \dot{\phi_i} (\vect{x}(t+\Delta t)) \ge 0 \ \textrm{ if } \phi_i(t) = 0  \label{eqn:normalCC-disc}  \\ 
& \qquad \textrm{ for } i=1,\ldots,n \nonumber
\end{align}
This equation specifies that a force is to be found such that applying the force between one of the robot's links and an object, \emph{already in contact at $t$}, over the interval $[t, t+\Delta t]$ yields a relative velocity indicating sustained contact or separation at $t+\Delta t$. We next incorporate the signed distance between the bodies:
\begin{align}
0 & \le f_{N_i}(t)\ \bot\ \dot{\phi_i} (\vect{x}(t+\Delta t)) \ge -\frac{\phi(\vect{x}(t))}{\Delta t} \label{eqn:normalCC-disc}  \\ 
& \qquad \textrm{ for } i=1,\ldots,n \nonumber
\end{align}
The removal of the conditional makes the constraint always active. Introducing a constraint of this form means that forces may be applied in some scenarios when they should not be (see Figure~\ref{fig:contact-spring} for an example). Alternatively, constraints introduced before bodies contact can be contradictory, making the problem infeasible. Existing proofs for time stepping simulation approaches indicate that such issues disappear for sufficiently small integration steps (or, in the context of inverse dynamics, sufficiently high frequency control loops); see~\cite{Anitescu:2004}, which proves that such errors are uniformly bounded in terms of the size of the time step and the current value of the velocity. 

\begin{figure}[htbp]
\centering
\includegraphics[width=.7\linewidth]{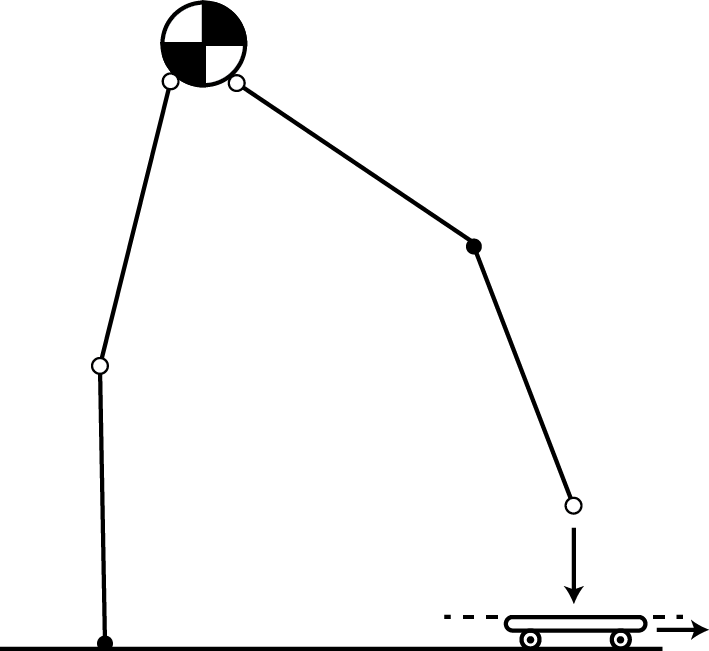}
\caption{An example of a contact constraint that overly constrains the motion between two disjoint bodies (the right foot and the skateboard). The contact constraint will keep the foot from moving below the dotted line. However, if the foot moves quickly downward and the skateboard moves quickly to the right (in $\Delta t$ time), inverse dynamics may predict, incorrectly, that a contact force will be applied to the foot. It should be apparent that these problems disappear as $\Delta t \to 0$, i.e., as the control loop frequency becomes sufficiently high.}
\label{fig:contact-spring}
\end{figure}

\subsection{Computing points of contact between geometries}
Given object poses data and geometric models, points of contact between robot links and environment objects can be computed using closest features. The particular algorithm used for computing closest features is dependent upon both the representation (e.g., implicit surface, polyhedron, constructive solid geometry) and the shape (e.g., sphere, cylinder, box). Algorithms and code can be found in sources like~\cite{Ericson:2005} and \mbox{\url{http://geometrictools.com}}. Figure~\ref{fig:contact-data} depicts the procedure for determining contact points and normals for two examples: a sphere vs. a half-space and for a sphere vs. a sphere.

For disjoint bodies like those depicted in Figure~\ref{fig:contact-data}, the contact point can be placed anywhere along the line segment connecting the closest features on the two bodies. Although the illustration depicts the contact point as placed at the midpoint of this line segment, this selection is arbitrary. Whether the contact point is located on the first body, on the second body, or midway between the two bodies, no formula is ``correct'' while the bodies are separated and every formula yields the same result when the bodies are touching.   

\begin{figure}[htbp]
\centering
\includegraphics[width=.4\linewidth]{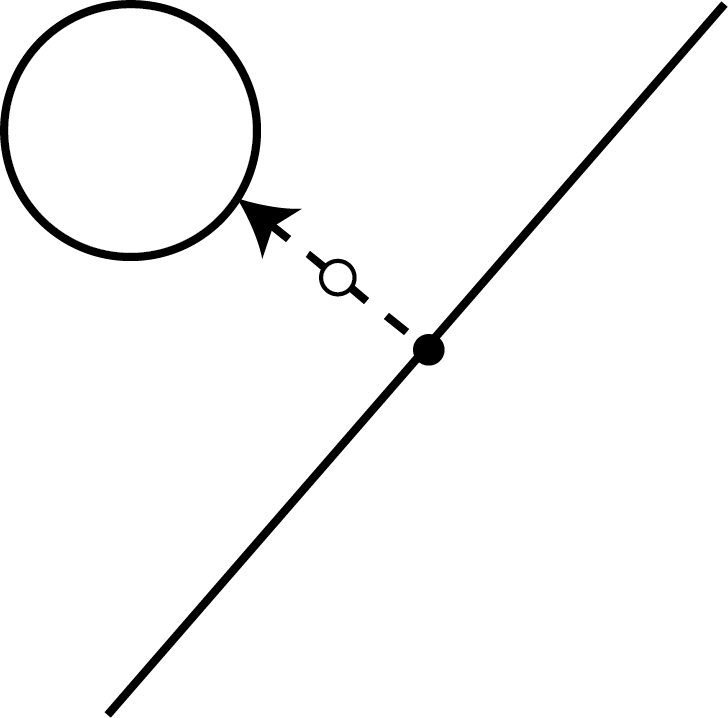} \hspace{.5cm}
\includegraphics[width=.4\linewidth]{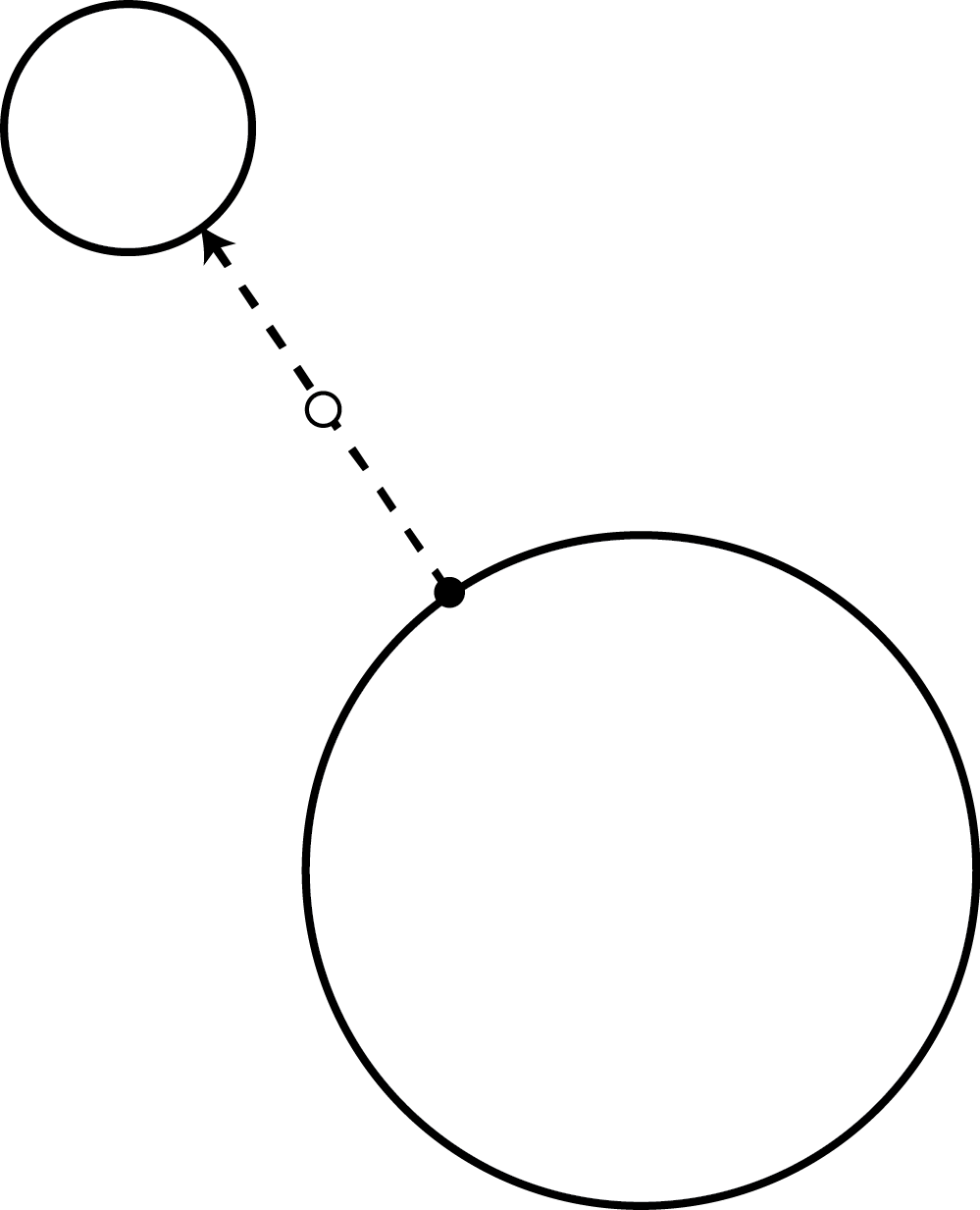}
\caption{Depiction of the manner of selecting points of contact and surface normals for disjoint bodies with spherical/half-space (left) and spherical/spherical geometries (right). Closest points on the objects are connected by dotted line segments. Surface normals are depicted with an arrow. Contact points are drawn as white circles with black outlines.}
\label{fig:contact-data}
\end{figure}

 \section{Inverse dynamics with no-slip constraints}
 \label{section:ID:no-slip}
 
Some contexts where inverse dynamics may be used (biomechanics, legged locomotion) may assume absence of slip (see, \emph{e.g.},~\citealp*{Righetti:2011, Zhao:2014}).  This section describes an inverse dynamics approach that computes reaction forces from contact using the non-impacting rigid contact model with no-slip constraints. Using no-slip constraints results in a symmetric, positive semidefinite LCP.  Such problems are equivalent to convex QPs by duality theory in optimization (see \citealp{Cottle:1992}), which implies polynomial time solvability. Convexity also permits mitigating indeterminate contact configurations, as will be seen in Section~\ref{section:no-slip:indeterminacy}.  This formulation inverts the rigid contact problem in a practical sense and is derived from first principles.

We present two algorithms in this section: Algorithm~\ref{alg:find-indices} ensures the no-slip constraints on the contact model are non-singular and thus guarantees that the inverse dynamics problem with contact is invertible; Algorithm~\ref{alg:PPM} presents a method of mitigating torque chatter from indeterminate contact (for contexts of inverse dynamics based control) by \emph{warm-starting}~\citep{Nocedal:2006} the solution for the LCP solver with the last solution.


\subsection{Normal contact constraints}
\label{sec:jacobian-evaluation}
The  equation below extends Equation~\ref{eqn:normalCC-vel} to multiple points
of contact (via the relationship $\dot{\vect{\phi}} = \mat{N}\vect{v}$), where $\mat{N} \in \mathbb{R}^n$ is the matrix of generalized wrenches along the contact normals (see Appendix~\ref{section:generalized-wrenches}): 
\begin{equation}
-\frac{\before \vect{\phi}}{\Delta t} \le \before \mat{N}\after \vect{v}\ \bot\ f_N \ge \vect{0} \label{eqn:normalCC:A-P}
\end{equation}
Because $\vect{\phi}$ is clearly time-dependent and the control
loop executes at discrete points in time, $\mat{N}$ must be treated as constant
over a time interval. 
$\before \mat{N}$ indicates that points of contact
are drawn from the current configuration of the environment and multi-body.
Analogous to time-stepping approaches for rigid body simulations with rigid
contact, \emph{all possible points of contact between rigid bodies over the interval $t_0$ and $t_f$ can be
incorporated into $\mat{N}$ as well}: as in time stepping approaches for simulation, it may 
not be necessary to apply forces at all of these points (the approaches implicitly can 
treat unnecessary points of contact as inactive, though additional computation will be necessary). \cite{Stewart:1998} showed that such an 
approach will converge
to the solution of the continuous time dynamics as $\Delta t = (t_f - t_0) \to 0$.  Given a sufficiently high control rate, $\Delta t$ will be small
and errors from assuming constant $\mat{N}$ over this interval should become
negligible.
 
\subsection{Discretized rigid body dynamics equation}
The discretized version of Equation~\ref{eqn:generalized-newton}, now separating contact forces into normal $(\before \mat{N}$) and tangential wrenches ($\before \mat{S}$ and $\before \mat{T}$ are matrices of generalized wrenches along the first and second contact tangent directions for all contacts) is:
\begin{align}
\label{eqn:discretized-generalized-newton}
\mat{M}(\after \vect{v} - \before \vect{v}) =& \tr{\mat{N}}\vect{f}_N + \tr{\mat{S}}\vect{f}_S + \tr{\mat{T}}\vect{f}_T - \ldots \\
& \tr{\mat{P}}\vect{\tau} + \Delta t \before \vect{f}_{\textrm{ext}} \notag 
\end{align}
Treating inverse dynamics at the velocity level is necessary to avoid the
inconsistent configurations that can arise in the rigid contact model when forces are required to be
non-impulsive (\citealp*{Stewart:2000a}, as also noted in Section~\ref{section:modeling-Coulomb}). As noted above, Stewart has shown that for sufficiently small $\Delta t$, $\after\vect{v}$ converges to the solution of the continuous time dynamics and contact equations~(\citeyear{Stewart:1998}).

\subsection{Inverse dynamics constraint}
The inverse dynamics constraint is used to specify the desired velocities only at actuated joints:
\begin{equation}
\mat{P}\after \vect{v} = \dot{\vect{q}}_{\textrm{des}} \label{eqn:invdyn}
\end{equation}
Desired velocities $\dot{\vect{q}}_{\textrm{des}}$ are calculated as:
\begin{equation}
\dot{\vect{q}}_{\textrm{des}} \equiv \dot{\vect{q}} + \Delta t \ddot{\vect{q}}_{\textrm{des}} 
\end{equation}
\subsection{No-slip (infinite friction) constraints}
\label{section:no-slip}

	Utilizing the first-order discretization (revisit Section~\ref{section:related-work:no-slip} to see why this is necessary), preventing tangential slip
at a contact is accomplished by using the constraints:
\begin{align}
\mat{S}\after \vect{v} & = \vect{0} \label{eqn:no-slip1} \\
\mat{T}\after \vect{v} & = \vect{0} \label{eqn:no-slip2} 
\end{align}
These constraints indicate that the velocity in the tangent plane is zero at 
time $t_f$; we will find the matrix representation to be more convenient for
expression as quadratic programs and linear complementarity problems than:
\begin{align*}
v_{s_i}\after = v_{t_i}\after = 0 \textrm{ for } i=1,\ldots,n,\\
\end{align*}
\emph{i.e.}, the notation used in Section~\ref{section:modeling-Coulomb}.
All presented equations are compiled below: 

\begin{tceqn}{Complementarity-based inverse dynamics without slip}
\begin{align}
                   &_\mathrm{non-interpenetration,\ compressive\ force,\ and} \notag \\
& _\mathrm{normal\ complementarity\ constraints:} \notag \\
                   & \vect{0} \le \vect{f}_N\ \bot\ \mat{N}\after\vect{v} \ge -\frac{\before \vect{\phi}}{\Delta t} \notag \\
                   &_\mathrm{no-slip\ constraints:} \notag \\
                   & \mat{S}\after\vect{v} = \vect{0} \notag \\
                   & \mat{T}\after\vect{v} = \vect{0} \notag \\
                   &_\mathrm{inverse\ dynamics:} \notag \\
                   & \mat{P}\after\vect{v} = \dot{\vect{q}}_{\textrm{des}} \notag \\
                   &_\mathrm{first-order\ dynamics:} \notag \\
                   & \after\vect{v} = \before\vect{v} + \inv{\mat{M}}(\tr{\mat{N}}\vect{f}_N + \tr{\mat{S}}\vect{f}_S + \ldots \notag \\
                   & \qquad \quad \tr{\mat{T}}\vect{f}_T - \tr{\mat{P}}\vect{\tau} + \Delta t \vect{f}_{ext}) \notag 
\end{align}
\end{tceqn}

Combining Equations~\ref{eqn:normalCC:A-P}--\ref{eqn:invdyn}, \ref{eqn:no-slip1}, and \ref{eqn:no-slip2} into a \emph{mixed linear complementarity problem} (MLCP, see Appendix~\ref{section:LCPs}) yields:  
\footnotesize
\begin{align}
&\hspace{-1.5mm}\begin{bmatrix}
\mat{M} & -\tr{\mat{P}} & -\tr{\mat{S}} & -\tr{\mat{T}} & -\tr{\mat{N}} \\
\mat{P} & \mat{0} & \mat{0} & \mat{0} & \mat{0} \\
\mat{S} & \mat{0} & \mat{0} & \mat{0} & \mat{0} \\
\mat{T} & \mat{0} & \mat{0} & \mat{0} & \mat{0} \\
\mat{N} & \mat{0} & \mat{0} & \mat{0} & \mat{0} \\
\end{bmatrix}
\begin{bmatrix}
\after \vect{v} \\
\vect{\tau} \\
\vect{f}_S \\
\vect{f}_T \\
\vect{f}_N \\
\end{bmatrix}
\hspace{-1mm} + \hspace{-1mm}
\begin{bmatrix}
\vect{\kappa} \\
-\dot{\vect{q}}_{\textrm{des}} \\
\vect{0} \\
\vect{0} \\
\frac{\before \vect{\phi}}{\Delta t}
\end{bmatrix} \hspace{-1mm}
= \hspace{-1mm}
\begin{bmatrix}
\vect{0} \\
\vect{0} \\
\vect{0} \\
\vect{0} \\
\vect{w}_N
\end{bmatrix} \label{eqn:no-slip-MLCP1} \\
& \vect{f}_N \ge \vect{0}, \vect{w}_N \ge \vect{0}, \tr{\vect{f}}_N \vect{w}_N = 0 \label{eqn:no-slip-MLCP2} 
\end{align}
\normalsize
where $\vect{\kappa} \triangleq -\Delta t \vect{f}_{\textrm{ext}} - \mat{M}\before \vect{v}$.  We define the MLCP block matrices---in the form of Equations~\ref{eqn:MLCP-begin}--\ref{eqn:MLCP-end} from Appendix~\ref{section:LCPs}---and draw from Equations~\ref{eqn:no-slip-MLCP1} and \ref{eqn:no-slip-MLCP2} to yield:
\begin{align*}
\mat{A} & \equiv 
\begin{bmatrix}
 \mat{M} & -\tr{\mat{P}} & -\tr{\mat{S}} &-\tr{\mat{T}} \\
 \mat{P} &   \mat{0}  &   \mat{0}  &   \mat{0}  \\
 \mat{S} &   \mat{0}  &   \mat{0}  &   \mat{0}  \\
 \mat{T} &   \mat{0}  &   \mat{0}  &   \mat{0}  
\end{bmatrix} &
\mat{C} &\equiv 
\begin{bmatrix}
-\tr{\mat{N}} \\
 \mat{0} \\
 \mat{0} \\
 \mat{0} 
\end{bmatrix} \\
\mat{D} &\equiv -\tr{\mat{C}} &
\mat{B} & \equiv \mat{0}  \\
\vect{x} &\equiv
\begin{bmatrix}
   \after \vect{v}  \\ \vect{\tau} \\ \vect{f}_S \\ \vect{f}_T
\end{bmatrix} & \vect{g} &\equiv
\begin{bmatrix}
-\vect{\kappa} \\ \dot{\vect{q}}_\textrm{des} \\ \vect{0} \\ \vect{0}
\end{bmatrix}\\
\vect{y} &\equiv
  \vect{f}_N & \vect{h} &\equiv \frac{\before \vect{\phi}}{\Delta t}
\end{align*}

	Applying Equations~\ref{eqn:MLCP-LCP1} and \ref{eqn:MLCP-LCP2} (again see Appendix~\ref{section:LCPs}), we transform the MLCP to LCP $(\vect{r},\mat{Q})$.  Substituting in variables from the no-slip inverse dynamics model and then simplifying yields:
\begin{align}
\mat{Q} &\equiv \tr{\mat{C}}\inv{\mat{A}}\mat{C} \\
\vect{r} &\equiv \frac{\before \vect{\phi}}{\Delta t} + \tr{\mat{C}}\inv{\mat{A}}\vect{g}
\end{align}
The definition of matrix $\mat{A}$ from above may be singular, which would prevent inversion, and thereby, conversion from the MLCP to an LCP.  We defined $\mat{P}$ as a selection matrix with full row rank, and the generalized inertia ($\mat{M}$) is symmetric, positive definite.  If $\mat{S}$ and $\mat{T}$ have full row rank as well, or we identify the largest subset of row blocks of $\mat{S}$ and $\mat{T}$ such that full row rank is attained, $\mat{A}$ will be invertible as well
(this can be seen by applying blockwise matrix inversion identities). Algorithm~\ref{alg:find-indices} performs the task of ensuring that matrix $\mat{A}$ is invertible. Removing the linearly dependent constraints from the $\mat{A}$ matrix does not affect the solubility of the MLCP, as proved in Appendix~\ref{section:mlcp-indep-constraints}.

From \cite{Bhatia:2007}, a matrix of $\mat{Q}$'s form must be non-negative definite, \emph{i.e.}, either positive-semidefinite (PSD) or positive definite (PD).  $\mat{Q}$ is the right product of $\mat{C}$ with its transpose about a symmetric PD matrix, $\mat{A}$. Therefore, $\mat{Q}$ is symmetric and either PSD or PD.

The singularity check on Lines~\ref{line:XTiMX:1} and~\ref{line:XTiMX:2} of Algorithm~\ref{alg:find-indices} is most quickly performed using Cholesky factorization; if the factorization is successful, the matrix is non-singular. Given that $\mat{M}$ is non-singular (it is symmetric, PD), the maximum size of $\mat{X}$ in Algorithm~\ref{alg:find-indices} is $m \times m$; if $\mat{X}$ were larger, it would be singular. 

The result is that the time complexity of Algorithm~\ref{alg:find-indices} is dominated by Lines~\ref{line:XTiMX:1} and~\ref{line:XTiMX:2}. As $\mat{X}$ changes by at most one row and one column per Cholesky factorization, singularity can be checked by  $O(m^2)$ updates to an initial $O(m^3)$ Cholesky factorization. The overall time complexity is \mbox{$O(m^3 + nm^2)$}.

\begin{algorithm}[H]
\caption{\textsc{Find-Indices}($\mat{M}, \mat{P}, \mat{S}, \mat{T})$, determines the row indices ($\mathcal{S}$, and $\mathcal{T}$) of $\mat{S}$ and $\mat{T}$ such that the matrix $\mat{A}$ (Equation~\ref{eqn:MLCP-begin} in Appendix~\ref{section:LCPs}) is non-singular. \label{alg:find-indices}}
\begin{algorithmic}[1]
\State $\mathcal{S} \leftarrow \emptyset$
\State $\mathcal{T} \leftarrow \emptyset$
\For {$i = 1, \ldots, n$}
\Comment $n$ is the number of contacts
\State $\mathcal{S}^* \leftarrow \mathcal{S} \cup \{ i \}$
\State Set $\mat{X} \leftarrow \begin{bmatrix} \tr{\mat{P}} & \tr{\mat{S}}_{\mathcal{S}^*} & \tr{\mat{T}}_\mathcal{T}\end{bmatrix}$
\If {$\tr{\mat{X}}\inv{\mat{M}}\mat{X}$ not singular} \label{line:XTiMX:1}
\State $\mathcal{S} \leftarrow \mathcal{S}^*$
\EndIf
\State $\mathcal{T}^* \leftarrow \mathcal{T} \cup \{ i \}$
\State Set $\mat{X} \leftarrow \begin{bmatrix} \tr{\mat{P}} & \tr{\mat{S}}_\mathcal{S} & \tr{\mat{T}}_{\mathcal{T}^*} \end{bmatrix}$
\If {$\tr{\mat{X}}\inv{\mat{M}}\mat{X}$ not singular} \label{line:XTiMX:2}
\State $\mathcal{T} \leftarrow \mathcal{T}^*$
\EndIf
\EndFor
\State \Return $\{\mathcal{S}, \mathcal{T} \}$
\end{algorithmic}
\end{algorithm}

\subsection{Retrieving the inverse dynamics forces}
\label{section:retrieving-forces}
Once the contact forces have been determined, one solves Equations~\ref{eqn:discretized-generalized-newton} and~\ref{eqn:invdyn} for $\{ \after \vect{v}, \vect{\tau} \}$, thereby obtaining the inverse dynamics forces. While the LCP
is solvable, it is possible that the desired accelerations
are inconsistent. As an example, consider a legged robot standing on a ground
plane without slip (such a case is similar to, but not identical to infinite friction, as noted in Section~\ref{section:related-work:no-slip}), and attempting to splay its legs outward while remaining in contact with the ground. Such cases can be readily identified by verifying that $\mat{N}\after \vect{v} \ge -\frac{\before \vect{\phi}}{\Delta t}$. If this constraint is not met, consistent
desired accelerations can be determined \emph{without re-solving the LCP}. 
For example, one could determine accelerations that deviate minimally from
the desired accelerations by solving a quadratic program:
\begin{align}
\minimize_{\after \vect{v}, \vect{\tau}}\ & ||\mat{P}\after \vect{v} - \dot{\vect{q}}_{\textrm{des}}|| \\
\textrm{subject\ to:\ } & \mat{N}\after \vect{v} \ge -\frac{\before \vect{\phi}}{\Delta t} \\
&\mat{S}\after \vect{v} = \vect{0} \\
&\mat{T}\after \vect{v} = \vect{0} \\
                        & \mat{M}\after \vect{v} = \mat{M} \before \vect{v} + \tr{\mat{N}}\vect{f}_N + \tr{\mat{S}}\vect{f}_S + \ldots \\
& \qquad \qquad \tr{\mat{T}}\vect{f}_T + \tr{\mat{P}}\vect{\tau} + \Delta t \vect{f}_{\textrm{ext}}
\end{align} 
This QP is always feasible: $\vect{\tau} = \vect{0}$ ensures that\\ $\mat{N}\after \vect{v} \ge -\frac{\before \vect{\phi}}{\Delta t}$, $\mat{S}\after \vect{v} = \vect{0}$, and $
\mat{T}\after \vect{v} = \vect{0}$.

\subsection{Indeterminacy mitigation}
\label{section:no-slip:indeterminacy}
We warm start a pivoting LCP solver (see Appendix~\ref{section:pivoting}) to bias the solver toward applying forces at the same points of contact (see Figure~\ref{fig:warm-starting})---tracking points of contact using the rigid body equations of motion---as were active on the previous inverse dynamics call (see Algorithm~\ref{alg:PPM}). \cite{Kuindersma:2014} also use warm starting to solve a different QP resulting from contact force prediction. However, we use warm starting to address indeterminacy in the rigid contact model while \citeauthor{Kuindersma:2014} use it to generally speed the solution process. 
 
 \begin{tcfigure}{Warm-Starting Example}
 \centering
 \begin{tabular}{ccc}
 Iteration $i$ & Iteration $i+1$ & Iteration $i+2$ \\
 \includegraphics[width=0.3\linewidth]{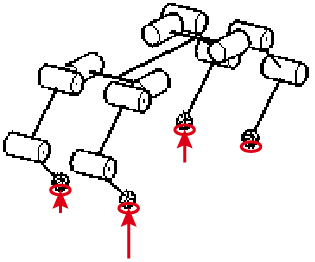}& \includegraphics[width=0.3\linewidth]{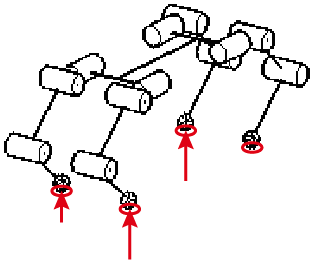} &\includegraphics[width=0.3\linewidth]{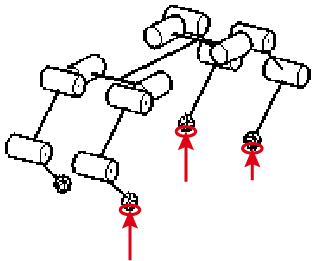}
\end{tabular}
\small
\begin{flushleft}
\emph{Left} (``cold start''): with four active contacts, the pivoting solver chooses three arbitrary \emph{non-basic} indices (in $\overline{\beta}$, see Appendix~\ref{section:pivoting}) to solve the LCP and then returns the solution. The solution applies the majority of upward force to two feet and applies a smaller amount of force to the third.  
\newline \emph{Center} (``warm start''):  With four active contacts, the pivoting solver chooses the \emph{same} three non-basic indices as the last solution to attempt to solve the LCP.  The warm-started solution will distribute upward forces similarly to the last solution, tending to provide consecutive solves with continuity over time.  
\newline \emph{Right} (``cold start''): one foot of the robot has broken contact with the ground; there are now three active contacts.  The solver returns a new solution, applying the majority of normal force to two legs, and applying a small amount of force to the third.
\addtocounter{figure}{-1}\refstepcounter{figure}\label{fig:warm-starting}
 \end{flushleft}
 \end{tcfigure}
 
Using warm starting, Algorithm~\ref{alg:PPM} will find a solution that predicts contact forces applied at the exact same points on the last iteration \emph{assuming that such a solution exists}. Such solutions do not exist \1 when the numbers and relative samples from the contact manifold change or \2 as contacts transition from active ($\dot{\phi_i}(\vect{x}, \vect{v}) \leq 0$) to inactive ($\dot{\phi}_i(\vect{x}, \vect{v}) > 0$), or vice versa. Case \2 implies immediate precedence or subsequence of case \1, which means that discontinuities in actuator torques will occur for at most two control loop iterations around a contact change (one discontinuity will generally occur due to the contact change itself).    

\subsection{Scaling inverse dynamics runtime linearly in number of contacts}
\label{section:ID:scaling-linearly}
The multi-body's number of generalized coordinates ($m$) are expected to remain constant. The number of contact points, $n$, depends
on the multi-body's link geometries, the environment, and whether the inverse dynamics approach should anticipate \emph{potential} contacts in $[t_0, t_f]$ (as discussed in Section~\ref{sec:jacobian-evaluation}). This
section describes a method to solve inverse dynamics problems with simultaneous
contact force computation that scales linearly with additional contacts. This
method will be applicable to all inverse dynamics approaches presented in
this paper except that described in Section~\ref{section:ID:Coulomb}: that problem results
in a copositive LCP~\citep{Cottle:1992} that the algorithm we
describe cannot generally solve. 

To this point in the article presentation, time complexity has been dominated by the $O(n^3)$ expected time solution to the LCPs. However, a system with $m$ degrees-of-freedom requires no more than $m$ positive force magnitudes applied along the contact normals to satisfy the constraints for the no-slip contact model. 
Proof is provided in Appendix~\ref{thm:maxcard}. Below, we describe how that proof can be leveraged to generally decrease the expected time complexity.

\paragraph{Modified PPM~I Algorithm}
We now describe a modification to the Principal Pivoting Method~I \citep{Cottle:1992} (PPM) for solving LCPs (see description of this algorithm in Appendix~\ref{section:pivoting}) that leverages the proof in Appendix~\ref{thm:maxcard} to attain expected $O(m^3 + nm^2)$ time complexity. A brief explanation of the mechanics of pivoting algorithms is provided in Appendix~\ref{section:pivoting}; we use the common notation of $\beta$ as the set of \emph{basic variables} and $\overline{\beta}$ as the set of \emph{non-basic variables}.

The PPM requires few modifications toward our purpose. These modifications are presented in Algorithm~\ref{alg:PPM}. First, the full matrix $\mat{N} \inv{\mat{M}} \tr{\mat{N}}$ is never constructed, because the construction is unnecessary and would require $O(n^3)$ time. Instead, Line~11 of the algorithm constructs a maximum $m \times m$ system; thus, that operation requires only $O(m^3)$ operations. Similarly, Lines 12 and 13 also leverage the proof from Appendix~\ref{thm:maxcard} to compute $\vect{w}^\dag$ and $\vect{a}^\dag$ efficiently (though these operations do not affect the asymptotic time complexity).
Expecting that the number of iterations for a pivoting algorithm is $O(n)$ in the size of the input~\citep{Cottle:1992} and assuming that each iteration requires at most two pivot operations (each rank-1 update operation to a matrix factorization will exhibit $O(m^2)$ time complexity), the asymptotic complexity of the modified PPM~I algorithm is $O(m^3 + m^2n)$. The termination conditions for the algorithm are not affected by our modifications.

Finally, we note that \cite{Baraff:1994} has proven that LCPs of the form $(\mat{H}\vect{w}, \mat{H}\inv{\mat{Q}}\tr{\mat{H}})$, where $\mat{H} \in \mathbb{R}^{p \times q}, \mat{Q} \in \mathbb{R}^{q \times q}, \mat{w} \in \mathbb{R}^q$, and $\mat{Q}$ is symmetric, PD are solvable. Thus, the inverse dynamics model will always
possess a solution.

\begin{algorithm}[H]
\footnotesize
\begin{algorithmic}[1]
\caption{\label{alg:PPM}$\{ \vect{z}, \vect{w}, \overline{\mathcal{B}} \} = $\textsc{ PPM}$(\vect{N}, \vect{M}, \vect{f}^*, \vect{z}^-)$ Solves the LCP $(\mat{N}\inv{\mat{M}}\vect{f}^*, \mat{N}\inv{\mat{M}}\tr{\mat{N}})$ resulting from convex, rigid contact models (the no-slip model and the complementarity-free model with Coulomb friction) . $\overline{\mathcal{B}}^*$ are the set of non-basic indices returned from the last call to \textsc{PPM}.}
\State $n \leftarrow $ rows($\mat{N}$)
\State $\vect{r} \leftarrow \mat{N}\cdot\vect{f}^*$
\State $i \leftarrow \argmin_i r_i$
\Comment {Check for trivial solution}
\If{$r_i \ge 0$}
\State \textbf{return}\ $\{ \vect{0}, \vect{r} \}$
\EndIf
\State $\overline{\mathcal{B}} \leftarrow \overline{\mathcal{B}}^*$
\If{$\overline{\mathcal{B}} = \emptyset$}
\State $\overline{\mathcal{B}} \leftarrow \{ i \}$
\Comment {Establish initial nonbasic indices}
\EndIf
\State $\mathcal{B} \leftarrow \{ 1, \ldots, n \} - \overline{\mathcal{B}}$
\Comment {Establish basic indices}
\While{\emph{true}}
\State $\mat{A} \leftarrow \mat{N}_{\overline{\mathcal{B}}} \cdot \inv{\mat{M}} \cdot \tr{\mat{N}_{\overline{\mathcal{B}}}}$ 
\State $\vect{b} \leftarrow \mat{N}_{\overline{\mathcal{B}}} \cdot \vect{f}^*$
\State $\vect{z}^\dag \leftarrow \inv{\mat{A}} \cdot -\vect{b}$ 
\Comment {Compute $\vect{z}$ non-basic components}
\State $\vect{a}^\dag \leftarrow \inv{\mat{M}} \cdot \tr{\mat{N}_{\overline{\mathcal{B}}}} \vect{z}^\dag + \vect{f}^*$
\State $\vect{w}^\dag \leftarrow \mat{N} \cdot \vect{a}^\dag$
\State $i \leftarrow \argmin_i w^\dag_i$
\Comment {Search for index to move into non-basic set}
\If{$w^\dag_i \ge 0$} 
  \State $j \leftarrow \argmin_i \vect{z}^\dag_i$
  \Comment {No index to move into the non-basic set; search for index to move into the basic set}
  \If{$z^\dag_j < 0$}
    \State $k \leftarrow \overline{\mathcal{B}}(j)$
    \State $\mathcal{B} \leftarrow \mathcal{B} \cup \{ k \}$
    \Comment {Move index $k$ into basic set}
    \State $\overline{\mathcal{B}} \leftarrow \overline{\mathcal{B}} - \{ k \}$
    \State \textbf{continue}
  \Else
    \State $\vect{z} \leftarrow \vect{0}$
    \State $\vect{z}_{\overline{\mathcal{B}}} \leftarrow \vect{z}^\dag$
    \State $\vect{w} \leftarrow \vect{0}$
    \State $\vect{w}_{\mathcal{B}} \leftarrow \vect{w}^\dag$
    \State \textbf{return} $\{ \vect{z}, \vect{w} \}$
  \EndIf
\Else
  \State $\overline{\mathcal{B}} \leftarrow \overline{\mathcal{B}} \cup \{ i \}$
  \Comment{Move index $i$ into non-basic set}
  \State $\mathcal{B} \leftarrow \mathcal{B} - \{ i \}$
  \State $j \leftarrow \argmin_i \vect{z}^\dag_i$
  \Comment {Try to find an index to move into the basic set}
  \If{$z^\dag_j < 0$}
    \State $k \leftarrow \overline{\mathcal{B}}(j)$
    \State $\mathcal{B} \leftarrow \mathcal{B} \cup \{ k \}$
    \Comment {Move index $k$ into basic set}
    \State $\overline{\mathcal{B}} \leftarrow \overline{\mathcal{B}} - \{ k \}$
  \EndIf
\EndIf
\EndWhile
\end{algorithmic}
\end{algorithm}

\section{Inverse dynamics with Coulomb friction} 
\label{section:ID:Coulomb}

	Though control with the absence of slip may facilitate grip and traction, the assumption is often not the case in reality.  Foot and manipulator geometries, and planned trajectories must be specially planned to prevent sliding contact, and assuming sticking contact may lead to disastrous results (see discussion on experimental results in Section~\ref{section:verificaton}).  Implementations of controllers that limit actuator forces to keep contact forces within the bounds of a friction constraint have been suggested to reduce the occurrence of unintentional slip in walking robots~\citep{Righetti:2013}. These methods also limit the reachable space of accelerative trajectories that the robot may follow, as all movements yielding sliding contact would be infeasible.  The model we present in this section permits sliding contact.  We formulate inverse dynamics using the computational model of rigid contact with Coulomb friction developed by~\cite{Stewart:1996} and \cite{Anitescu:1997}; the equations in this section closely follow those in~\cite{Anitescu:1997}. 

\subsection{Coulomb friction constraints}
Still utilizing the first-order discretization of the rigid body dynamics (Equation~\ref{eqn:discretized-generalized-newton}), we reproduce the linearized Coulomb friction constraints from~\cite{Anitescu:1997} without explanation (identical except for slight notational differences):
\begin{align}
& \vect{0} \le \mat{E}\vect{\lambda} + \before \mat{F}\after \vect{v}\ \bot\ \vect{f}_F \ge \vect{0} \label{eqn:A-P:Coulomb1} 
\\
& \vect{0} \le \boldsymbol{\mu} f_N - \tr{\mat{E}}\vect{f}_F\ \bot\ \vect{\lambda} \ge \vect{0} \label{eqn:A-P:Coulomb2} 
\end{align}
where $\mat{E} \in \mathbb{R}^{n \times nk}$ ($k$ is the number of edges in the polygonal approximation to the friction cone) retains its definition from~\cite{Anitescu:1997} as a sparse selection matrix containing blocks of ``ones'' vectors, $\boldsymbol{\mu} \in \mathbb{R}^{n \times n}$ is a diagonal matrix with elements corresponding to the coefficients of friction at the $n$ contacts, $\vect{\lambda}$ is a variable
roughly equivalent to magnitude of tangent velocities after contact forces
are applied, and $\mat{F} \in \mathbb{R}^{nk \times m}$ (equivalent to $\mat{D}$
in~\citealp*{Anitescu:1997}) is the matrix of wrenches of
frictional forces at $k$ tangents to each contact point. If the friction cone
is approximated by a pyramid (an assumption we make in the remainder of the
article), then:
\begin{align}
	\mat{F} &\equiv \tr{\begin{bmatrix}\tr{\mat{S}}\ &\ -\tr{\mat{S}}\ &\ \tr{\mat{T}} &\ -\tr{\mat{T}} \end{bmatrix}} \notag \\
	\vect{f}_F &\equiv \tr{\begin{bmatrix} \tr{\vect{f}_S^+} & \tr{\vect{f}_S^-}  & \tr{\vect{f}_T^+} & \tr{\vect{f}_T^-} \end{bmatrix}}  \notag
\end{align}
	where $\vect{f}_S = \vect{f}_S^+ - \vect{f}_S^-$ and $\vect{f}_T = \vect{f}_T^+ - \vect{f}_T^-$.  Given these substitutions, the contact model with inverse dynamics becomes:

\begin{tceqn}{Complementarity-based inverse dynamics}
\begin{align}
                   &_\mathrm{non-interpenetration,\ compressive\ force,\ and} \notag \\
& _\mathrm{normal\ complementarity\ constraints:} \notag \\
                   & \vect{0} \le \vect{f}_N\ \bot\ \mat{N}\after\vect{v} \ge -\frac{\before \vect{\phi}}{\Delta t} \\
                   &_\mathrm{Coulomb\ friction\ constraints:} \notag \\
& \vect{0} \le \lambda \vect{e} + \mat{F}\after \vect{v}\ \bot\ \vect{f}_F \ge \vect{0} \\ 
& \vect{0} \le \mu f_N - \tr{\vect{e}}\vect{f}_F\ \bot\ \boldsymbol{\lambda} \ge \vect{0} \\
                   &_\mathrm{inverse\ dynamics:} \notag \\
                   & \mat{P}\after\vect{v} = \dot{\vect{q}}_{\textrm{des}}\\
                   &_\mathrm{first-order\ dynamics:} \notag \\
                   & \after\vect{v} = \before\vect{v} + \inv{\mat{M}}(\tr{\mat{N}}\vect{f}_N + \ldots \\
                   & \qquad \quad \tr{\mat{F}}\vect{f}_F + \Delta t\vect{f}_{ext} - \tr{\mat{P}}\vect{\tau}) \notag 
\end{align}
\end{tceqn}

\subsection{Resulting MLCP}
Combining Equations~\ref{eqn:normalCC:A-P}--\ref{eqn:invdyn} and~\ref{eqn:A-P:Coulomb1}--\ref{eqn:A-P:Coulomb2} results in the MLCP:
\footnotesize
\begin{align}
&\begin{bmatrix}
\mat{M} & -\tr{\mat{P}} & -\tr{\mat{N}} & -\tr{\mat{F}} & \mat{0} \\
\mat{P} & \mat{0} & \mat{0} & \mat{0} & \mat{0} \\
\mat{N} & \mat{0} & \mat{0} & \mat{0} & \mat{0} \\
\mat{F} & \mat{0} & \mat{0} & \mat{0} & \mat{E} \\
\mat{0} & \mat{0} & \mat{\mu} & -\tr{\mat{E}} & \mat{0} 
\end{bmatrix}
\begin{bmatrix}
\after \vect{v} \\
\vect{\tau} \\
\vect{f}_N \\
\vect{f}_F \\
\vect{\lambda}
\end{bmatrix}
\hspace{-1mm} + \hspace{-1mm}
\begin{bmatrix}
-\vect{\kappa} \\
-\dot{\vect{q}}_{\textrm{des}} \\
\frac{\before \vect{\phi}}{\Delta t} \\
\vect{0} \\
\vect{0}
\end{bmatrix}
\hspace{-1mm} = \hspace{-1mm}
\begin{bmatrix}
\vect{0} \\
\vect{0} \\
\vect{w}_N \\
\vect{w}_F \\
\vect{w}_{\lambda}
\end{bmatrix} \label{eqn:A-P:MLCP-start} \\
& \vect{f}_N \ge \vect{0}, \vect{w}_N \ge \vect{0}, \tr{\vect{f}}_N \vect{w}_N = 0\\
& \vect{f}_F \ge \vect{0}, \vect{w}_F \ge \vect{0}, \tr{\vect{f}}_F \vect{w}_F = 0\\
& \vect{\lambda} \ge \vect{0}, \vect{w}_\lambda \ge \vect{0}, \tr{\vect{\lambda}} \vect{w}_\lambda = 0 \label{eqn:A-P:MLCP-end}
\end{align}
\normalsize
 Vectors $\vect{w}_N$ and $\vect{w}_F$ correspond to the normal and tangential velocities after impulsive forces have been applied.

\subsubsection{Transformation to LCP and proof of solution existence}
The MLCP can be transformed to a LCP as described by~\cite{Cottle:1992} by
solving for the unconstrained variables $\after \vect{v}$ and $\vect{\tau}$. This transformation is possible because the matrix:
\begin{equation}
\mat{X} \equiv
\begin{bmatrix}
\mat{M} & \tr{\mat{P}} \\
\mat{P} & \mat{0}
\end{bmatrix} 
\end{equation} 
is non-singular.  Proof comes from blockwise invertibility of this matrix, which requires only invertibility of $\mat{M}$ (guaranteed because generalized inertia matrices are positive definite) and $\mat{P}\inv{\mat{M}}\tr{\mat{P}}$. This latter matrix selects exactly those rows and columns corresponding to the joint space inertia matrix~\citep{Featherstone:1987}, which is also positive definite. After eliminating the unconstrained variables $\after \vect{v}$ and $\vect{\tau}$, the following LCP results:

\footnotesize

\begin{align}
& \begin{smallbmatrix}
\mat{N}\inv{\mat{M}}\tr{\mat{N}} & \mat{N}\inv{\mat{M}}\tr{\mat{F}} & \mat{E} \\
\mat{F}\inv{\mat{M}}\tr{\mat{N}} & \mat{F}\inv{\mat{M}}\tr{\mat{F}} & \mat{0} \\
-\tr{\mat{E}} & \boldsymbol{\mu} & \mat{0}
\end{smallbmatrix}
\begin{smallbmatrix}
\vect{f}_N \\
\vect{f}_F \\
\vect{\lambda}
\end{smallbmatrix}
\hspace{-1mm} + \hspace{-1mm}
\begin{smallbmatrix}
\frac{\before \vect{\phi}}{\Delta t} - \mat{N}\inv{\mat{M}}\vect{\kappa} \\
-\mat{F}\inv{\mat{M}}\vect{\kappa} \\
\vect{0}
\end{smallbmatrix}
\hspace{-1mm} = \hspace{-1mm}
\begin{smallbmatrix}
\vect{w}_N \\
\vect{w}_F \\
\vect{w}_\lambda
\end{smallbmatrix} \label{eqn:LCP-matrix} \\
& \vect{f}_N \ge \vect{0}, \vect{w}_N \ge \vect{0}, \tr{\vect{f}}_N \vect{w}_N = 0 \label{eqn:LCP-start} \\
& \vect{f}_F \ge \vect{0}, \vect{w}_F \ge \vect{0}, \tr{\vect{f}}_F \vect{w}_F = 0\\
& \vect{\lambda} \ge \vect{0}, \vect{w}_\lambda \ge \vect{0}, \tr{\vect{\lambda}} \vect{w}_\lambda = 0 \label{eqn:LCP-end}
\end{align} 
\normalsize

The discussion in~\cite{Stewart:1996} can be used to show that this LCP matrix is \emph{copositive} (see~\citealp*{Cottle:1992}, Definition 3.8.1), since for any vector $\vect{z} = \tr{\begin{bmatrix} \tr{\vect{f}_N} & \tr{\vect{f}_F} & \tr{\vect{\lambda}} \end{bmatrix}} \geq \vect{0}$,

\begin{align}
\tr{\begin{bmatrix} \vect{f}_N \\ \vect{f}_F \\ \vect{\lambda} \end{bmatrix}}  
\begin{bmatrix}
\mat{N}\inv{\mat{M}}\tr{\mat{N}} & \mat{N}\inv{\mat{M}}\tr{\mat{F}} & \mat{E} \\
\mat{F}\inv{\mat{M}}\tr{\mat{N}} & \mat{F}\inv{\mat{M}}\tr{\mat{F}} & \mat{0} \\
 \boldsymbol{\mu} & -\tr{\mat{E}}& \mat{0}
\end{bmatrix}
\begin{bmatrix} \vect{f}_N \\ \vect{f}_F \\ \vect{\lambda} \end{bmatrix} = \notag \\\tr{(\tr{\mat{N}}\vect{f}_N + \tr{\mat{F}}\vect{f}_N)}\inv{\mat{M}}(\tr{\mat{N}}\vect{f}_N + \tr{\mat{F}}\vect{f}_N) + \tr{\vect{f}_N}\boldsymbol{\mu}\vect{\lambda} \geq 0
\end{align}

because $\inv{\mat{M}}$ is positive definite and $\boldsymbol{\mu}$ is a diagonal matrix with non-negative elements. The transformation from the MLCP to the LCP yields $(k+2)n$ LCP variables (the per-contact allocation is: one for the normal contact magnitude, $k$ for the frictional force components, and one for an element of $\vect{\lambda}$) and at most $2m$ unconstrained variables.

As noted by~\cite{Anitescu:1997}, Lemke's Algorithm can provably solve such copositive LCPs~\citep{Cottle:1992} if
precautions are taken to prevent cycling through indices. After solving the LCP, joint torques can be retrieved exactly as in Section~\ref{section:retrieving-forces}. \emph{Thus, we have shown---using elementary extensions to the work in~\cite{Anitescu:1997}---that a right inverse of the non-impacting rigid contact model
exists} (as first broached in Section~\ref{section:intro}). Additionally, the
expected running time of Lemke's Algorithm is cubic in the number of variables,
so this inverse can be computed in expected polynomial time.

\subsection{Contact indeterminacy}
	Though the approach presented in this section produces a solution to the inverse dynamics problem with simultaneous contact force computation, the approach can converge to a vastly different, but equally valid solution at each controller iteration.  However, unlike the no-slip model, it is unclear how to bias the solution to the LCP, because a means for warm starting Lemke's Algorithm is currently unknown (our experience confirms the common wisdom that using the basis corresponding to the last solution usually leads to excessive pivoting).
	

	Generating a torque profile that would evolve without generating torque chatter requires checking all possible solutions of the LCP if this approach is to be used for robot control.  Generating all solutions requires a linear system solve for each combination of basic and non-basic indices among all problem variables.  Enumerating all possible solutions yields exponential time complexity, the same as the worst case complexity of Lemke's Algorithm~\citep{Lemke:1965}.
	After all solutions have been enumerated, torque chatter would be 
eliminated by using the solution that differs minimally from the last solution.
Algorithm~\ref{alg:MINDIFF} presents this approach.
\begin{algorithm}[htpb]
\footnotesize
\begin{algorithmic}[1]
\caption{\label{alg:MINDIFF}$\{ \vect{x}, \epsilon \} = $\textsc{MINDIFF}$(\mat{A},\mat{B},\mat{C},\mat{D}, \vect{g},\vect{h}, \overline{\mathcal{B}}, \vect{x}_{0},n)$ Computes the solution to the LCP \mbox{$(\vect{h} - \mat{D}\inv{\mat{A}}\vect{g}, \mat{B} - \mat{D}\inv{\mat{A}}\mat{C})$} that is closest (by Euclidean norm) to vector $\vect{x}_0$ using a recursive approach.  $n$ is always initialized as rows($\mat{B}$).}
\If{$n > 0$}
	\State $\{ \vect{x}_{1}, \epsilon_1 \} = $\textsc{MINDIFF}$(\mat{A},\mat{B},\mat{C},\mat{D}, \vect{g},\vect{h}, \overline{\mathcal{B}}, \vect{x}_{0}, n-1)$
	\State $\overline{\mathcal{B}} \leftarrow \{\overline{\mathcal{B}}, n \}$
	\Comment {Establish nonbasic indices}
	\State $\{ \vect{x}_{2}, \epsilon_2 \} = $\textsc{MINDIFF}$(\mat{A},\mat{B},\mat{C},\mat{D}, \vect{g},\vect{h}, \overline{\mathcal{B}}, \vect{x}_{0}, n-1)$
	\If{ $\epsilon_1 < \epsilon_2$ } \textbf{return} $\{ \vect{x}_{1}, \epsilon_1 \} $
	\Else\quad \textbf{return} $\{ \vect{x}_{2}, \epsilon_2 \} $
	\EndIf
\Else
	\State $\{\vect{z},\vect{w}\} = \mathrm{LCP}(\vect{h}_{nb} - \mat{D}_{nb}\inv{\mat{A}}\vect{g}, \mat{B}_{nb} - \mat{D}_{nb}\inv{\mat{A}}\mat{C}_{nb})$
	\State $\vect{x} \leftarrow \inv{\mat{A}}(\mat{C}_{nb}\vect{z}_{nb} + \vect{g})$
	\State \textbf{return} $\{\vect{x}, \norm{\vect{x} - \vect{x}_{0}}\}$
\EndIf
\end{algorithmic}
\end{algorithm}

The fact that we can enumerate all solutions to the problem in exponential time proves that solving the problem is at worst NP-hard, though following an enumerative approach is not practical.    


\section{Convex inverse dynamics without normal complementarity}
\label{section:approximate-idyn}

This section describes an approach for inverse dynamics that mitigates
indeterminacy in rigid contact using the impact model described in Section~\ref{section:related-work:QP}. The approach is almost identical to the ``standard''
rigid contact model described in Section~\ref{section:rigid-body-contact}, but for the absence of the normal complementarity
constraint. 

The approach
works by determining contact and actual forces in a first step and then
solving within the nullspace of the objective function (Equation~\ref{KE:objective}) such
that joint forces are minimized. The resulting problem is strictly convex,
and thus torques are continuous in time (and more likely safe for a robot to
apply) if the underlying dynamics are smooth.
This latter assumption is violated only when a discontinuity occurs from one
control loop iteration to the next, as it would when contact is broken, 
bodies newly come into contact, the contact surface changes, or contact 
between two bodies switches between slipping and sticking. 

Torque chatter due to contact indeterminacy can be avoided by ensuring that contact forces are not cycled rapidly between points of contact under potentially indeterminate contact configurations across successive controller iterations. \cite{Zapolsky:2014} eliminate torque chatter using a QP stage that searches over the optimal set of contact forces (using a convex relaxation to the rigid contact model) for forces that minimize the $\ell_2$-norm of joint torques. 

\subsection{Two-stage vs. one-stage approaches}
An alternative,
one-stage approach is described by \cite{Todorov:2014}, who regularizes the
quadratic objective matrix to attain the requisite strict convexity. Another
one-stage approach (which we test in Section~\ref{section:experiments}) uses the warm starting-based 
solution technique described in Section~\ref{section:ID:no-slip} to mitigate
contact indeterminacy. The two-stage approach
described below confers the following advantages over  
one-stage approaches: \1 no regularization factor need be chosen---there has yet to be a physical interpretation behind regularization factors, and computing a minimal regularization factor would be computationally expensive; and \2 the two-stage approach allows the particular solution to be selected using an arbitrary 
objective criterion---minimizing actuator torques is particularly relevant for
robotics applications. Two stage approaches are somewhat slower, though we have demonstrated performance suitably
fast for real-time control loops on quadrupedal robots in \cite{Zapolsky:2014}. We present the two stage approach without further comment, as the reader can realize the one stage approach, if desired, by regularizing the Hessian matrix in the quadratic program. 


\subsection{Computing inverse dynamics and contact forces simultaneously (Stage I)}
\label{section:inverse-dynamics-model}
\label{sec:idyn-method}

For simplicity of presentation, we will assume that the number of edges in the approximation of the friction cone for each contact is four; in other words, we will use a friction pyramid in place of a cone. The inverse dynamics problem
is formulated as follows: 

\begin{tceqn}{Complementarity-free inverse dynamics: Stage I}
\begin{align}
                   &_\mathrm{dissipate\ kinetic\ energy\ maximally:} \notag \\
\minimize_{\vect{f}_N, \vect{f}_F, \after \vect{v}, \vect{\tau}}\ & \frac{1}{2}\tr{\after\vect{v}}\mat{M}\after\vect{v}\label{eqn:objective1} \\ & \notag  \\
\textrm{subject to: }   &_\mathrm{non-interpenetration\ constraint:} \notag \\ 
& \mat{N}\after\vect{v} \geq -\frac{\before \vect{\phi}}{\Delta t} \label{const:noninterpen}\\
                   &_\mathrm{variable\ non-negativity} \notag \\
                   &_\mathrm{(for\ formulation\ convenience):} \notag \\
                   & \vect{f}_N \geq \vect{0}, \quad \vect{f}_F \geq \vect{0} \\
                    &_\mathrm{Coulomb\ friction:} \notag \\
                   & \mu f_{N_i} \geq \tr{\vect{1}}\vect{f}_{F_i} \label{eqn:frictpoly}\\
                   &_\mathrm{first-order\ velocity\ relationship:} \notag \\
                   & \after\vect{v} = \before\vect{v} + \inv{\mat{M}}(\tr{\mat{N}}\vect{f}_N + \ldots \\
                   & \qquad \quad \tr{\mat{F}}\vect{f}_F + \Delta t\vect{f}_{ext} - \tr{\mat{P}}\vect{\tau}) \notag \\
                   &_\mathrm{inverse\ dynamics:} \notag \\
                   & \mat{P}\after\vect{v} = \dot{\vect{q}}_{\textrm{des}}\label{eqn:invdyn2}
\end{align}
\end{tceqn}

As discussed in Section~\ref{section:related-work:QP}, we have shown that the
contact model always has a solution (\emph{i.e.}, the QP is always feasible) and
that the contact forces will not do positive work~\citep{Drumwright:2010b}.
The addition of the inverse dynamics constraint (Equation~\ref{eqn:invdyn2}) will not change this result---the frictionless version of this QP is identical to an LCP of the form that Baraff has proven solvable (see Section~\ref{section:ID:scaling-linearly}), which means that the QP is feasible.  As in the inverse dynamics approach in Section~\ref{section:ID:Coulomb}, the first order 
approximation to acceleration avoids inconsistent configurations that can occur 
in rigid contact with Coulomb friction.  The worst-case time complexity of solving this convex model is polynomial in the number of contact features~\citep{Boyd:2004}. High frequency control loops limit $n$ to approximately four contacts given present hardware and using fast active-set methods.

\subsubsection{Removing equality constraints}
The optimization in this section is a convex quadratic program with inequality and equality constraints.  We remove the equality constraints through substitution.
This reduces the size of the optimization problem; removing linear equality constraints also eliminates significant variables if transforming the QP to a LCP via optimization duality theory~\citep{Cottle:1992}.\footnote{We use such a transformation in our work, which allows us to apply \LEMKE~\citep{LEMKE}, which is
freely available, numerically robust (using Tikhonov regularization), and relatively fast.}

The resulting QP takes the form:
\begin{align}
\minimize_{\vect{f}_N, \vect{f}_F} & \tr{\begin{bmatrix} \vect{f}_N \\ \vect{f}_F \end{bmatrix}}\left(\begin{bmatrix} \mat{N}\inv{\mat{X}}\tr{\mat{N}} & \mat{N}\inv{\mat{X}}\tr{\mat{F}} \\ \mat{F}\inv{\mat{X}}\tr{\mat{N}} & \mat{F}\inv{\mat{X}}\tr{\mat{F}} \end{bmatrix} \begin{bmatrix} \vect{f}_N \\ \vect{f}_F \end{bmatrix} + \ldots \right. \notag \\
& \left. \qquad \qquad \qquad \begin{bmatrix} -\mat{N}\vect{\kappa} \\ -\mat{F}\vect{\kappa} \end{bmatrix}\right) \\
\textrm{subject\ to: } & \begin{bmatrix} \mat{N}\inv{\mat{X}}\tr{\mat{N}} & \mat{N}\inv{\mat{X}}\tr{\mat{F}} \end{bmatrix} \begin{bmatrix}\vect{f}_N \\ \vect{f}_F \end{bmatrix} - \mat{N}\vect{\kappa} \ge \vect{0} \\
& \vect{f}_N \ge \vect{0}, \vect{f}_F \ge \vect{0} \\
& \mu f_{N_i} \ge \tr{\vect{1}}\vect{f}_{F_i}
\end{align}
Once $\vect{f}_N$ and $\vect{f}_F$ have been determined, the inverse dynamics
forces are computed using:
\begin{align}
\begin{bmatrix}\after \vect{v} \\ \vect{\tau} \end{bmatrix} = \inv{\mat{X}}\begin{bmatrix} -\vect{\kappa} + \tr{\mat{N}}\vect{f}_N + \tr{\mat{F}}\vect{f}_F \\ \dot{\vect{q}}_{\textrm{des}} \end{bmatrix}
\end{align}
As in Section~\ref{section:retrieving-forces}, consistency 
in the desired accelerations can be verified and modified without re-solving the
QP if found to be inconsistent.
 
\subsubsection{Minimizing floating point computations}
\label{section:phaseI}
Because inverse dynamics may be used within real-time control loops, this
section describes an approach that can minimize floating point computations
over the formulation described above. 

Assume that we first solve for the joint forces $\vect{f}_{ID}$ necessary to solve the inverse dynamics problem under no contact constraints.
The new velocity $\after\vect{v}$ is now defined as:
\begin{equation}
\after\vect{v} = \before\vect{v} + \inv{\mat{M}}(\tr{\mat{N}}\vect{f}_N + \tr{\mat{F}}\vect{f}_F + \Delta t\vect{f}_{ext} + \begin{bmatrix} \vect{0} \\ \Delta t(\vect{f}_{ID}+ \vect{x})  \end{bmatrix})
\end{equation}
where we define $\vect{x}$ to be the actuator forces that are added to $\vect{f}_{ID}$ to counteract contact forces. To simplify our derivations, we will define the following vectors and matrices:
\begin{align}
\mat{R} &\equiv \begin{bmatrix} \tr{\mat{N}} & \tr{\mat{F}} \end{bmatrix} \\
\vect{z} &\equiv \tr{\begin{bmatrix} \vect{f}_N & \vect{f}_F \end{bmatrix}} \\
\mat{M} &\equiv \begin{matrix} &\begin{smallmatrix}\ {nb}\ &\ \ {nq} \end{smallmatrix} \\ \begin{smallmatrix} {nb}\\ \\ {nq} \end{smallmatrix} &\begin{bmatrix} \mat{A} &\mat{B} \\ \tr{\mat{B}} & \mat{C} \end{bmatrix}  \end{matrix}\\
\inv{\mat{M}} &\equiv \begin{matrix} &\begin{smallmatrix}\ {nb}\ &\ \ {nq} \end{smallmatrix} \\ \begin{smallmatrix} {nb}\\ \\ {nq} \end{smallmatrix} &\begin{bmatrix} \mat{D} & \mat{E} \\ \tr{\mat{E}} & \mat{G} \end{bmatrix}  \end{matrix}\\
\vect{j} &\equiv \vect{v}_b + \begin{bmatrix}\mat{D} & \mat{E}\end{bmatrix}(\Delta t \vect{f}_{ext} + \begin{bmatrix} \vect{0} \\ \Delta t\vect{f}_{ID}\end{bmatrix}) \\
\vect{k} &\equiv \vect{v}_q + \begin{bmatrix}\tr{\mat{E}} & \mat{G}\end{bmatrix}(\Delta t \vect{f}_{ext} + \begin{bmatrix} \vect{0} \\ \Delta t \vect{f}_{ID}\end{bmatrix}) \\
\vect{f}_{ID} &\equiv \mat{C}\ddot{\vect{q}} - \vect{f}_{ext,nq}
\end{align}

Where $nq$ is the total joint degrees of freedom of the robot, and $nb$ is the total base degrees of freedom for the robot ($nb = 6$ for floating bases).

The components of $\after\vect{v}$ are then defined as follows:
\begin{align}
\after\vect{v}_b &\equiv \vect{j} + \begin{bmatrix}\mat{D} & \mat{E}\end{bmatrix}(\mat{R}\vect{z} + \begin{bmatrix} \vect{0} \\ \Delta t \vect{x}\end{bmatrix}) \label{eqn:vbstar} \\
\after\vect{v}_q &\equiv \vect{k} + \begin{bmatrix}\tr{\mat{E}} & \mat{G}\end{bmatrix}(\mat{R}\vect{z} + \begin{bmatrix} \vect{0} \\ \Delta t \vect{x}\end{bmatrix}) = \dot{\vect{q}}_{\textrm{des}}
\end{align}
From the latter equation we can solve for $\vect{x}$:
\begin{equation}
\label{eqn:x}
\vect{x} = \frac{\inv{\mat{G}}(\after\vect{v}_q - \vect{k} - \begin{bmatrix}\tr{\mat{E}} & \mat{G}\end{bmatrix}\mat{R}\vect{z})}{\Delta t}
\end{equation}
Equation \ref{eqn:x} indicates that once contact forces are computed, determining the actuator forces for inverse dynamics requires solving only a linear equation.  Substituting the solution for $\vect{x}$ into Eqn. \ref{eqn:vbstar} yields:
\begin{equation}
\after\vect{v}_b = \vect{j} + \begin{bmatrix}\mat{D} & \mat{E}\end{bmatrix}\mat{R}\vect{z} + \mat{E}\inv{\mat{G}}(\after\vect{v}_q - \vect{k} - \begin{bmatrix}\tr{\mat{E}} & \mat{G}\end{bmatrix}\mat{R}\vect{z})
\end{equation} 
To simplify further derivation, we will define a new matrix and a new vector:
\begin{align}
\mat{Z} & \equiv \left( \begin{bmatrix}\mat{D} & \mat{E}\end{bmatrix} - \mat{E}\inv{\mat{G}}\begin{bmatrix}\tr{\mat{E}} & \mat{G}\end{bmatrix} \right) \mat{R} \\
\vect{p} & \equiv \vect{j} + \mat{E}\inv{\mat{G}}(\after\vect{v}_q - \vect{k})
\end{align}
Now, $\after\vect{v}_b$ can be defined simply, and solely in terms of $\vect{z}$, as:
\begin{equation}
\after \vect{v}_b \equiv \mat{Z}\vect{z} + \vect{p}
\end{equation}
We now represent the objective function (Eqn. \ref{eqn:objective1}) in block form as:
\begin{equation}
f(.) \equiv \frac{1}{2}\tr{\begin{bmatrix}\after\vect{v}_b \\ \after\vect{v}_q \end{bmatrix}}\begin{bmatrix}\mat{A} & \mat{B} \\ \tr{\mat{B}} & \mat{C} \end{bmatrix}\begin{bmatrix}\after\vect{v}_b \\ \after\vect{v}_q \end{bmatrix}
\end{equation}
which is identical to:
\begin{equation}
f(.) \equiv \frac{1}{2} \tr{\after\vect{v}_b}\mat{A}\after\vect{v}_b + \after\vect{v}_b\tr{\mat{B}}\after\vect{v}_q + \frac{1}{2} \tr{\after\vect{v}_q}\mat{C}\after\vect{v}_q
\end{equation}
As we will be attempting to minimize $f(.)$ with regard to $\vect{z}$, which the last term of the above equation does not depend on, that term is ignored hereafter.  Expanding remaining terms using Equation~\ref{eqn:vbstar}, the new objective function is:
\begin{align}
f(.) &\equiv \frac{1}{2} \tr{\vect{z}}\tr{\mat{Z}}\mat{A}\mat{Z}\vect{z} + \tr{\vect{z}}\tr{\mat{Z}}\mat{A}\vect{p} + \tr{\vect{z}}\tr{\mat{Z}}\mat{B}\after\vect{v}_q \\
 &\equiv \frac{1}{2} \tr{\vect{z}}\tr{\mat{Z}}\mat{A}\mat{Z}\vect{z} + \tr{\vect{z}}(\tr{\mat{Z}}\mat{A}\vect{p} + \tr{\mat{Z}}\mat{B}\after\vect{v}_q)
\end{align}
subject to the following constraints:
\begin{align}
\tr{\mat{N}}\begin{bmatrix} \mat{Z}\vect{z} + \vect{p} \\ \vect{v}^*_q \end{bmatrix} &\geq -\frac{\before \vect{\phi}}{\Delta t} \\
\vect{f}_{N,i} &\geq 0\qquad (\textrm{for}~i=1\ldots n) \\
\mu\vect{f}_{N,i} &\geq \vect{c}_{S,i} + \vect{c}_{T,i}
\end{align}
Symmetry and positive semi-definiteness of the QP follows from symmetry and positive definiteness of $\mat{A}$.  Once the solution to this QP is determined, the actuator forces $\vect{x} + \vect{f}_{ID}$ determined via inverse dynamics can be recovered.

\subsubsection{Floating point operation counts}
Operation counts for matrix-vector arithmetic and numerical linear algebra
are taken from~\cite{Hunger:2007}.

\paragraph{Before simplifications:} 
Floating point operations (\emph{flops}) necessary to setup the \textit{Stage I} model as presented initially sum to 77,729 flops, substituting: $m = 18, nq = 16, n = 4, k = 4$.

\begin{table}[H]
\centering
\begin{tabular}{|c|c|}
\hline
operation & \emph{flops} \\
\hline
$LD\tr{L}(\mat{X}) $ \ & $\frac{m^3}{3}+m^2 \text{(nq)}+m^2+m \text{(nq)}^2+2 m \text{(nq)}$\\
&$-\frac{7m}{3}+\frac{\text{(nq)}^3}{3}+\text{(nq)}^2-\frac{7 \text{(nq)}}{3}+1$ \\
$\inv{\mat{X}}\tr{\mat{N}}$&  $m + 2 m^2 n + (nq) + 4 m n (nq) + 2 n (nq)^2$ \\
$\inv{\mat{X}}\tr{\mat{F}}$ & $m + 2 k m^2 n + (nq) + 4 k m n (nq) + 2 k n (nq)^2$ \\
$\mat{N}\inv{\mat{X}}\tr{\mat{N}}$ & $2mn^2 - mn$ \\
$\mat{F}\inv{\mat{X}}\tr{\mat{N}}$ & $2mn^2k - mn$ \\
$\mat{F}\inv{\mat{X}}\tr{\mat{F}}$&  $2mnk^2 - mnk$ \\
$\kappa$  &  $2m^2 - m$ \\
$\inv{\mat{X}}\kappa$ & $2(m + (nq))^2$ \\
$\mat{N}\inv{\mat{X}}\kappa$ & $mn - m$ \\
$\mat{F}\inv{\mat{X}}\kappa$  & $mnk - m$ \\
$\vect{\tau}$ & $2(m+(nq))^2$ \\
\hline
\end{tabular}
\caption{Floating point operations (\emph{flops}) per task without floating point optimizations.}
\end{table}


\paragraph{After simplifications:} 
Floating point operations necessary to setup the \textit{Stage I} model after modified to reduce computational costs sum to 73,163 flops when substituting: $m = 18, nq = 16, n = 4, k = 4,nb = m-nq$, a total of $6.24\%$ reduction in computation.  When substituting: $m = 18, nq = 12, n = 4, k = 4,nb = m-nq$, we observed 102,457 flops for this new method and 
62,109 flops before simplification.  Thus, a calculation of the total number of floating point operations should be performed to determine whether the floating point simplification is actually more efficient for a particular robot.

\begin{table}[H]
\centering
\begin{tabular}{|c|c|}
\hline
operation & \emph{flops} \\
\hline
$\inv{(L\tr{L}(\mat{M}))} $ &$\frac{2}{3}m^3 + \frac{1}{2}m^2 + \frac{5}{6}m$ \\
$L\tr{L}(\mat{G}) $ & $\frac{1}{3}(nq)^3 + \frac{1}{2}(nq)^2 + \frac{1}{6}nq$\\
$\mat{Z}$ & $nb\ m + nq\ n(nk+1)(2m - 1) +$ \\
&  $nb\ m(2nq - 1)  + 2nq^2m$ \\
$\vect{p}$ & $2nb\ nq + 2nq^2 + 3 nq + 2 nb + 2m + 2m\ nq$ \\
$\tr{\mat{Z}}\mat{A}\mat{Z}$ & $2nb^2(n(nk+1))-nb\ (n(nk+1))$ \\
&$+ 2nb\ (n(nk+1))^2 - (n(nk+1))^2$ \\
$\tr{\mat{Z}}\mat{A}\vect{p}$ & $(n(nk+1) + nb)(2nb-1)$ \\
$\tr{\mat{Z}}\mat{B}\vect{v}^*_q$ & $(n(nk+1) + nq)(2nq-1)$ \\
$\tr{\mat{N}}\mat{Z}$ & $n^2(nk+1)(2nb -1)$ \\
$\tr{\mat{N}}\left[\begin{smallmatrix} \vect{p} \\ \vect{v}^*_q \end{smallmatrix}\right]$& $n(2m-1)$ \\
\hline
\end{tabular}
\caption{Floating point operations (\emph{flops}) with floating point optimizations.}
\end{table}

\subsubsection{Recomputing Inverse Dynamics to Stabilize Actuator Torques (Stage II)}
\label{section:phaseII}
In the case that the matrix $\tr{\mat{Z}}\mat{A}\mat{Z}$ is singular, the contact model is only convex rather than strictly convex \citep{Boyd:2004}.  Where a strictly convex system has just one optimal solution, the convex problem may have multiple equally optimal solutions.  Conceptually, contact forces that predict that two legs, three legs, or four legs support a walking robot are all equally valid. A method is thus needed to optimize within the contact model's solution space while favoring solutions that predict contact forces at all contacting links (and thus preventing the rapid torque cycling between arbitrary optimal solutions).  As we mentioned in Section~\ref{section:contact-models:indeterminacy}, defining a manner by which actuator torques evolve over time, or selecting a preferred distribution of contact forces may remedy the issues resulting from indeterminacy.  One such method would select---from the space of optimal solutions to the indeterminate contact problem---the one that minimizes the $\ell_2$-norm of joint torques.  If we
denote the solution that was computed in the last section as $\vect{z}_0$, the following optimization problem will yield the desired result: 

\begin{figure*}
\center
\includegraphics[width=0.75\linewidth]{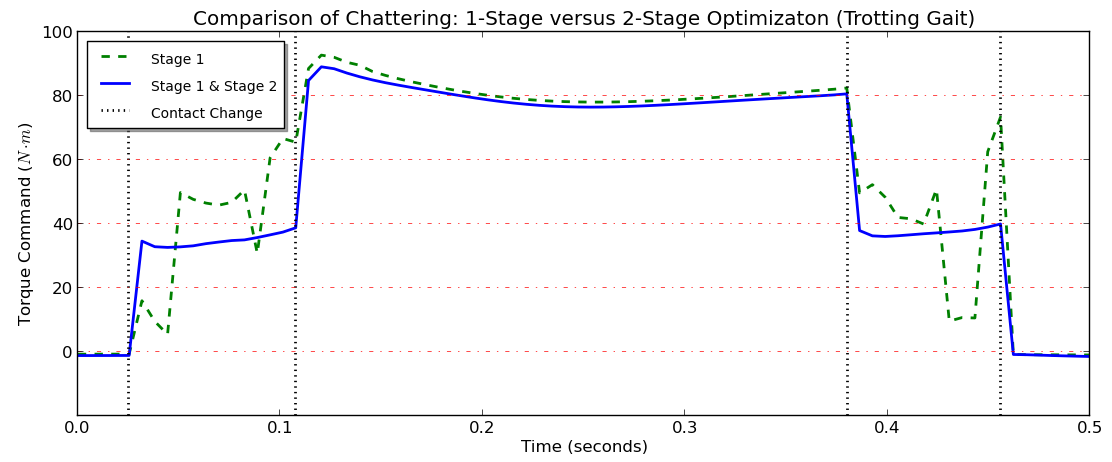} 
\caption{Plot of torque chatter while controlling with inverse dynamics using an indeterminate contact model (Stage 1) versus the smooth torque profile produced by a determinate contact model (Stage 1 \& Stage 2).} 
\label{fig:chatter}
\end{figure*}

\begin{tceqn}{Complementarity-free inverse dynamics: Stage II}
\begin{align}
                   &_\mathrm{find\ minimum\ norm\ motor\ torques:} \notag \\
\minimize_{\vect{f}_N, \vect{f}_F}\ & \frac{1}{2}\tr{\vect{\tau}}\vect{\tau} \label{eqn:objective2} \\
\textrm{subject to: } & \textrm{Equations~\ref{const:noninterpen}--\ref{eqn:invdyn2}} \nonumber \\
                   &_\mathrm{maintain\ Stage\ I\ objective:} \notag \\
			        & \frac{1}{2}\after\tr{\vect{v}}\mat{M}\after\vect{v} \leq f(\vect{z}_0) \label{eqn:null} 
\end{align}
\end{tceqn}
We call the method described in Section~\ref{section:phaseI}~\emph{Stage I} and the optimization problem above \emph{Stage II}. Constraining for a quadratic objective function (Equation~\ref{eqn:objective1}) with a quadratic inequality constraint yields a QCQP that may not be solvable sufficiently quickly for high frequency control loops. We now show how to use the nullspace of $\tr{\mat{Z}}\mat{A}\mat{Z}$ to perform this second optimization \emph{without explicitly considering the quadratic inequality constraint}; thus, a QP problem formulation is retained.  Assume that the matrix $\mat{W}$ gives the nullspace of $\tr{\mat{Z}}\mat{A}\mat{Z}$. The vector of contact forces will now be given as $\vect{z} + \mat{W}\vect{w}$, where $\vect{w}$ will be the optimization vector.

The kinetic energy from applying the contact impulses is:
\begin{align*}
\epsilon_2 & = \frac{1}{2}\tr{(\vect{z} + \mat{W}\vect{w})}\tr{\mat{Z}}\mat{A}\mat{Z}(\vect{z}  + \mat{W}\vect{w})\\
& \qquad + \tr{(\vect{z} + \mat{W}\vect{w})}(\tr{\mat{Z}}\mat{A}\vect{p} + \tr{\mat{Z}}\mat{B}\after\vect{v}_q) \\
& = \frac{1}{2} \tr{\vect{z}}\tr{\mat{Z}}\mat{A}\mat{Z}\vect{z} + \tr{\vect{z}}(\tr{\mat{Z}}\mat{A}\vect{p}
 + \tr{\mat{Z}}\mat{B}\after\vect{v}_q)\\
& \qquad + \tr{\vect{w}}\tr{\mat{W}}(\tr{\mat{Z}}\mat{A}\vect{p} + \tr{\mat{Z}}\mat{B}\after\vect{v}_q)
\end{align*}
The terms $\frac{1}{2}\tr{\vect{w}}\tr{\mat{W}}\tr{\mat{Z}}\mat{A}\mat{Z}\mat{W}\vect{w}$ and $\vect{z}\tr{\mat{Z}}\mat{A}\mat{Z}\mat{W}\vect{w}$ are not included above because both are zero: $\mat{W}$ is in the nullspace of $\tr{\mat{Z}}\mat{A}\mat{Z}$. The energy dissipated in the second stage. $\epsilon_2$, should be equal to the energy dissipated in the first stage, $\epsilon_1$. Thus, we want $\epsilon_2 - \epsilon_1 = 0$. Algebra yields:
\begin{equation}
\tr{\vect{w}}\tr{\mat{W}}(\tr{\mat{Z}}\mat{A}\vect{p} + \tr{\mat{Z}}\mat{B}\after\vect{v}_q) = 0
\end{equation}
We minimize the $\ell_2$-norm of joint torques with respect to contact forces by first defining $\vect{y}$ as:
\begin{equation}
\vect{y} \equiv \frac{\inv{\mat{G}}\Big(\after\vect{v}_q - \vect{k} - \begin{bmatrix}\tr{\mat{E}} & \mat{G}\end{bmatrix}\mat{R}(\vect{z} + \mat{W}\vect{w})\Big)}{\Delta t}
\end{equation}
The resulting objective is:
\begin{equation}
\label{equation:g}
g(\vect{w}) \equiv \frac{1}{2} \tr{\vect{y}}\vect{y}
\end{equation}
From this, the following optimization problem arises:
\begin{align}
\mathop{\mathrm{min}}_{\vect{w}}\ & \frac{1}{2}\tr{\vect{y}}\vect{y} \label{eqn:g} \\
\textrm{subject to: } & (\tr{\vect{p}}\tr{\mat{A}} + \after\tr{\vect{v}_q}\tr{\mat{B}})\mat{Z}\mat{W}\vect{w} = \vect{0} \label{eqn:redundant1} \\ 
& \tr{\mat{N}}\begin{bmatrix} \mat{Z}(\vect{z} + \mat{W}\vect{w}) + \vect{p} \\ \after\vect{v}_q \end{bmatrix} \geq -\frac{\before \vect{\phi}}{\Delta t} \label{eqn:redundant2} \\
& (\vect{z} + \mat{W}\vect{w})_i \geq \vect{0} ,\quad \textrm{ (for } i=1\ldots 5n) \label{eqn:poly} \\
& \mu_i (\vect{z} + \mat{W}\vect{w})_i \geq \mat{X}_{S_i}(\vect{z} + \mat{W}\vect{w})_{S_i} + \ldots \nonumber \\
& \hspace{.95in} \mat{X}_{T_i}(\vect{z} + \mat{W}\vect{w})_{T_i} \nonumber \\
& \hspace{.95in} (\textrm{for } i=1\ldots n) \label{eqn:frict-p2}
\end{align}
Equation~\ref{eqn:poly} constrains the contact force vector to only positive values, accounting for 5 positive directions to which force can be applied at the contact manifold ($\hat{n},\hat{s}, -\hat{s}, \hat{t}, -\hat{t}$).\footnote{We consider negative and positive $\hat{s}$ and $\hat{t}$ because \LEMKE requires all variables to be positive.} We use a proof that $\mat{Z} \cdot \ker(\tr{\mat{Z}}\mat{A}\mat{Z}) = \mat{0}$ \citep{Zapolsky:2014} to render $n+1$ (Equations~\ref{eqn:redundant1}~and~\ref{eqn:redundant2}) of $7n+1$ linear constraints (Equations~\ref{eqn:redundant1}--\ref{eqn:frict-p2}) unnecessary.

Finally, expanding and simplifying Equation \ref{equation:g} (removing terms that do not contain $\vect{w}$, scaling by $\Delta t^2$), and using the identity $\mat{U} \equiv \begin{bmatrix} \tr{\mat{E}} & \mat{G} \end{bmatrix}\mat{R}$ yields:
\begin{align}
g(\vect{w}) \equiv& \frac{1}{2} \tr{\vect{w}}\tr{\mat{W}}\tr{\mat{U}}\tr{\inv{\mat{G}}}\inv{\mat{G}}\mat{U}\mat{W}\vect{w} \\
& + \tr{\vect{z}}\tr{\mat{U}} \tr{\inv{\mat{G}}}\inv{\mat{G}}\mat{U}\mat{W}\vect{w} - \tr{\after\vect{v}_q}\tr{\inv{\mat{G}}}\inv{\mat{G}}\mat{U}\mat{W}\vect{w} \nonumber \\
& + \tr{\vect{k}}\tr{\inv{\mat{G}}}\inv{\mat{G}}\mat{U}\mat{W}\vect{w} \nonumber
\end{align}

Finally, actuator torques for the robot can be retrieved by calculating $\vect{y} + \vect{f}_{ID}$.

\subsubsection{Feasibility and time complexity}
It should be clear that a feasible point ($\vect{w} = \vect{0}$) always exists 
for the optimization problem. The dimensionality (\mbox{$n \times n$} in the number
of contact points) of $\tr{\mat{Z}}\mat{A}\mat{Z}$ yields a nullspace computation of $O(n^3)$ and represents one third of the Stage~II running times in our experiments. For quadrupedal robots with single point contacts, for example, the 
dimensionality of $\vect{w}$ is typically at most two, yielding fewer than 
$6n+2$ total 
optimization variables (each linear constraint 
introduces six KKT dual variables for the simplest friction cone approximation). Timing results given $n$ contacts for this virtual robot are available in Section~\ref{section:timing}.

\section{Experiments}
\label{section:experiments} 
\label{sec:experiments}

This section assesses the inverse dynamics controllers under a range of conditions on a virtual, locomoting quadrupedal robot (depicted in Figures~\ref{fig:links}) and a virtual manipulator grasping a box.  For points of comparison, we also provide performance data for three reference controllers (depicted in Figures~\ref{control:PID} and \ref{fig:control1}).  These experiments also serve to illustrate that the inverse dynamics approaches function as expected.  
	
	We explore the effects of possible modeling infidelities and sensing inaccuracies by testing locomotion performance on rigid planar terrain, rigid non-planar terrain, and on compliant planar terrain. The last of these is an
example of modeling infidelity, as the compliant terrain violates
the assumption of rigid contact.   Sensing inaccuracies may be introduced from physical sensor noise or perception error producing, \emph{e.g.}, erroneous friction or contact normal estimates. 
All code and data for experiments conducted in simulation, as well as videos of the virtual robots, are located (and can thus be reproduced) at:
\newline \texttt{\small http://github.com/PositronicsLab/idyn-experiments}.

\begin{tcfigure}{Locomotion planner graph}
\hspace{-0.15in}\includegraphics[width=1.1\linewidth]{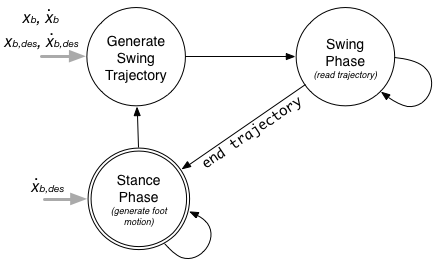}
\addtocounter{figure}{-1}\refstepcounter{figure}\label{fig:planner}
\vspace{-0.2in}
\end{tcfigure}

\subsection{Platforms}

	We evaluate performance of all controllers on a simulated quadruped (see Figure~\ref{fig:links}).  This test platform measures 16 cm tall and has three degree of freedom legs. Its feet are modeled as spherical collision geometries, creating a a single point contact manifold with the terrain.  We use this platform to assess the effectiveness of our inverse dynamics implementation with one to four points of contact.  Results we present from this platform are applicable to biped locomotion, the only differentiating factor being the presence of a more involved planning and balance system driving the quadruped.
	
\begin{figure}[h!]
\center
\includegraphics[width=\linewidth]{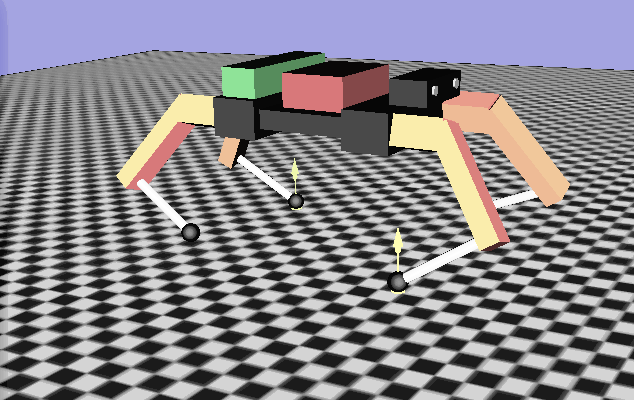} 
\caption{Snapshot of a quadruped robot in the \Moby simulator on planar terrain.}
\label{fig:links}
\end{figure}

Additionally, we demonstrate the adaptability of this approach on a manipulator grasping a box (see Figure~\ref{fig:arm}).  The arm has seven degrees of freedom and each finger has one degree of freedom, totaling eleven actuated degrees of freedom.  The finger tips have spherical collision geometries, creating a a single point contact manifold with a grasped object at each fingertip.  The grasped box has six non-actuated degrees of freedom with a surface friction that is specified in each experiment.

\begin{figure}[h!]
\center
\includegraphics[width=\linewidth]{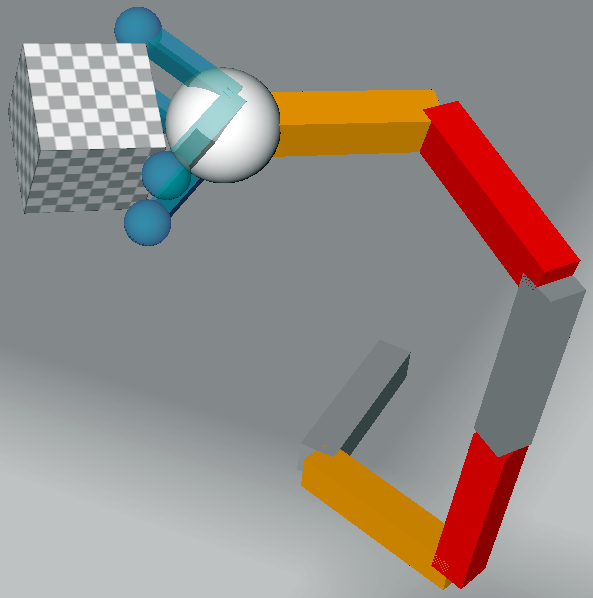} 
\caption{Snapshot of a fixed-base manipulator in the \Moby simulator grasping a box with four spherical fingertips.}
\label{fig:arm}
\end{figure}

\subsection{Source of planned trajectories}

\subsubsection{Locomotion trajectory planner}

	  We assess the performance of the reference controllers (described below) against the inverse dynamics controllers with contact force prediction presented in this work.  The quadruped follows trajectories generated by our open source legged locomotion software \Pacer\footnote{\Pacer is available from: \newline{\tt\small http://github.com/PositronicsLab/Pacer}}. Our software implements balance stabilization, footstep planning, spline-based trajectory planning (described below), and inverse dynamics-based controllers with simultaneous contact force computation that are the focus of this article. The following terms will be used in our description of the planning system: 

\begin{description}
 \item [\emph{gait}] Cyclic pattern of timings determining when to make and break contact with the ground during locomotion.
 \item [\emph{stance phase}] The interval of time where a foot is planned to be in contact with the ground (foot applies force to move robot)
 \item [\emph{swing phase}] The planned interval of time for a foot to be swinging over the ground (to position the foot for the next stance phase)
 \item [\emph{duty-factor}] Planned portion of the gait for a foot to be in the stance phase
 \item [\emph{touchdown}] Time when a foot makes contact with the ground and transitions from swing to stance phase
 \item [\emph{liftoff}] Time when a foot breaks contact with the ground and transitions from stance to swing phase
\end{description}

	The trajectories generated by the planner are defined in operational space.  Swing phase behavior is calculated as a velocity-clamped cubic spline at the start of each step and replanned during that swing phase as necessary.   Stance foot velocities are determined by the desired base velocity and calculated at each controller iteration.  The phase of the gait for each foot is determined by a gait timing pattern and gait duty-factor assigned to each foot.  
	
	The planner (illustrated in Figure~\ref{fig:planner}) takes as input desired planar base velocity ($\dot{\vect{x}}_{b,des} = [\dot{x},\dot{y},\dot{\theta} ]$) and plans touchdown and liftoff locations connected with splined trajectories for the feet to drive the robot across the environment at the desired velocity.  The planner then outputs a trajectory for each foot in the local frame of the robot.  After end effector trajectories have been planned, joint trajectories are determined at each controller cycle using inverse kinematics.

	

\subsubsection{Arm trajectory planner}

	The fixed-base manipulator is command to follow a simple sinusoidal trajectory parameterized over time.  The arm oscillates through its periodic motion about three times per second.  The four fingers gripping the box during the experiment are commanded to maintain zero velocity and acceleration while gripping the box, and to close further if not contacting the grasped object.
	
\subsection{Evaluated controllers}
\label{sec:controllers}

	We use the same error-feedback in all cases for the purpose of reducing joint tracking error from drift (see baseline controller in Figure~\ref{control:PID}).  The gains used for PID control are identical between all controllers but differ between robots.  The PID error feedback controller is defined in configuration-space on all robots.  Balance and stabilization are handled in this trajectory planning stage, balancing the robot as it performs its task.  The stabilization implementation uses an inverted pendulum model for balance, applying only virtual compressive forces along the contact normal to stabilize the robot~\citep{Sugihara:2003}.  The error-feedback and stabilization modules also accumulate error-feedback from configuration-space errors into the vector of desired acceleration ($\ddot{q}_{des}$) input into the inverse dynamics controllers.  

\begin{tcfigure}{Baseline Controller}
\hspace{-0.2in}\includegraphics[width=1.13\linewidth]{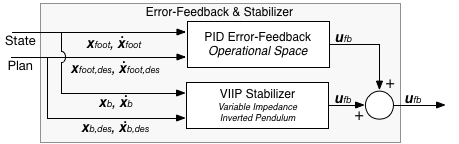}
\small
\textbf{PID}: Reference PID joint error-feedback controller, PD operational space error-feedback controller (quadruped only), and VIIP stabilization (quadruped only).
\addtocounter{figure}{-1}\refstepcounter{figure}\label{control:PID}
\end{tcfigure}



	We compare the controllers described in this article, which will hereafter be referred to as \textbf{ID($t_i$)}$_{\texttt{solver},\texttt{friction}}$, where the possible solvers are: $\texttt{solver} =  \{$QP, LCP$\}$ for QP and LCP-based optimization models, respectively; and the possible friction models are: $\texttt{friction} = \{ \mu , \infty\}$ for finite Coulomb friction and no-slip models, respectively.  
	
	We compare the controllers implemented using the methods described in Sections~\ref{section:ID:no-slip}, \ref{section:ID:Coulomb} and \ref{section:approximate-idyn}, see Figure~\ref{controller:IDt}) against the reference controllers (Figure~\ref{fig:control1}), using finite and infinite friction coefficients to permit comparison against no-slip and  Coulomb friction models, respectively.  Time ``$t_i$'' in \textbf{ID($t_i$)} refers to the use of contact forces predicted at the current controller time.  The experimental (presented) controllers include: \textbf{ID($t_{i}$)}$_{LCP,\infty}$ is the \emph{ab initio} controller from Section~\ref{section:ID:no-slip} that uses an LCP model, to predict contact reaction forces with no-slip constraints; \textbf{ID($t_{i}$)}$_{LCP,\mu}$ is the \emph{ab initio} controller from Section~\ref{section:ID:Coulomb} that uses an LCP model, to predict contact reaction forces with Coulomb friction; \textbf{ID($t_{i}$)}$_{QP,\mu}$ is the controller from Section~\ref{section:approximate-idyn} that uses a QP-based optimization approach for contact force prediction; \textbf{ID($t_{i}$)}$_{QP,\infty}$ is the same controller as \textbf{ID($t_{i}$)}$_{QP,\mu}$ from Section~\ref{sec:idyn-method}, but set to allow infinite frictional forces.

\begin{tcfigure}{Experimental Controllers}
\hspace{-0.17in}\includegraphics[width=1.12\linewidth]{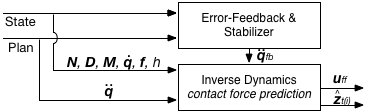}
\small
\addtocounter{figure}{-1}\refstepcounter{figure}\label{controller:IDt}
\textbf{ID($t_{i}$)}: Inverse dynamics controller with predictive contact forces (this work) generates an estimate of contact forces at the current time ($\hat{\vect{z}}(t)$) given contact state and internal forces.
\end{tcfigure}


\begin{tcfigure}{Reference Controllers}
\hspace{-0.17in}\includegraphics[width=1.12\linewidth]{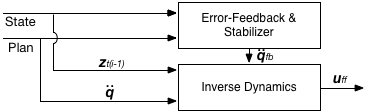}
\small
\addtocounter{figure}{-1}\refstepcounter{figure}\label{fig:control1}
\textbf{ID($t_{i-1}$)}: Reference inverse dynamics controller using exact sensed contact forces from the most recent contact force measurement, $\vect{z}(t-\Delta t)$\\
		\hrule
\vspace{2px}
\hspace{-0.17in}\includegraphics[width=1.12\linewidth]{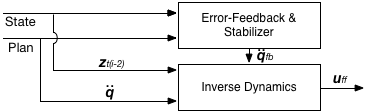}
\textbf{ID($t_{i-2}$)}: Reference inverse dynamics controller using exact sensed contact forces from the second most recent contact force measurement, $\vect{z}(t-2\Delta t)$
\end{tcfigure}

	The reference inverse dynamics controllers use sensed contact forces; the sensed forces are the exact forces applied by the simulator to the robot on the previous simulation step (\emph{i.e.}, there is a sensing lag of $\Delta t$ on these measurements, the simulation step size). We denote the controller using these exact sensed contact forces as \textbf{ID($t_{i-1}$)}.  Controller ``\textbf{ID($t_{i-1}$)}'' uses the exact value of sensed forces from the immediately previous time step in simulation and represents an upper limit on the performance of using sensors to incorporate sensed contact forces into an inverse dynamics model. \emph{In situ} implementation of contact force sensing should result in worse performance than the controller described here, as it would be subject to inaccuracies in sensing and delays of multiple controller cycles as sensor data is filtered to smooth noise; we examine the effect of a second controller cycle delay with \textbf{ID}($t_{i-2}$) (see Figure~\ref{fig:control1}). 


\subsection{Software and simulation setup}

\software{Pacer} runs alongside the open source simulator \Moby\footnote{Obtained from \url{https://github.com/PositronicsLab/Moby}}, which was used to simulate the legged locomotion scenarios used in the experiments.  \software{Moby} was arbitrarily set to simulate contact with the Stewart-Trinkle / Anitescu-Potra rigid contact models~\citep{Stewart:1996,Anitescu:1997}; therefore, the contact models utilized by the simulator match those used in our reference controllers, \emph{permitting us to compare our contact force predictions directly against those determined by the simulator}. Both simulations and controllers had access to identical data: kinematics (joint locations, link lengths, \emph{etc.}), dynamics (generalized inertia matrices, external forces), and friction coefficients at points of contact.  \Moby provides accurate time of contact calculation for impacting bodies, yielding more accurate contact point and normal information.  Other simulators step past the moment of contact, and approximate contact information based on the intersecting geometries.  The accurate contact information provided by \Moby allows us to test the inverse dynamics controllers under more realistic conditions: contact may break or be formed between control loops.

\subsection{Terrain types for locomotion experiments}
	We evaluate the performance of the baseline, reference, and presented controllers on a planar terrain.  We use four cases to encompass expected surface properties experienced in locomotion.  We model sticky and slippery terrain with frictional properties corresponding to a Coulomb friction of $0.1$ for low (slippery) friction and infinity for high (sticky) friction.  We also model rigid and compliant terrain, using rigid and penalty (spring-damper) based contact model, respectively.  The compliant terrain is modeled to be relatively hard, yielding only 1 mm interpenetration of the foot into the plane while the quadruped is at rest.  Our contact prediction models all assume rigid contact; a compliant terrain will assess whether the inverse dynamics model for predictive contact is viable when the contact model of the environment does not match its internal model.  Successful locomotion using such a controller on a compliant surface would indicate (to some confidence) that the controller is robust to modeling infidelities and thus more robust in an uncertain environment.
	
	We also test the controllers on a rigid height map terrain with non-vertical surface normals and varied surface friction to assess robustness on a more natural terrain. We limited extremes in the variability of the terrain, limiting bumps to 3 cm in height (about one fifth the height of the quadruped see Figure~\ref{fig:links-terrain}).  so that the performance of the foothold planner (not presented in this work) would not bias performance results.  Friction values were selected between the upper and lower limits of Coulomb friction ($\mu \sim \mathcal{U}(0.1,1.5)$) found in various literature on legged locomotion.

\begin{figure}[h!]
\center
\includegraphics[width=\linewidth]{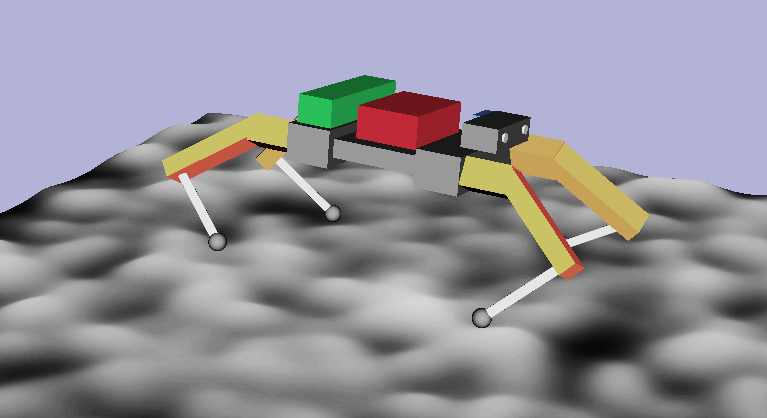} 
\caption{Snapshot of a quadruped robot in the \Moby simulator on rough terrain.}
\label{fig:links-terrain}
\end{figure}

\subsection{Tasks}
	We simulate the quadruped trotting between a set of waypoints on a planar surface for 30 virtual seconds.  The process of trotting to a waypoint and turning toward a new goal stresses some basic abilities needed to locomote successfully: \1 acceleration to and from rest; \2 turning in place, and \3 turning while moving forward. For the trotting gait we assigned gait duration of 0.3 seconds per cycle, duty-factor 75\% of the total gait duration, a step height of 1.5 cm, and touchdown times $\{0.25, 0.75, 0.75, 0.25\}$ for \{left front, right front, left hind, right hind\} feet, respectively.  These values were chosen as a generally functional set of parameters for a trotting gait on a 16 cm tall quadruped. The desired forward velocity of the robot over the test was 20 cm/s.  Over the interval of the experiment, the robots are in both determinate and \emph{possibly} indeterminate contact configurations and, as noted above, undergo numerous contact transitions.  We show (in Section~\ref{sec:trajectory-tracking}) that the controllers we present in Sections~\ref{section:ID:Coulomb} and~\ref{section:ID:no-slip} are feasible for controlling a quadruped robot over a trot; we compare contact force predictions made by all presented controllers to the reaction forces generated by the simulation (Section~\ref{exp:controller-sim}); and we measure running times of all presented controllers given numerous additional contacts in Section~\ref{exp:speed}. 
	
	We simulate the manipulator grasping a box while following a simple, sinusoidal joint trajectory.  During this process the hand is susceptible to making contact transitions as the box slips from the grasp.  We record the divergence from the desired trajectory over the course of the experiments.  We note that the objective of this task is to accurately follow the joint trajectory---predicting joint torques and contact forces with rigid contact constraints---not to hold onto the box firmly.

\section{Results}

\label{section:results}
This section quantifies and plots results from the experiments in the previous
section in five ways: \1 joint trajectory tracking; \2 accuracy of contact force prediction; \3 torque command smoothness; \4 center-of-mass behavior over a gait; \5 computation speed.  
 
\subsubsection{Trajectory tracking on planar surfaces:}
\label{sec:trajectory-tracking}

	We analyze tracking performance using the quadruped platform on both rigid planar and compliant planar 
surfaces (see Figure~\ref{fig:error-plot}).  Joint tracking data was collected from the simulated quadruped using the baseline, reference and experimental controllers to locomote on a planar terrain with varying surface properties.  Numerical results for these experiments are presented in Table~\ref{table:error}.  The experimental controllers implementing contact force prediction,\textbf{ID($t_i$)}, either outperformed or matched the performance of the inverse dynamics formulations using sensed contact, \textbf{ID($t_{i-1}$)} and \textbf{ID($t_{i-2}$)}. 

	 As expected, the baseline controller (PID) performed substantially worse than all inverse dynamics systems for positional tracking on a low friction surface. Also, only the inverse dynamics controllers that use predictive contact managed to gracefully locomote with no-slip contact.  
	 
	 The reference inverse dynamics controllers with sensed contact performed the worst on high friction surfaces, only serving to degrade locomotion performance from the baseline controller over the course of the experiment.  We assume that the performance of a well tuned PID error-feedback controller may be due to the control system absorbing some of the error introduced by poor planning, where more accurate tracking of planned trajectories may lead to worse overall locomotion stability.

  \begin{figure}[H]
\figuretitle{Trajectory Tracking: Quadruped}
\center
\figuretitle{High friction ($\mu = \infty$), rigid surface}
\includegraphics[width=\linewidth]{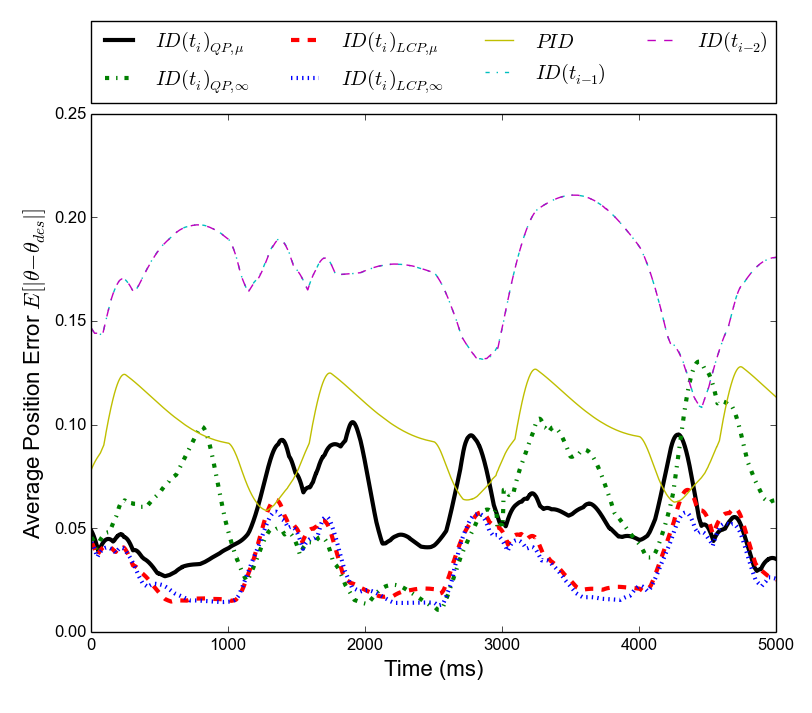}
\figuretitle{Low friction ($\mu = 0.1$), rigid surface}
\includegraphics[width=\linewidth]{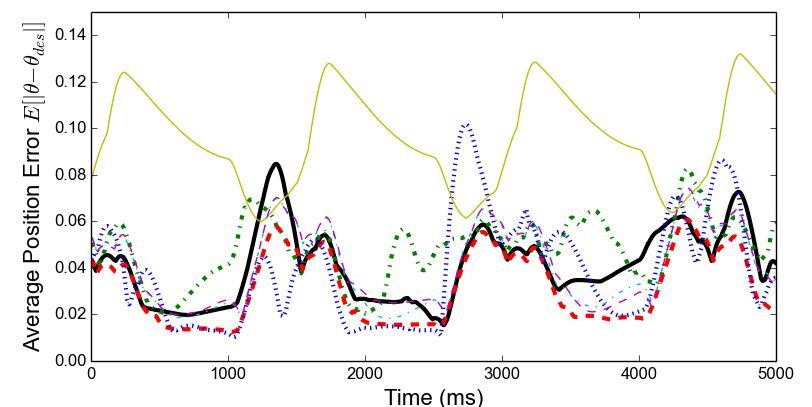} 
\figuretitle{High friction ($\mu = \infty$), compliant surface}
\includegraphics[width=\linewidth]{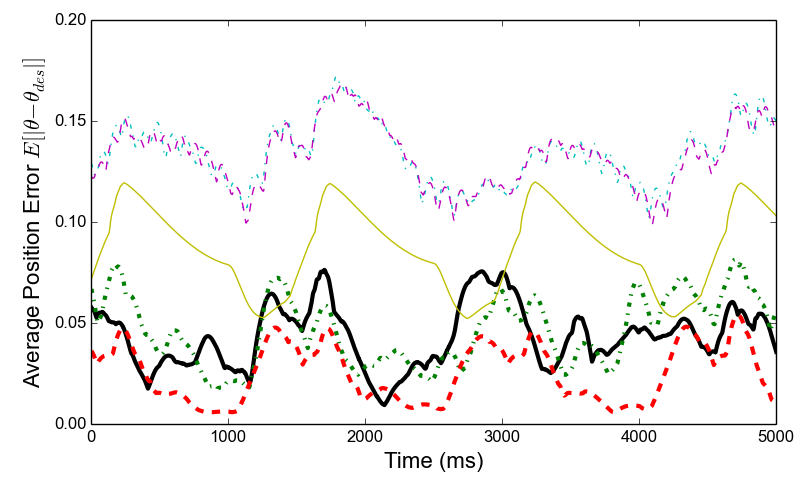}
\figuretitle{Low friction ($\mu = 0.1$), compliant surface}
\includegraphics[width=\linewidth]{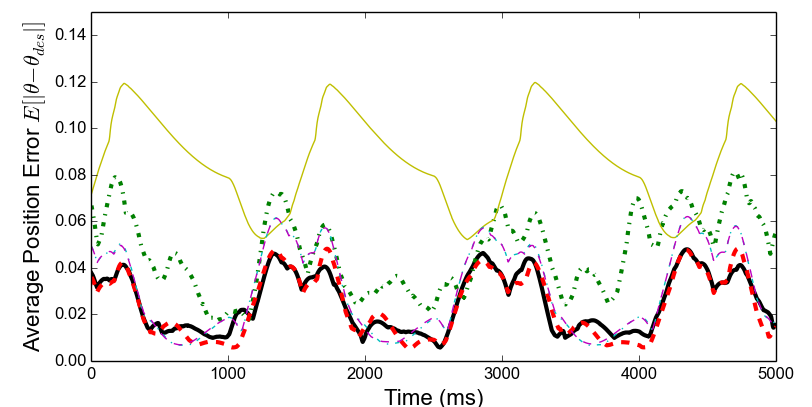} 
\caption{\label{fig:error-plot}\emph{Average position error for all joints} ($E[|\theta-\theta_{des}|]$) over time while the quadruped performs a trotting gait.}
\end{figure}

\begin{table*}[htpb]
\centering
\figuretitle{Trajectory Tracking Error}
\medskip
\begin{tabular}{cc}
\begin{tabular}{|l|c|c|}
\hline
\multicolumn{3}{|c|}{Rigid, Low Friction}\\
\hline
Controller & positional error& velocity error \\
\hline
\textbf{ID($t_{i}$)}$_{QP,\mu}$& 0.0310 & 1.9425 \\
\textbf{ID($t_{i}$)}$_{QP,\infty}$&      0.0483 & 2.6295 \\
\textbf{ID($t_{i}$)}$_{LCP,\mu}$&    \bf 0.0239 &  1.8386 \\
\textbf{ID($t_{i}$)}$_{LCP,\infty}$& - &    - \\
\textbf{PID}& 0.0895 &    \bf1.5569 \\
\textbf{ID($t_{i-1}$)}& 0.0325 &    1.7952 \\
\textbf{ID($t_{i-2}$)}& 0.0328 &    1.7830 \\
\hline
\end{tabular}
&
\begin{tabular}{|l|c|c|}
\hline
\multicolumn{3}{|c|}{Rigid, High Friction}\\
\hline
Controller & positional error& velocity error \\
\hline
\textbf{ID($t_{i}$)}$_{QP,\mu}$&   0.0486 &  2.2596 \\
\textbf{ID($t_{i}$)}$_{QP,\infty}$&   0.0654 &  2.5737 \\
\textbf{ID($t_{i}$)}$_{LCP,\mu}$& \bf0.0259 &    1.8784 \\
\textbf{ID($t_{i}$)}$_{LCP,\infty}$& 0.0260 &    1.8950  \\
\textbf{PID}& 0.0916 &    \bf1.4653 \\
\textbf{ID($t_{i-1}$)}& 0.1317 &    2.5585 \\
\textbf{ID($t_{i-2}$)}& 0.1316 &    2.5608 \\
\hline
\end{tabular}
\\
\begin{tabular}{|l|c|c|}
\hline
\multicolumn{3}{|c|}{Compliant, Low Friction}\\
\hline
Controller & positional error& velocity error \\
\hline
\textbf{ID($t_{i}$)}$_{QP,\mu}$& \bf 0.0217 & 2.0365 \\
\textbf{ID($t_{i}$)}$_{QP,\infty}$&      - & - \\
\textbf{ID($t_{i}$)}$_{LCP,\mu}$&    0.0219   & 2.0786 \\
\textbf{ID($t_{i}$)}$_{LCP,\infty}$& - &    - \\
\textbf{PID}& 0.0850  &  \bf1.5845\\
\textbf{ID($t_{i-1}$)}& 0.0265   & 1.8858\\
\textbf{ID($t_{i-2}$)}&  0.0267   & 1.8742\\
\hline
\end{tabular}
&
\begin{tabular}{|l|c|c|}
\hline
\multicolumn{3}{|c|}{Compliant, High Friction}\\
\hline
Controller & positional error& velocity error \\
\hline
\textbf{ID($t_{i}$)}$_{QP,\mu}$& 0.0342  &  2.9360 \\
\textbf{ID($t_{i}$)}$_{QP,\infty}$&  0.0446  &  3.9779 \\
\textbf{ID($t_{i}$)}$_{LCP,\mu}$& \bf 0.0226 &   2.1243 \\
\textbf{ID($t_{i}$)}$_{LCP,\infty}$& - &    - \\
\textbf{PID}& 0.0850  & \bf 1.5845 \\
\textbf{ID($t_{i-1}$)}&0.1270  &  4.4377 \\
\textbf{ID($t_{i-2}$)}&0.1270  & 4.2061 \\
\hline
\end{tabular}
\end{tabular}
\caption{ \label{table:error}Expected trajectory tracking error for quadrupedal locomotion (positional: mean magnitude of radian error for all joints over trajectory duration ($E[E[|\theta - \theta_{des}|]]$), velocity: mean magnitude of radians/second error for all joints over trajectory duration ($E[E[|\dot{\theta} - \dot{\theta}_{des}|]]$)) of inverse dynamics controllers (\textbf{ID(..)}) and baseline (\textbf{PID}) controller.}
\end{table*}

\subsection{Smoothness of torque commands}

\begin{figure}[h!]
\figuretitle{Torque Chatter}
\includegraphics[width=\linewidth]{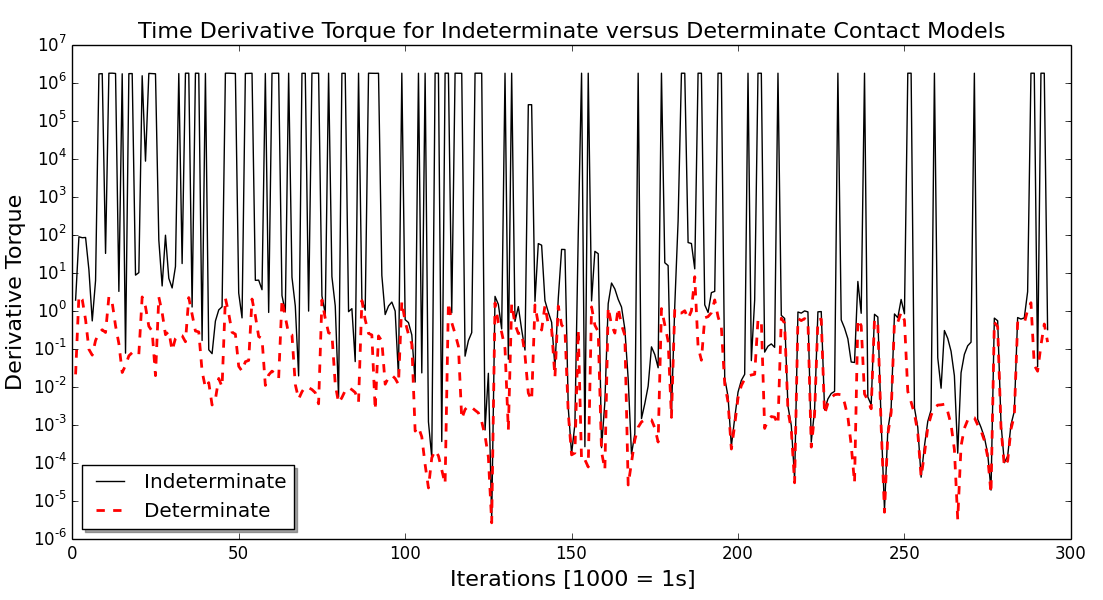}
\caption{Time derivative torque when using the inverse dynamics method (\textbf{ID($t_i$)}$_{QP,\mu}$) with means for mitigating torque chatter from indeterminate contact (red/dotted) vs. no such approach (black/solid).}
\label{fig:torque}
\end{figure}

 \begin{table}[H]
\figuretitle{Torque Smoothness}
\medskip
\centering
\begin{tabular}{|l|l|l|}
\hline
Controller & $E[|\Delta \tau|]$ & $E[|\tau|]$  \\
\hline
\textbf{ID($t_{i}$)}$_{QP,\mu}$& \bf 52.5718 &  \bf  0.2016 \\
\textbf{ID($t_{i}$)}$_{QP,\infty}$& 245.5117 & 0.4239\\
\textbf{ID($t_{i}$)}$_{LCP,\mu}$&  677.9693 & 0.8351\\
\textbf{ID($t_{i}$)}$_{LCP,\infty}$&351.1803 &  0.4055\\
\textbf{ID($t_{i-1}$)} & 974.3312 & 0.9918\\
\textbf{ID($t_{i-2}$)} &  17528.0 & 2.7416\\
\textbf{PID} & 100.1571 &    0.3413 \\
\hline
\end{tabular}
\caption{\label{table:torque} Average derivative torque magnitude (denoted $E[|\Delta \tau|]$) and average torque magnitude (denoted $E[|\tau|]$) for all controllers.}
\end{table}

 Figure~\ref{fig:torque} shows the effects of an indeterminate contact model on torque smoothness.  Derivative of torque commands are substantially smaller when incorporating a method of mitigating chatter in the inverse dynamics-derived joint torques.  We observe a five order of magnitude reduction in the maximum of the time derivative torque when using a torque smoothing stage with the Drumwright-Shell contact model.  Controller \textbf{ID($t_{i}$)}$_{LCP,\mu}$ is the only presented controller unable to mitigate torque chatter (seen in the ``indeterminate'' case in Figure~\ref{fig:torque}) and therefore produces the worst performance from the presented inverse dynamics methods.  Though it demonstrates successful behavior in simulation, this method would likely not be suitable for use on a physical platform.  The reference inverse dynamics controllers (\textbf{ID($t_{i-1}$)} and \textbf{ID($t_{i-2}$)}) exhibit significant torque chatter also.  
 
	We measure the ``smoothness'' of torque commands as the mean magnitude of derivative torque over time. We gather from the data presented in Table~\ref{table:torque} that the two phase QP-based inverse dynamics controller (\textbf{ID($t_{i}$)}$_{QP,\mu}$) followed by the baseline controller (\textbf{PID}) are the most suitable for use on a physical platform.   Controller \textbf{ID($t_{i}$)}$_{QP,\mu}$ uses the lowest torque to locomote while also mitigating sudden changes in torque that may damage robotic hardware.  

\subsection{Verification of correctness of inverse dynamics}
	\label{section:verificaton}

	 We verify correctness of the inverse dynamics approaches by comparing the contact predictions against the reaction forces generated by the simulation. The comparison considers only the $\ell_1$-norm of normal forces, though frictional forces are coupled to the normal forces (so ignoring the frictional forces is not likely to skew the results).  We evaluate each experimental controller's contact force prediction accuracy given sticky and slippery frictional properties on rigid and compliant surfaces. 

The QP-based controllers are able to predict the contact normal force in simulation to a relative error between 12--30\%.\footnote{The QP-based inverse dynamics models use a contact model that differs from the model used within the simulator. When the simulation uses the identical QP-based contact model, prediction exhibits approximately 1\% relative error.}  The \textbf{ID($t_{i}$)}$_{LCP,\mu}$, \textbf{ID($t_{i}$)}$_{LCP,\infty}$ controllers demonstrated contact force prediction between 1.16-1.94\% relative error while predicting normal forces on a rigid surface (see Table~\ref{table:controller-sim}).  The QP based controllers performed as well on a compliant surface as they did on the rigid surfaces, while the performance of the \textbf{ID($t_{i}$)}$_{LCP,\mu}$, \textbf{ID($t_{i}$)}$_{LCP,\mu}$ controllers was substantially degraded on compliant surfaces.
		
	The LCP-based inverse dynamics models (\textbf{ID($t_{i}$)}$_{LCP,\mu}$ and \textbf{ID($t_{i}$)}$_{LCP,\infty}$) use a contact model that matches that used by the simulator.  Nevertheless no inverse dynamics predictions always match the measurements provided by the simulator. Investigation
determined that the slight differences are due to \1 occasional inconsistency
in the desired accelerations (we do not use the check described in Section~\ref{section:retrieving-forces}); \2 the approximation of the
friction cone by a friction pyramid in our experiments (the axes of the
pyramid do not necessarily align between the simulation and the inverse dynamics model); and \3 the regularization occasionally necessary to solve the LCP (inverse dynamics might require regularization while the simulation might not, or \emph{vice versa}).

\begin{table*}[h!]
\centering
\figuretitle{Contact Force Prediction Error}
\medskip
\begin{tabular}{cc}
\begin{tabular}{|l|c|c|}
\hline
\multicolumn{3}{|c|}{Rigid, Low Friction}\\
\hline
Controller & absolute error& relative error \\
\hline
\textbf{ID($t_{i}$)}$_{QP,\mu}$& 3.8009 N & 12.53\% \\
\textbf{ID($t_{i}$)}$_{QP,\infty}$&      8.4567 N & 22.26\% \\
\textbf{ID($t_{i}$)}$_{LCP,\mu}$&     0.9371 N &  1.94\% \\
\textbf{ID($t_{i}$)}$_{LCP,\infty}$& - &    - \\
\hline
\end{tabular}
&
\begin{tabular}{|l|c|c|}
\hline
\multicolumn{3}{|c|}{Rigid, High Friction}\\
\hline
Controller & absolute error& relative error \\
\hline
\textbf{ID($t_{i}$)}$_{QP,\mu}$&   13.8457 N &  27.48\% \\
\textbf{ID($t_{i}$)}$_{QP,\infty}$&   12.4153 N &  25.26\% \\
\textbf{ID($t_{i}$)}$_{LCP,\mu}$& 1.2768 N &    1.55\% \\
\textbf{ID($t_{i}$)}$_{LCP,\infty}$& 0.3572 N &    1.16 \% \\
\hline
\end{tabular}
\\
\begin{tabular}{|l|c|c|}
\hline
\multicolumn{3}{|c|}{Compliant, Low Friction}\\
\hline
Controller & absolute error& relative error \\
\hline
\textbf{ID($t_{i}$)}$_{QP,\mu}$& 7.8260 N & 17.12\% \\
\textbf{ID($t_{i}$)}$_{QP,\infty}$&      - & - \\
\textbf{ID($t_{i}$)}$_{LCP,\mu}$&     4.6385 N &  6.37\% \\
\textbf{ID($t_{i}$)}$_{LCP,\infty}$& - &    - \\
\hline
\end{tabular}
&
\begin{tabular}{|l|c|c|}
\hline
\multicolumn{3}{|c|}{Compliant, High Friction}\\
\hline
Controller & absolute error& relative error \\
\hline
\textbf{ID($t_{i}$)}$_{QP,\mu}$& 14.9225 N & 30.86\% \\
\textbf{ID($t_{i}$)}$_{QP,\infty}$&      15.1897 N & 30.66\% \\
\textbf{ID($t_{i}$)}$_{LCP,\mu}$&     12.4896 N &  24.00\% \\
\textbf{ID($t_{i}$)}$_{LCP,\infty}$& - &    - \\
\hline
\end{tabular}
\end{tabular}
\caption{ \label{table:controller-sim} Average contact force prediction error (summed normal forces) of inverse dynamics controllers \emph{vs.} measured reaction forces from simulation.  The quadruped exerts 47.0882 N of force against the ground when at rest under standard gravity.  Results marked with a ``-'' indicate that the quadruped was unable to complete the locomotion task before falling.}
\end{table*}

	\label{exp:controller-sim}

\subsection{Controller behavior}
	The presented data supports the utilization of the QP-based inverse dynamics model incorporating Coulomb friction, at least for purposes of control of existing physical hardware. We also observed that utilizing Coulomb friction in inverse dynamics leads to much more stable locomotion control on various friction surfaces.  The no-slip contact models proved to be more prone to predicting excessive tangential forces and destabilizing the quadruped while not offering much additional performance for trajectory tracking.  Accordingly, subsequent results for locomotion on a height map and controlling a fixed-base manipulator while grasping a box are reported only for \textbf{ID($t_{i}$)}$_{QP,\mu}$ which can be referred to more generally as ``the inverse dynamics controller with contact force prediction'' or \textbf{ID($t_{i}$)}.

\paragraph{Rigid non-planar surface:}

	Figure~\ref{fig:links-terrain} plots trajectory tracking performance of the locomoting quadruped on rigid terrain with variable friction (ranging between low and high
values of Coulomb friction for contacting materials as reported in literature).  Three reference controllers are compared against our inverse dynamics controller.  During this experiment only the ideal sensor controller \textbf{ID($t_{i-1}$)} consistently produced better positional tracking than our proposed controller (\textbf{ID($t_{i}$)}).  Our experimental controller reduced tracking error below that of error-feedback control alone by 19\%.

\begin{figure}[H]
\figuretitle{Trajectory Tracking: Quadruped}
\center
\figuretitle{Random friction, rigid heightmap}
\includegraphics[width=\linewidth]{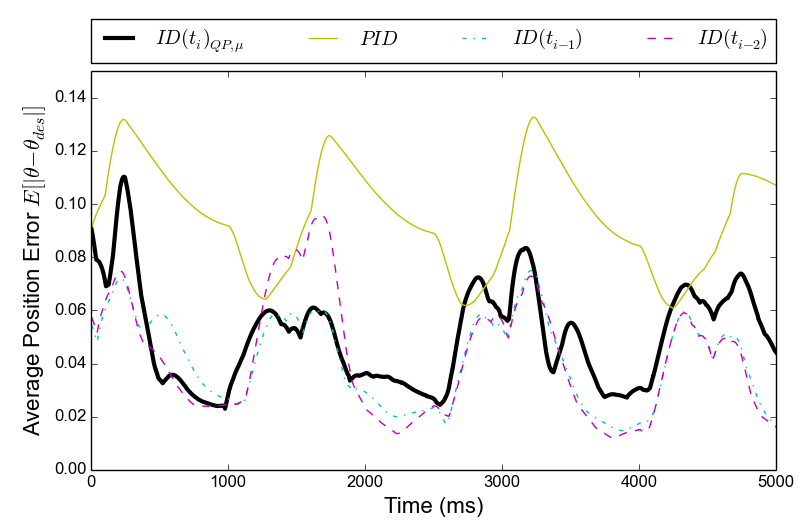} 
\vspace{-0.2in}
\caption{Joint trajectory tracking for a quadruped on a rigid heightmap with uniform random friction $\mu \sim \mathcal{U}(0.1,1.5)$.}
\vspace{-0.2in}
\label{fig:terrain-rand}.
\end{figure}

\subsection{Center-of-mass tracking performance}
\label{sec:performance}

The ability of the controller to track the quadruped's center-of-mass over a path is a meta metric, as one expects this metric to be dependent upon joint tracking accuracy. Figure~\ref{fig:terrain} shows that both the PID and the \textbf{ID($t_i$)} methods are able to track 
straight line paths fairly well. \textbf{ID($t_{i-1}$)}, which yielded better 
joint position tracking, does not track the center-of-mass as well. 
 \textbf{ID($t_{i-2}$)} results in worse tracking with respect to both joint position
 and center-of-mass position. We hypothesize that this discrepancy is due to
our observation that \textbf{ID($t_{i-1}$)} and \textbf{ID($t_{i-2}$)} yield significantly larger 
joint velocity tracking errors than the PID and \textbf{ID($t_i$)} controllers.

\begin{figure*}[h!]
\begin{tabular}{cc}
\includegraphics[width=0.5\linewidth]{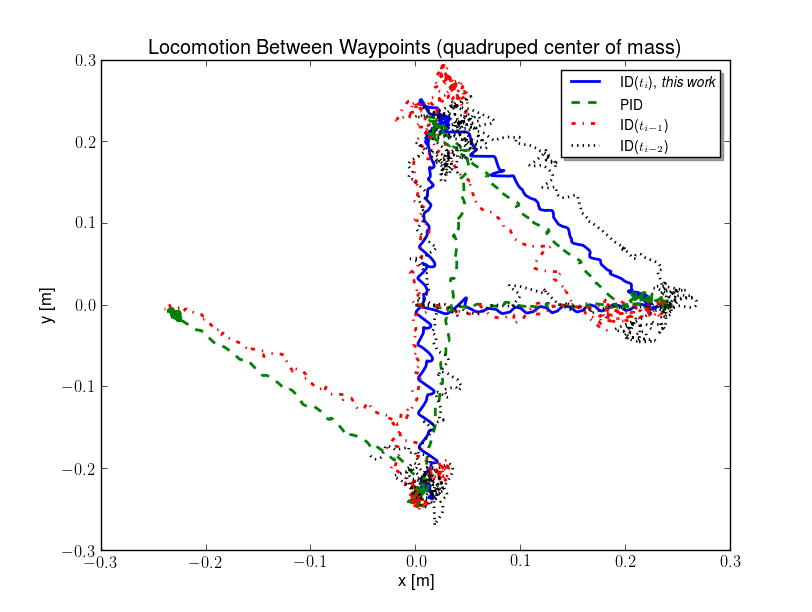} &
\includegraphics[width=0.5\linewidth]{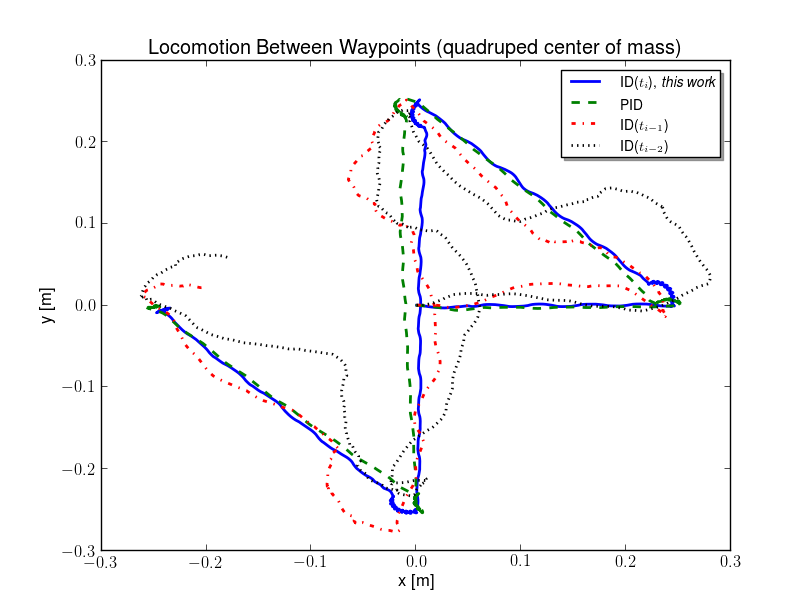} 
\end{tabular}
\caption{\label{fig:terrain}Center-of-mass path in the horizontal plane between waypoints over 30 seconds.  (left) high friction; (right) low friction. The quadruped is commanded
to follow straight line paths between points $\{(0,0),(0.25,0),(0,0.25),(0,-0.25),(-0.25,0)\}$.}
\end{figure*}

\subsubsection{Fixed base manipulator grasping a box}
	
	Trajectory tracking results for the fixed-base manipulator are presented in Figure~\ref{fig:manipulator-tracking}.  We observed large errors in the \textbf{PID} and \textbf{ID($t_{i-1}$)} controllers while grasping the box, the sensed force inverse dynamics controller was adversely affected when attempting to manipulate the sticky object, applying excessive forces while manipulating the box with high friction ($\mu = \infty$).  Both the \textbf{PID} and \textbf{ID($t_{i}$)}$_{QP,\mu}$ controllers dropped the box with low friction ($\mu = 1.0$) at about 1500 milliseconds. We observed the trajectory error quickly converge to zero after the inverse dynamics method dropped the grasped object, while the \textbf{PID} controller maintained a fairly high level of positional error.  The sensed contact inverse dynamics controller \textbf{ID($t_{i-1}$)} performed at the same accuracy as the predictive contact force inverse dynamics controller, and managed to not drop the box over the course of the three second experiment.

	Though the box slipped from the grasp of the inverse dynamics controlled manipulator, its tracking error did not increase substantially.  This demonstrates a capability of the controller to direct the robot through the task with intermittent contact transitions with heavy objects, while maintaining accuracy in performing its trajectory-following task. 	

\begin{figure}[H]
\center
\figuretitle{Trajectory Tracking: Manipulator}
\figuretitle{High friction ($\mu = \infty$), rigid surface}
\includegraphics[width=\linewidth]{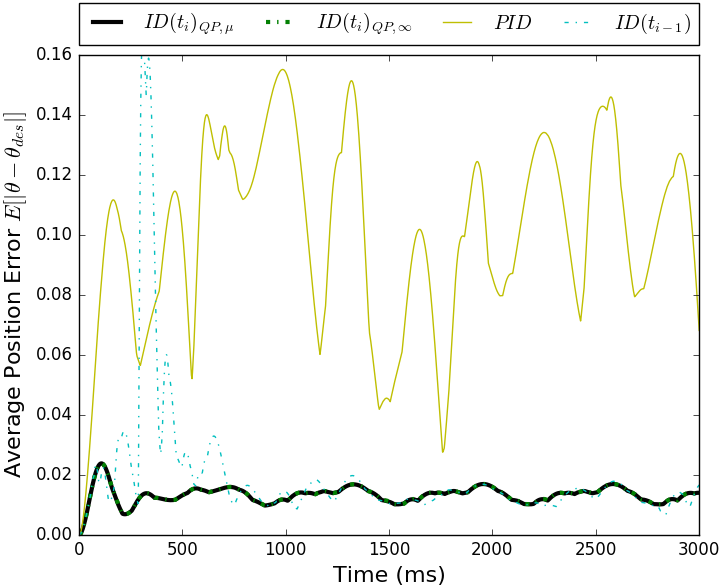} 
\figuretitle{Low friction ($\mu = 1.0$), rigid surface}
\includegraphics[width=\linewidth]{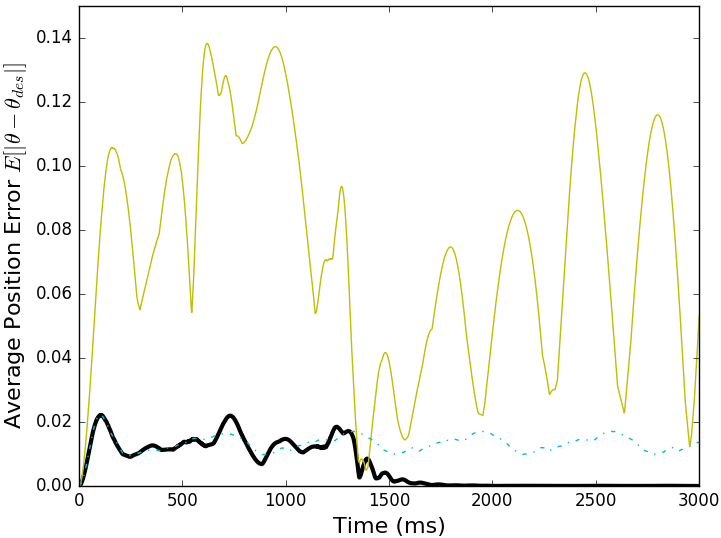} 
\vspace{-0.2in}
\caption{\label{fig:manipulator-tracking} Joint trajectory tracking for a fixed base manipulator grasping a heavy box (6000 $\frac{kg}{m^3}$) with friction: (top) $\mu = \infty$--- no-slip; and (bottom) $\mu = 1$.}
\vspace{-0.2in}
\end{figure}

\subsubsection{Running time experiments}
\label{section:timing}
\label{exp:speed}

\begin{figure}[H]
\begin{center}
\includegraphics[width=\linewidth]{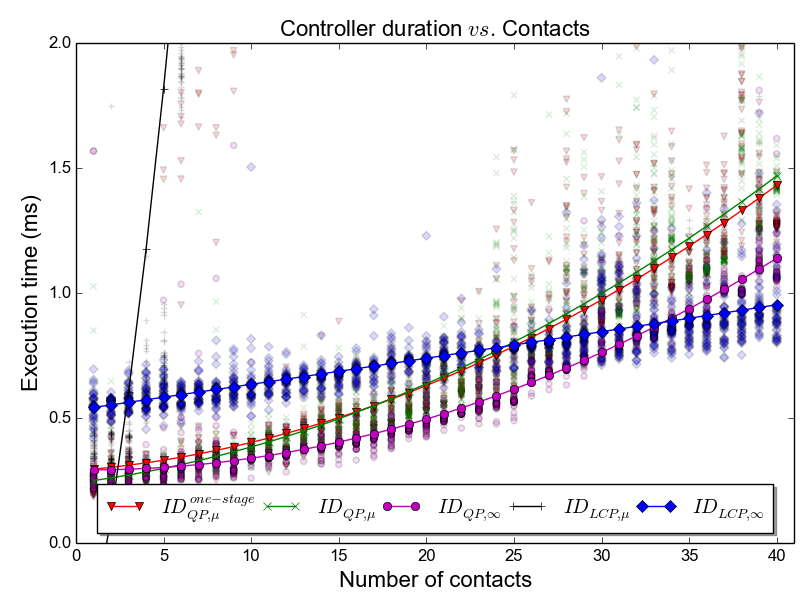}
\caption{\label{fig:timing-plot} Inverse dynamics controller runtimes for increasing numbers of contacts (Quadruped with spherical feet).}
\end{center}
\vspace{-0.2in}
\end{figure}

We measured the computation time for each controller during the quadruped experiments, artificially increasing the number of simultaneous contacts at each foot-ground interface.\footnote{Experiments were performed on a 2011 \software{MacBook Pro} Laptop with a 2.7 GHz Intel Core i7 CPU.} Figure~\ref{fig:timing-plot} shows that inverse dynamics method \textbf{ID($t_{i}$)}$_{LCP,\infty}$ scales linearly with additional contacts.  The fast pivoting algorithm (\textbf{ID($t_{i}$)}$_{LCP,\infty}$) can process in excess of 40 contacts while maintaining below a 1 ms runtime---capable of 1000 Hz control rate.  The experimental QP-based controllers: \textbf{ID($t_{i}$)}$_{QP,\infty}$, \textbf{ID($t_{i}$)}$_{QP,\mu}^\mathrm{one-stage}$, \textbf{ID($t_{i}$)}$_{QP,\mu}$ supported a control rate of 1000 Hz to about 30 total contacts when warm-starting the LCP solver with the previous solution.  Controller \textbf{ID($t_{i}$)}$_{LCP,\mu}$ did not support warm-starting or the fast pivoting algorithm and was only able to maintain around a 1 ms expected runtime for fewer than 4 contacts. The runtime for \textbf{ID($t_{i}$)}$_{LCP,\infty}$ was substantially higher than the QP-based model, despite a significantly reduced problem size, for fewer than approximately 30 points of contact.  This disparity 
is due to the high computational cost of Lines 6 and 10 in Algorithm~\ref{alg:find-indices}.

\subsection{Discussion of inverse dynamics based control for legged locomotion}

The inverse dynamics controller that predicts contact forces (\textbf{ID($t_{i}$)}$_{QP,\mu}$) performs well, at least in simulation, while mitigating destructive torque chatter. Over all tests, we observe the use of inverse dynamics control
with predicted contact forces, \emph{i.e.}, \textbf{ID($t_{i}$)}, can more closely track a joint 
trajectory than the alternatives---including inverse dynamics methods using sensed contact forces (even with perfect sensing) and PID control.  The inverse dynamics controller described in this work \textbf{ID($t_i$)} and the baseline PID controller,  were able to track center-of-mass position of quadrupeds with little deviation from the desired direction of motion, while the inverse dynamics controllers using sensed contact forces performed substantially worse at this task (as seen in Section~\ref{sec:performance}).  With these results in mind we find that contact force prediction has more potential for producing successful robot locomotion on terrain with varied surface properties.  We expect that as sensor technology and computational speeds improve, contact force prediction for inverse dynamics control will even more greatly outperform methods dependent on force sensors. 

\section{Conclusion}
We presented multiple, fast inverse dynamics methods---a method that assumes
no slip (extremely fast), a QP-based method without complementarity (very fast),
that same method with torque chattering mitigation (fast enough for 
real-time control loops at 1000Hz using current computational hardware on typical quadrupedal robots), and an
LCP-based method that enforces complementarity (fast). 
We showed that a right inverse exists for the rigid contact model with
complementarity, and we conducted 
asymptotic 
time complexity analysis for all inverse dynamics methods. Each method
is likely well suited to a particular application. For example, 
Section~\ref{section:results} found that the last of these methods yields
accurate contact force predictions (and therefore better joint trajectory
tracking) for virtual robots simulated with the Stewart-Trinkle/Anitescu-Potra contact models. Finally, we assessed performance---running times, joint
space trajectory tracking accuracy, and an indication of task execution
capability (for legged locomotion and manipulation)---under various contact
modeling assumptions.

\section{Acknowledgements}
This work was supported, in part, by NSF\\ \mbox{CMMI-110532} and
Open Source Robotics Foundation (under award NASA \mbox{NNX12AR09G}).


\begin{appendices}

\begin{landscape}
\section{. Scoring each inverse dynamics controller implementation}
\label{appendix:score-idyn}
The table in this section (Table~\ref{table:idyn-table}) scores each inverse dynamics controller implementation according to its qualitative performance over the course of each experiment.    

\begin{table}
\begin{center}
\hspace{-0.25in}
	\begin{tabular}{l | l l l l l}
		\toprule
		\bf Task   & \specialcell{No-slip (no impact) \\ \textbf{ID($t_{i}$)}$_{LCP,\infty}$  (Section~\ref{section:ID:no-slip}) } & \specialcell{No-slip (impact model) \\ \textbf{ID($t_{i}$)}$_{QP,\infty}$  (Section~\ref{section:approximate-idyn}) }& \specialcell{Coulomb friction (no impact) \\ \textbf{ID($t_{i}$)}$_{LCP,\mu}$  (Section~\ref{section:ID:Coulomb})} & \specialcell{Coulomb friction (impact model) \\ \textbf{ID($t_{i}$)}$_{QP,\mu}$  (Section~\ref{section:approximate-idyn}) }\\  
		\midrule
		\multirow{4}{1.6in}{\bf Biped walking / quadruped two-footed gait high friction terrain\\ $[0-2]$ contacts } 
			& \cellcolor{green!25}{\quad \textbf{PASS \dag}}					& \cellcolor{green!25}{\quad \textbf{PASS}} 							& \cellcolor{green!25}{\quad \textbf{PASS}} 							& \cellcolor{green!25}{\quad \textbf{PASS}}						\\
			& \gooditem continuous contact forces  		& \gooditem continuous contact forces  		& \gooditem continuous contact forces  		& \gooditem continuous contact forces  		\\
			& \gooditem distributed contact forces 	& \gooditem distributed contact forces	& \gooditem distributed contact forces	& \gooditem distributed contact forces 	\\
			& \gooditem fastest computation  				&  \baditem slower computation  				& \baditem slower computation  				& \baditem slowest computation  				\\
				\midrule
		\multirow{4}{1.6in}{\bf Biped walking / quadruped two-footed gait on low friction terrain\\ $[0-2]$ contacts } 
			& \cellcolor{red!25}{\quad \textbf{FAIL}} (slip and fall)				& \cellcolor{red!25}{\quad \textbf{FAIL}} (slip and fall)				& \cellcolor{green!25}{\quad \textbf{PASS \dag}} 							& \cellcolor{green!25}{\quad \textbf{PASS}} 							\\
			& \baditem excessive contact forces	 	& \baditem excessive contact forces		& \gooditem continuous contact forces  	& \gooditem continuous contact forces  		\\
			&								&								& \gooditem distributed contact forces	& \gooditem distributed contact forces 	\\
			& 						  		&						 		& \baditem slower computation  		& \baditem slowest computation  				\\
				\midrule
		\multirow{4}{1.6in}{\bf Quadruped (or greater number of legs) walking on high friction terrain\\ $[2-n]$ contacts }
			& \cellcolor{green!25}{\quad \textbf{PASS}} 							& \cellcolor{green!25}{\quad \textbf{PASS \dag}}							& \cellcolor{red!25}{\quad \textbf{FAIL}} (torque chatter)						& \cellcolor{green!25}{\quad \textbf{PASS}} 							\\
			& \gooditem continuous contact forces  		& \gooditem continuous contact forces  		& \baditem discontinuous contact forces  		& \gooditem continuous contact forces  		\\
			& \baditem undistributed contact forces 	& \gooditem distributed contact forces 	& \baditem undistributed contact forces	& \gooditem distributed contact forces 	\\
			& \gooditem fastest computation  				&  \baditem slower computation  				&											& \baditem slowest computation  				\\
				\midrule
		\multirow{4}{1.6in}{\bf Quadruped (or greater number of legs) walking on low friction terrain\\ $[2-n]$ contacts }
			& \cellcolor{red!25}{\quad \textbf{FAIL}} (slip and fall)				& \cellcolor{red!25}{\quad \textbf{FAIL}} (slip and fall)				& \cellcolor{red!25}{\quad \textbf{FAIL}} (torque chatter)						& \cellcolor{green!25}{\quad \textbf{PASS \dag}} 							\\
			& \baditem excessive contact forces	 	& \baditem excessive contact forces		& \baditem discontinuous contact forces  		& \gooditem continuous contact forces  		\\
			& 										& 										& \baditem undistributed contact forces	& \gooditem distributed contact forces 	\\
			& 										& 										& 							  		& 							  		\\
				\midrule
		\multirow{4}{1.6in}{\bf Fixed base manipulator grasping a high friction object\\ $[2-n]$ contacts }
			& \cellcolor{green!25}{\quad \textbf{PASS}} 							& \cellcolor{green!25}{\quad \textbf{PASS \dag}}							& \cellcolor{red!25}{\quad \textbf{FAIL}} (torque chatter)						& \cellcolor{green!25}{\quad \textbf{PASS}} 							\\
			& \gooditem continuous contact forces  		& \gooditem continuous contact forces  		& \baditem discontinuous contact forces  		& \gooditem continuous contact forces  		\\
			& \baditem undistributed contact forces 	& \gooditem distributed contact forces 	& \baditem undistributed contact forces	& \gooditem distributed contact forces 	\\
			& \gooditem fastest computation  				&  \baditem slower computation  				&											& \baditem slowest computation  				\\
				\midrule
		 \multirow{4}{1.6in}{\bf Fixed base manipulator grasping a low friction object\\ $[2-n]$ contacts }
			& \cellcolor{red!25}{\quad \textbf{FAIL}} (slip and fall)				& \cellcolor{red!25}{\quad \textbf{FAIL}} (slip and fall)				& \cellcolor{red!25}{\quad \textbf{FAIL}} (torque chatter)						& \cellcolor{green!25}{\quad \textbf{PASS \dag}} 							\\
			& \baditem excessive contact forces	& \baditem excessive contact forces& \baditem discontinuous contact forces  		& \gooditem continuous contact forces  		\\
			& 							& 							& \baditem undistributed contact forces		& \gooditem distributed contact forces 	\\
			& 							& 							& 							  		& 							  		\\
			\bottomrule
	\end{tabular}
\end{center}
	\caption{\label{table:idyn-table} A table describing the behavior of each inverse dynamics controller implementation when used to control disparate robot morphologies through different tasks.  
	\newline If the robot performed the task without failing any of our criteria (no torque chatter, no falling) it is marked as a pass; Otherwise, the task will be marked as a failure for the reason noted in parenthesis. 
	\newline \textbf{\dag}: Indicates which inverse dynamics implementation we determined to be the best controller for the example task, prioritizing: \1 [critical] Successful performance of the task; \2 [critical] Mitigation of torque chatter (continuous contact forces); \3 [non-critical] Even distribution of contact forces (distributed contact forces); \4 [non-critical] Computation speed.}
\end{table}
 \end{landscape}
 
\section{Generalized contact wrenches}
\label{section:generalized-wrenches}
A contact wrench applied to a rigid body will take the form:
\begin{align}
\vect{q} \equiv \begin{bmatrix} \hat{\vect{q}} \\ \vect{r} \times \hat{\vect{q}} \end{bmatrix}
\end{align}
where $\hat{\vect{q}}$ is a vector in $\mathbb{R}^3$ and $\vect{r}$ is the vector from the center of mass of the rigid body to the point of contact (which we denote $\vect{p}$). For a multi-rigid body defined in $m$ minimal coordinates, a \emph{generalized contact wrench} $\mat{Q} \in \mathbb{R}^m$ for single point of contact $\vect{p}$ would take the form:
\begin{align}
\mat{Q} = \tr{\mat{J}}\vect{q}
\end{align}
where $\mat{J} \in \mathbb{R}^{6 \times m}$ is the manipulator Jacobian (see, \emph{e.g.},~\citealp{Sciavicco:2000am}) computed with respect to $\vect{p}$.

\section{Relationship between LCPs and MLCPs}
This section describes the \emph{mixed linear complimentarity problem} (MLCP) and its relationship to the ``pure'' LCP.
	
\label{section:LCPs}
	Assume the LCP ($\vect{r}, \mat{Q}$) for $\vect{r} \in \mathbb{R}^a$ and $\mat{Q} \in \mathbb{R}^{a \times a}$:
\begin{equation}
\begin{tabular}{cccc}
$\vect{w} = \mat{Q}\vect{z} + \vect{r}$ \quad & $\vect{w} \geq \vect{0}$  & $\vect{z} \geq \vect{0}$ & $\tr{\vect{z}}\vect{w}  = 0 $
\end{tabular}
\end{equation}
for unknown vectors $\vect{z}, \vect{w} \in \mathbb{R}^a$.  A mixed linear complementarity problem (MLCP) is defined by the matrices $\mat{A} \in \mathbb{R}^{p \times s}$, $\mat{C} \in \mathbb{R}^{p \times t}$, $\mat{D} \in \mathbb{R}^{r \times s}$, $\mat{B} \in \mathbb{R}^{r \times t}$, $\vect{x} \in \mathbb{R}^s$, $\vect{y} \in \mathbb{R}^t$, $\vect{g} \in \mathbb{R}^p$, and $\vect{h} \in \mathbb{R}^r$ (where $p=s$ and $r=t$) and is subject to the following constraints:
\begin{align}
\mat{A}\vect{x} + \mat{C}\vect{z} + \vect{g} & = \vect{0} \label{eqn:MLCP-begin} \\
\mat{D}\vect{x} + \mat{B}\vect{z} + \vect{h} & \geq \vect{0} \\
\vect{z} & \geq \vect{0} \\
\tr{\vect{z}}(\mat{D}\vect{x} + \mat{B}\vect{z} + \vect{h}) & = 0 \label{eqn:MLCP-end}
\end{align}
The $\vect{x}$ variables are unconstrained, while the $\vect{z}$ variables must be non-negative.  If $\mat{A}$ is non-singular, the unconstrained variables can be computed as:
\begin{align}
\vect{x} = -\inv{\mat{A}}(\mat{C}\vect{z} + \vect{g}) \label{eqn:MLCP-free}
\end{align}
Substituting $\vect{x}$ into Equations \ref{eqn:MLCP-begin}--\ref{eqn:MLCP-end} yields the pure LCP ($\vect{r}, \mat{Q}$):
\begin{align}
\mat{Q} & \equiv \mat{B} - \mat{D}\inv{\mat{A}}\mat{C} \label{eqn:MLCP-LCP1} \\
\vect{r} & \equiv \vect{h} - \mat{D}\inv{\mat{A}}\vect{g} \label{eqn:MLCP-LCP2}
\end{align}
A solution $(\vect{z}, \vect{w})$ to this LCP obeys the relationship $\mat{Q}\vect{z} + \vect{r} = \vect{w}$; given $\vect{z}$, $\vect{x}$ is determined via Equation~\ref{eqn:MLCP-free}, solving the MCLP.

\section{The Principal Pivoting Method for solving LCPs}
\label{section:pivoting}
	
  The Principal Pivot Method I~\citep{Cottle:1968,Murty:1988} (PPM), which solves LCPs with $P$-matrices (complex square matrices with fully non-negative principal minors~\citep{Murty:1988} that includes positive semi-definite matrices as a proper subset). The resulting algorithm limits the size of matrix solves and multiplications.

The PPM uses sets $\alpha$, $\overline{\alpha}$, $\beta$, and $\overline{\beta}$ for LCP variables $\vect{z}$ and $\vect{w}$. The first two sets correspond to the $\vect{z}$ variables while the latter two correspond to the $\vect{w}$ variables. The sets have the following properties for an LCP of order $n$:
\begin{enumerate}
\item $\alpha \cup \overline{\alpha} = \{ 1, \ldots, n \}$
\item $\alpha \cap \overline{\alpha} = \emptyset$
\item $\beta \cup \overline{\beta} = \{ 1, \ldots, n \}$
\item $\beta \cap \overline{\beta} = \emptyset$
\end{enumerate}
Of a pair of LCP variables, $(z_i, w_i)$, index $i$ will either be in $\alpha$ and $\overline{\beta}$ or $\beta$ and $\overline{\alpha}$. If an index belongs to $\alpha$ or $\beta$, the variable is a \emph{basic variable}; otherwise, it is a \emph{non-basic variable}. Using this set, partition the LCP matrices and vectors as shown below:
\begin{equation}
\begin{bmatrix}
\vect{w}_{\beta} \\
\vect{w}_{\overline{\beta}}
\end{bmatrix} =
\begin{bmatrix}
\mat{A}_{\beta \alpha} & \mat{A}_{\beta \overline{\alpha}} \\ 
\mat{A}_{\overline{\beta} \alpha} & \mat{A}_{\overline{\beta} \overline{\alpha}} 
\end{bmatrix}
\begin{bmatrix}
\vect{z}_{\alpha}\\
\vect{z}_{\overline{\alpha}} 
\end{bmatrix}
+
\begin{bmatrix}
\vect{q}_{\beta}\\
\vect{q}_{\overline{\beta}}
\end{bmatrix} \nonumber
\end{equation}
Isolating the basic and non-basic variables on different sides yields:
\begin{align}
\begin{bmatrix}
\vect{z}_{\overline{\alpha}} \\
\vect{w}_{\overline{\beta}}
\end{bmatrix} = &
\begin{bmatrix}
-\mat{A}_{\beta \overline{\alpha}}\mat{A}_{\beta \alpha} &
\inv{\mat{A}}_{\beta \overline{\alpha}} \\
\mat{A}_{\overline{\beta} \alpha} - \mat{A}_{\overline{\beta} \overline{\alpha}}\inv{\mat{A}}_{\beta \overline{\alpha}}\mat{A}_{\beta \alpha} &
\mat{A}_{\overline{\beta} \overline{\alpha}}\inv{\mat{A}}_{\beta \overline{\alpha}}
\end{bmatrix}
\begin{bmatrix}
\vect{z}_\alpha \\
\vect{w}_\beta 
\end{bmatrix}
+ \ldots \nonumber \\
& \quad
\begin{bmatrix}
-\inv{\mat{A}}_{\beta \overline{\alpha}}\vect{q}_\beta \\
-\mat{A}_{\overline{\beta} \overline{\alpha}}\inv{\mat{A}}_{\beta \overline{\alpha}}\vect{q}_\beta + \vect{q}_{\overline{\beta}}
\end{bmatrix} \nonumber
\end{align}
If we set the values of the basic variables to zero, then solving for the values of the non-basic variables $\vect{z}_{\overline{\alpha}}$ and $\vect{w}_{\overline{\beta}}$ entails only computing the vector (repeated from above):
\begin{equation}
\begin{bmatrix}
-\inv{\mat{A}}_{\beta \overline{\alpha}}\vect{q}_\beta \\
-\mat{A}_{\overline{\beta} \overline{\alpha}}\inv{\mat{A}}_{\beta \overline{\alpha}}\vect{q}_\beta + \vect{q}_{\overline{\beta}}
\end{bmatrix}
\end{equation} 
	
	PPM~I operates in the following manner: \1 Find an index $i$ of a basic variable $x_i$ (where $x_i$ is either $w_i$ or $z_i$, depending which of the two is basic) such that $x_i < 0$; \2 swap the variables between basic and non-basic sets for index $i$ (\emph{e.g.}, if $w_i$ is basic and $z_i$ is non-basic, make $w_i$ non-basic and $z_i$ basic); \3 determine new values of $\vect{z}$ and $\vect{w}$; \4 repeat \1\hspace{.75mm}--\hspace{.1mm}\3 until no basic variable has a negative value. 

\section{Proof that removing linearly dependent equality constraints from the MLCP in Section~\ref{section:no-slip} does not alter the MLCP solution}
\label{section:mlcp-indep-constraints}
\begin{theorem}
The solution to the MLCP in Equations~\ref{eqn:no-slip-MLCP1} and~\ref{eqn:no-slip-MLCP2} without linearly dependent equality constraints removed from $\mat{A}$ is identical to the solution to the MLCP with reduced $\mat{A}$ matrix and unconstrained variables set to zero that correspond to the linearly dependent equality constraints.
\end{theorem}
\begin{proof}
Assume that $\mat{U}$ is a matrix with rows consisting of a set of linearly independent vectors $\{ \vect{u}_1, \ldots, \vect{u}_n \}$, where $n \in \mathbb{N}$.  Each of these vectors comes from a row of $\mat{P}$, $\mat{S}$, or $\mat{T}$. Assume $\mat{W}$ is a matrix with rows consisting of vectors $\{ \vect{w}_1, \ldots, \vect{w}_m \}$, each of which is a linear combination of the rows of $\mat{U}$, for $m \in \mathbb{N}$. $\mat{U}$ and $\mat{W}$ are related in the following way: $\mat{Z} \cdot \mat{U} = \mat{W}$, for some matrix $\mat{Z}$. The MLCP from Equations~\ref{eqn:no-slip-MLCP1}~and~\ref{eqn:no-slip-MLCP2} can then be rewritten as:
 \begin{align}
&\hspace{-1.5mm}\begin{bmatrix}
\mat{M} & -\tr{\mat{U}} & -\tr{(\mat{Z}\mat{U})} & -\tr{\mat{N}} \\
\mat{U} & \mat{0} & \mat{0} & \mat{0}  \\
\mat{Z}\mat{U} & \mat{0} & \mat{0} & \mat{0}  \\
\mat{N} & \mat{0} & \mat{0} & \mat{0}  \\
\end{bmatrix}
\begin{bmatrix}
\after \vect{v} \\
\vect{f}_U \\
\vect{0} \\
\vect{f}_N \\
\end{bmatrix}
\hspace{-1mm} + \hspace{-1mm}
\begin{bmatrix}
\vect{\kappa} \\
\vect{0} \\
\vect{0} \\
\frac{\before \vect{\phi}}{\Delta t}
\end{bmatrix} \hspace{-1mm}
= \hspace{-1mm}
\begin{bmatrix}
\vect{0} \\
\vect{0} \\
\vect{0} \\
\vect{w}_N
\end{bmatrix} \\
& \vect{f}_N \ge \vect{0}, \vect{w}_N \ge \vect{0}, \tr{\vect{f}}_N \vect{w}_N = 0 
\end{align}
\normalsize
where $\vect{f}_U$ are unconstrained variables that correspond to the 
linearly independent equality constraints. 
Note that the value of $\vect{0}$ is assigned to the variables corresponding to the linearly dependent equality constraints. Since values for $\after \vect{v}$, $\vect{f}_U$, and $\vect{f}_N$ that satisfy the equations above require $\mat{U} \after \vect{v} = \vect{0}$, the constraint $\mat{Z}\mat{U} \after \vect{v} = \vect{0}$ is automatically satisfied. \qed
\end{proof}

\section{Proof that no more than $m$ positive force magnitudes need be applied along contact normals to a $m$ degree of freedom multibody to solve contact model constraints}
	\label{thm:maxcard}
This proof will use the matrix of generalized contact wrenches, $\mat{N}\in \mathbb{R}^{n \times m}$ (introduced in Section~\ref{sec:jacobian-evaluation}), and $\mat{M} \in \mathbb{R}^{m \times m}$, the generalized inertia matrix for the multi-body. $\vect{z}_I$ is the vector of contact force magnitudes and consists of strictly positive values.

Assume we permute and partition the rows of $\mat{N}$ into $r$ linearly independent and $n-r$ linearly dependent rows, denoted by indices $I$ and $D$, respectively, as follows:
\begin{equation}
\mat{N} = \begin{bmatrix}\mat{N}_I \\ \mat{N}_D \end{bmatrix}
\end{equation}
Then the LCP vectors $\vect{q} = \mat{N}\vect{v}$, $\vect{z} \in \mathbb{R}^n$, and $\vect{w} \in \mathbb{R}^n$ and LCP matrix $\mat{Q}=\mat{N}\inv{\mat{M}}\tr{\mat{N}}$ can be partitioned as follows:
\begin{equation}
\begin{bmatrix} \mat{Q}_{II} & \mat{Q}_{ID} \\ \mat{Q}_{DI} & \mat{Q}_{DD} \end{bmatrix}\begin{bmatrix}\vect{z}_I \\ \vect{z}_D \end{bmatrix} + \begin{bmatrix}\vect{q}_I \\ \vect{q}_D \end{bmatrix} = \begin{bmatrix}\vect{w}_I \\ \vect{w}_D \end{bmatrix}
\end{equation}
Given some matrix $\boldsymbol{\gamma} \in \mathbb{R}^{(n-r) \times r}$, it is the case that $\mat{N}_{D} = \boldsymbol{\gamma}\mat{N}_{I}$, and therefore that $\mat{Q}_{DI} =  \boldsymbol{\gamma}\mat{N}_I\inv{\mat{M}}\tr{\mat{N}_{I}}$, $\mat{Q}_{ID} =  \mat{N}_I\inv{\mat{M}}\tr{\mat{N}_{I}}\tr{\boldsymbol{\gamma}}$ (by symmetry), \newline $\mat{Q}_{DD} = \boldsymbol{\gamma}\mat{N}_I\inv{\mat{M}}\tr{\mat{N}_{I}}\tr{\boldsymbol{\gamma}}$, and $\vect{q}_D = \boldsymbol{\gamma}\mat{N}_I\vect{v}$.

\begin{lemma}
\label{lemma:rank}
Since $\rank{\mat{N}\mat{M}} \leq \min\left(\rank{\mat{N}},\rank{\mat{M}}\right)$, the number of positive components of $\vect{z}_I$ can not be greater than rank$(\mat{N})$.
\end{lemma}
\begin{proof}
The columns of $\mat{N}\mat{M}$ have $\mat{N}$ multiplied by each column of $\mat{M}$, \emph{i.e.}, $\mat{N}\mat{M} = \begin{bmatrix} \mat{N}\vect{m}_1 & \mat{N}\vect{m}_2 & \ldots & \mat{N}\vect{m}_m \end{bmatrix}$.  Columns in $\mat{M}$ that are linearly dependent will thus produce columns in $\mat{N}\mat{M}$ that are linearly dependent (with precisely the same coefficients). Thus, rank($\mat{N}\mat{M}$) $\leq$ rank($\mat{M}$). Applying the same argument to the transposes produces rank($\mat{N}\mat{M}$) $\leq$ rank($\mat{N}$), thereby proving the claim. \qed
\end{proof}
We now show that no more positive force magnitudes are necessary to solve the LCP in the case that the number of positive components of $\vect{z}_I$ is equal to the rank of $\mat{N}$.

\begin{theorem}
If $(\vect{z}_I = \vect{a}, \vect{w}_I = \vect{0})$ is a solution to the LCP $(\vect{q}_I, \mat{Q}_{II})$, then $(\begin{bmatrix}\tr{\vect{z}_I} = \tr{\vect{a}} & \tr{\vect{z}_D} = \tr{\vect{0}}\end{bmatrix}^{\mathsf{T}}, \vect{w} = \vect{0})$ is a solution to the LCP $(\vect{q}, \mat{Q})$.
\end{theorem}
\begin{proof}
For $(\begin{bmatrix}\tr{\vect{z}_I} = \tr{\vect{a}} & \tr{\vect{z}_D} = \tr{\vect{0}}\end{bmatrix}^{\mathsf{T}}, \vect{w} = \vect{0})$ to be a solution to the LCP $(\vect{q}, \mat{Q})$, six conditions must be satisfied:
\begin{enumerate}
\item $\vect{z}_I \geq \vect{0}$
\item $\vect{w}_I \geq \vect{0}$
\item $\tr{\vect{z}_I}\vect{w}_I = 0$
\item $\vect{z}_D \geq \vect{0}$
\item $\vect{w}_D \geq \vect{0}$
\item $\tr{\vect{z}_D}\vect{w}_D = 0$
\end{enumerate}
Of these, \1\,, \4, and \6 are met trivially by the assumptions of the theorem. Since $\vect{z}_D = \vect{0}$, $\mat{Q}_{II}\vect{z}_I + \mat{Q}_{ID}\vect{z}_D + \vect{q}_I = \vect{0}$, and thus $\vect{w}_I = \vect{0}$, thus satisfying \2 and \3. Also due to $\vect{z}_D = \vect{0}$, it suffices to show for \5 that $\mat{Q}_{DI}\vect{z}_I + \vect{q}_D \geq \vect{0}$. From above, the left hand side of this equation is equivalent to $\boldsymbol{\gamma}(\mat{N}_I\inv{\mat{M}}\tr{\mat{N}_I}\vect{a} + \mat{N}_I\vect{v})$, or $\boldsymbol{\gamma} \vect{w}_I$, which itself is equivalent to $\boldsymbol{\gamma} \vect{0}$. Thus, $\vect{w}_D = \vect{0}$. \qed
\end{proof}

\end{appendices}

\bibliographystyle{harvard}
\bibliography{new}

\end{document}